\documentclass{article}

\PassOptionsToPackage{numbers, compress}{natbib}

\usepackage{hyperref}

\usepackage[final]{nips_2018}

\usepackage{amsmath}
\usepackage{amssymb}
\usepackage{mathtools}
\usepackage{amsthm}

\usepackage{amssymb}
\usepackage{amsmath}
\usepackage{amsthm}
\DeclareMathOperator{\E}{\mathbb{E}}

\usepackage{mathtools}
\usepackage{bm}
\usepackage{pgfplots}
\usepackage{subcaption}
\usepackage{verbatim}
\usepackage{graphicx}
\usepackage{dsfont}
\usepackage{appendix}

\usepackage{algorithm, algpseudocode}

\graphicspath{ {./images/} }
\DeclareMathOperator{\Tr}{Tr}
\DeclareMathOperator{\nul}{null}

\DeclareMathOperator{\diag}{diag}

\DeclareMathOperator*{\argmax}{arg\,max}

\DeclareMathOperator{\vect}{vec}
\DeclareMathOperator{\vech}{vech}
\DeclareMathOperator{\logm}{\mathbf{logm}}

\newtheorem{prop}{Proposition}
\newtheorem{lemma}{Lemma}

\newtheorem{corollary}{Corollary}

\usepackage[utf8]{inputenc} 
\usepackage[T1]{fontenc}    
\usepackage{hyperref}       
\usepackage{url}           
\usepackage{booktabs}       
\usepackage{amsfonts}       
\usepackage{nicefrac}       
\usepackage{microtype}      
\usepackage{placeins}

\usepackage{amsmath, amssymb}
\usepackage{bm}
\usepackage{bbm}
\usepackage{pgfplots}
\usepackage{subcaption}
\usepackage{verbatim}
\usepackage{graphicx}
\usepackage[numbers]{natbib}
\usepackage{notoccite}
\usepackage{mathtools}

\graphicspath{ {./images/} }

\usepackage[utf8]{inputenc} 
\usepackage[T1]{fontenc}    
\usepackage{hyperref}       
\usepackage{url}            
\usepackage{booktabs}       
\usepackage{amsfonts}       
\usepackage{nicefrac}       
\usepackage{microtype}      
\usepackage{placeins}
\usepackage{bbm}
\usepackage{authblk}

\font\myfont=cmr12 at 14pt
\title{\myfont System Identification for Continuous-time Linear Dynamical Systems}
\author[1]{Peter Halmos}
\author[2]{Jonathan Pillow}
\author[3-6]{David A. Knowles}

\affil[1]{\textit{Department of Computer Science, Princeton University}}
\affil[2]{\textit{Princeton Neuroscience Institute}}
\affil[3]{\textit{Department of Computer Science, Columbia University}}
\affil[4]{\textit{Department of Systems Biology, Columbia University}}
\affil[5]{\textit{Data Science Institute, Columbia University}}
\affil[6]{\textit{New York Genome Center}}

\pgfplotsset{compat=1.18}
\begin{document}
\maketitle

\begin{abstract}
System identification for the Kalman filter, relying on the expectation-maximization (EM) procedure to learn the underlying parameters of a dynamical system, has largely been studied assuming that observations are sampled at equally-spaced time points. However, in many applications this is an unrealistic assumption. We provide an algorithm for system identification for the continuous-time Kalman filter. This is done by relying on a solution to an It\^{o} stochastic differential equation (SDE) for the latent state and covariance. We introduce a novel two-filter form for the posterior, which yields analytical backward-direction updates which do not require the forward-pass to be pre-computed. Using this posterior, we derive a procedure which estimates the parameters of the SDE, naturally incorporating irregularly sampled measurements. Generalizing the learning of linear dynamical systems (LDS) to continuous-time may extend the use of LDS to data which is not regularly sampled or has intermittent missing values. We apply the method by learning the parameters of a latent, multivariate Fokker-Planck SDE representing a toggle-switch genetic circuit that uses biologically realistic parameters, and demonstrate superior performance to the discrete-time Kalman filter as the step-size irregularity and spectral-radius of the dynamics-matrix increases.
\end{abstract}
\section{Introduction}
The Kalman filter is a ubiquitous method in control theory and engineering, routinely used for the estimation of the latent state of a time-varying system subject to noisy observations and controls \cite{KF_OG}. The problem of system identification, i.e., the identification of the parameters of a dynamical system, has been addressed in discrete-time for linear time-invariant (LTI) systems using the expectation-maximization (EM) algorithm \cite{KFEM_OG}. EM is used for non-linear system identification in the context of switching LDS (SLDS), which generalize the parametric linearity to local regions in state-space or in time \cite{slds,slds_param_est}. Moreover, LDS has gained renewed interest in the machine learning community from the perspective of state-space models (SSM) such as S4 \cite{smith2023simplified} and Mamba \cite{gu2023mamba}, which achieve SOTA sequential modeling performance competitive with transformers. These models rely on LDS as a block-unit which may be stacked with non-linear activation functions. Both SLDS and SSMs generally consider LDS units in discrete-time--as such, generalizing learning of LDS to continuous-time could extend many non-linear system identification and machine learning methods which rely on LDS. The discrete-time Kalman filter in particular assumes homogenous sample times, and a few works have addressed this limitation in the context of parameter-learning. Mbalawata et al. considered a gradient-ascent based approach on the forward-filter likelihood, which provides maximum-likelihood estimates for the parameters of the continuous-discrete filter \cite{cdkf_maxlike}. While this does not require the computation time of smoothing, maximum-likelihood is generally insufficient for parameter estimation in models with latent variables. Expectation-maximization has been used in the continuous-discrete case with specific assumptions, such as observations which are Lebesgue-sampled under particular integer sample times, and also for continuous-time hidden Markov models (HMMs), but not in the most general case of continuous-time linear dynamical systems with no assumptions on the sample-time distribution \cite{lebesgue_em} \cite{cont_HMM}. 

\section{Background}

\textbf{The Kalman filter.} The standard discrete-time Kalman filter has an immediate correspondence to first-order homogeneous linear dynamics, of the general form $\mathbf{\dot{x}} = \mathbf{A x}$, where $\mathbf{A} \in \mathbb{R}^{n \times n}$ and $\mathbf{x} \in \mathbb{R}^{n}$. The ansatz for the solution of $\mathbf{x}(t) = e^{\mathbf{A}t} \mathbf{c}$, for continuous time $t$ and constant of integration $\mathbf{c}$, yields $\mathbf{\dot{x}} = \mathbf{A} e^{\mathbf{A}t} \mathbf{c}$, as expected. Integrating the continuous solution $\int_{t_{k-1}}^{t_{k}} \mathbf{\dot{x}} dt = \int_{0}^{t_{k}-t_{k-1}} \mathbf{A} e^{\mathbf{A}s} ds \mathbf{c}$, we find $\mathbf{x}(t_{k}) - \mathbf{x}(t_{k-1}) = (e^{\mathbf{A}(t_{k} - t_{k-1})} - \mathbbm{1})\mathbf{x}(t_{k-1})$ for $\mathbf{c} = \mathbf{x}(t_{k-1})$. This shows that any linear differential equation over a fixed sample interval $\tau_{k}$ has a corresponding discrete-time equivalent $\mathbf{x}(t_{k}) = e^{\mathbf{A}(t_{k} - t_{k-1})} \mathbf{x}(t_{k-1}) = e^{\mathbf{A} \tau_{k}} \mathbf{x}(t_{k-1}) = \mathbf{B}_{k} \mathbf{x}(t_{k-1})$
With $\tau_{k} = t_{k} - t_{k-1}$. By fixing the interval $\tau = t_{k} - t_{k-1}$ between a measurement at $t_{k}$, and a prediction from time $t_{k-1}$, the Kalman filter assumes a discrete evolution of the form
$$
\mathbf{x} (t_{k}) = e^{\mathbf{A} \tau_{k}} \mathbf{x} (t_{k-1}) \triangleq e^{\mathbf{A} \tau} \mathbf{x}_{k-1} = \mathbf{B} \mathbf{x}_{k-1}.$$
This yields the assumption that one learns a fixed matrix $\mathbf{B}$, rather than a continuous-time ODE. Usually EM for Kalman filtering is studied in the setting where one learns a single matrix $\mathbf{B}$ assuming the sample-times at which observations are observed are equally-spaced. A similar restriction is applied to the motion-model covariance matrix which is assumed to be a fixed matrix, which may in reality depend on time. After discrete-time EM, the only manner in which one may equate $\mathbf{B}$ to the continuous dynamics is by an evaluation of the matrix logarithm $\mathbf{A} = \frac{\logm(\mathbf{B})}{\tau}$. For irregularly sampled time-steps without any assumptions, one cannot disambiguate the true dynamics $\mathbf{A}$ with the discrete-time Kalman filter and any methods reliant on its assumption of time-discretization. As such, we propose a continuous-time generalization of the EM procedure for learning the parameters.

\textbf{Expectation-Maximization.} If one has observed pairs of points $\{ \mathbf{z}_{i}, \mathbf{x}_{i} \}_{i=1}^{N}$ sampled from some density in a parametric family $\mathcal{F_{\theta}}$, the method of maximum-likelihood enables the direct maximization of the joint-density for $\hat{\theta} = \argmax_{\theta \in \Theta} \log{\mathds{P}_{\mathbf{XZ}}(\mathbf{x}_{1:N}, \mathbf{z}_{1:N}| \theta)}$. When the variables $\mathbf{x}$ are latent, the maximization of the marginal likelihood expressed in terms of the observables is given as:
$$\log{\mathds{P}_{\mathbf{Z}}(\mathbf{z}| \theta)}  := \log{\int_{\mathbf{x}}\mathds{P}(\mathbf{z}, \mathbf{x}| \theta) d\mathbf{x}}$$
A direct-maximization of the marginal log-likelihood is intractable, given that the logarithm cannot be moved into the integral. Introducing a distribution $\mathds{Q}_{\mathbf{X}}$ over the latent-variables $\mathbf{x}_{1:N}$, one may re-express the marginal log-likelihood as
$$\log{\mathds{P}_{\mathbf{Z}}(\mathbf{z}| \theta)} = \int_{\mathbf{x}} \log{\mathds{P}_{\mathbf{Z}}(\mathbf{z}| \theta)} \mathds{Q}_{\mathbf{x}}(\mathbf{x}) d\mathbf{x}$$
Expanding this further, one finds:
$$
\int_{\mathbf{x}}  \biggl(\log{\mathds{P}_{\mathbf{x}}(\mathbf{x} , \mathbf{z} | \theta) -  \log{\mathds{P}_{\mathbf{z}}(\mathbf{x} | \mathbf{z}, \theta)} } \biggr)  \mathds{Q}_{\mathbf{x}}(\mathbf{x}) d\mathbf{x} =$$
\begin{small}
\[\int_{\mathbf{x}} \log{\left[ \frac{\mathds{P}_{\mathbf{X}}(\mathbf{x} , \mathbf{z} | \theta)}{\mathds{Q}_{\mathbf{X}}(\mathbf{x})} \right]} \mathds{Q}_{\mathbf{X}}(\mathbf{x}) d\mathbf{x} 
-  \int_{\mathbf{x}} \log{\left[\frac{\mathds{P}_{\mathbf{Z}}(\mathbf{x} | \mathbf{z}, \theta)}{\mathds{Q}_{\mathbf{X}}(\mathbf{x})} \right]}   \mathds{Q}_{\mathbf{X}}(\mathbf{x}) d\mathbf{x}
\]\end{small}
$$
\triangleq \text{ELBO}(\mathds{Q}_{\mathbf{X}}(\mathbf{x}); \theta) + D_{KL}(\mathds{Q}_{\mathbf{X}}(\mathbf{x}) \parallel \mathds{P}_{
\mathbf{X}}(\mathbf{x} | \mathbf{z}, \theta))
$$
Here we have defined $\text{ELBO}(\mathds{Q}_{\mathbf{X}}(\mathbf{x}); \theta)$ as the evidence-lower bound, and $D_{KL}(\mathds{Q}_{\mathbf{X}}(\mathbf{x}) \parallel \mathds{P}_{\mathbf{X}}(\mathbf{x} | \mathbf{z}, \theta))$ as the KL-divergence between the conditional posterior over $\mathbf{x}$ and the true posterior where $D_{KL}(\mathds{Q}_{\mathbf{X}}(\mathbf{x}) \parallel \mathds{P}_{\mathbf{X}}(\mathbf{x} | \mathbf{z}, \theta)) \geq 0$ with equality if and only if $\mathds{Q}_{\mathbf{X}}(\mathbf{x}) = \mathds{P}_{X}(\mathbf{x} | \mathbf{z}, \theta)$. Thus the ELBO is a lower-bound to the marginal log-likelihood.
The EM-algorithm constitutes an iterative procedure where for the $t$\textsuperscript{th} iteration, one performs an E-step which fixes the previous parameter $\theta_{t-1}$ and optimizes $\text{ELBO}(\mathds{Q}_{\mathbf{X}}(\mathbf{x}); \theta)$ with respect to the posterior $\mathds{Q}_{\mathbf{X}}(\mathbf{x})$. One alternates the E-step with an M-step which fixes the posterior $\mathds{Q}_{\mathbf{X}}(\mathbf{x})$ and optimizes $\text{ELBO}(\mathds{Q}_{\mathbf{X}}(\mathbf{x}); \theta)$ with respect to $\theta$ to yield $\theta_{t}$. These alternating optimizations monotonically increase the log-likelihood, proceeding until $\theta$ is a fixed-point.


\textbf{Filtering and the discrete-time Kalman filter.} Given some dataset $\{\mathbf{z}_{i}\}_{i=1}^{T}$ and dynamics model there are generally three tasks of interest. \textbf{(1.)} Filtering, or inferring the posterior distribution over the latent (or observed) variable given previously observed time-points $\mathds{P}_{\mathbf{X}}(\mathbf{x}_{t}| \mathbf{z}_{1:t-1}; \theta)$. \textbf{(2.)} Smoothing, or inferring the posterior for a latent (or observed) variable given the full data $\mathds{P}_{\mathbf{X}}(\mathbf{x}_{t}| \mathbf{z}_{1:T}; \theta)$. \textbf{(3.)} Forecasting, or predicting a future observation $\mathds{P}_{\mathbf{X}}(\mathbf{z}_{T + \delta} | \mathbf{z}_{1:T}; \theta)$, for $\delta > 0$ \cite{rubanova_thesis}. The determination of a filtered state $\mathds{P}_{\mathbf{X}}(\mathbf{x}_{t} | \mathbf{z}_{1:t-1}, \theta)$ is an essential component to smoothing, as the smoothed density can factorize into a forward-filtered density dependent on $\mathbf{z}_{1:t-1}$ and a recursive component which uses the information from the future states, $\mathbf{z}_{t:T}$. 

The Kalman filter, in discrete-time, has a motion-matrix $\mathbf{B} \in \mathbb{R}^{n \times n}$, an observation matrix relating the latent-state to an observation $\mathbf{H} \in \mathbb{R}^{m \times n}$, an initial mean state $\mathbf{\mu_{0}} \in \mathbb{R}^{n}$. It also has a constant motion-covariance $\mathbf{Q} \in \mathcal{S}_{n}^{+}$, an observation-model covariance $\mathbf{R} \in \mathcal{S}_{m}^{+}$, and an initial state covariance $\mathbf{P_{0}} \in \mathcal{S}_{n}^{+}$.
The dynamics-model of the discrete-time Kalman filter is
\begin{equation}
\mathbf{x}_{k} = \mathbf{B x}_{k-1} + \mathbf{w},
\end{equation}
and the observation model is
\begin{equation}
\mathbf{z}_{k} = \mathbf{H x}_{k} + \mathbf{v},
\end{equation}
where the motion (latent) and observation noise are distributed with constant-covariance matrices as $\mathbf{w} \sim \mathcal{N}(\mathbf{0}, \mathbf{Q})$, and $\mathbf{v} \sim \mathcal{N}(\mathbf{0}, \mathbf{R})$. The initial state is sampled in accordance with the initial mean and covariance $\mathbf{x_{0}} \sim \mathbf{\mu_{0}} + \mathbf{w_{0}}$, for $\mathbf{w_{0}} \sim \mathcal{N}(\mathbf{0}, \mathbf{P_{0}})$.

\section{Statement of the Problem}
We consider the following It\^{o} SDE, which generalizes the discrete-time formulation of the Kalman filter to continuous-time \cite{cdkf_sols_OG}
\begin{equation}
    d\mathbf{x}(t) = \mathbf{A x}(t) dt + d\mathbf{w}(t) ,
\end{equation}
where we have second-order conditions on the noise-differential of $\E[d\mathbf{w}(t)d\mathbf{w}(t)^{\mathbf{T}}] = \mathbf{Q_c} dt$ and $\E[\mathbf{w}(t)\mathbf{w}(s)^{\mathbf{T}}] = \mathbf{Q_c} \delta(t - s)$, for $\delta$ denoting a Dirac-delta. This It\^{o} SDE is equivalent to the standard SDE,
\begin{equation}\label{eqn:SDE1}
\mathbf{\dot{x}}(t) = \mathbf{A x}(t) + \mathbf{w}(t). 
\end{equation}
We assume system observability, and express this in the form typically found in the continuous-discrete Kalman filter, where the measurements themselves are observed at discrete but potentially irregularly-spaced time-points,
\begin{equation}\label{eqn:SDE2}
\mathbf{z}(t_{k}) = \mathbf{H x}(t_{k}) + \mathbf{v}.
\end{equation}
As is standard, we separate observation noise from motion-model noise, with the solution for the motion-model covariance given by a time-dependent integral equation from the SDE. The observation noise is distributed as $\mathbf{v} \sim \mathcal{N}(\mathbf{0}, \mathbf{R})$ and the motion-model noise as $\mathbf{w}(t) \sim N(\mathbf{0}, \mathbf{Q}(t))$.
We introduce the standard notation $f/-$ indicating an a priori estimate of the system state in the forward direction without measurement, and $f/+$ indicating the posterior given to the state after a measurement (with the notation for the backwards direction analogously denoted $b/-$ and $b/+$). The continuous-discrete extension of the discrete-time Kalman filter, which involves a latent-dynamics matrix $\mathbf{A}$ of first-order derivatives, has an established result \cite{cdkf_sols} for the forward filter $\mathds{P}_{\mathbf{X}}(\mathbf{x}_{t} | \mathbf{z}_{1:t-1}, \theta) = \mathcal{N}(\mathbf{x}_{t} | \mathbf{x}_{t}^{f/+}, \mathbf{P}_{t}^{f/+} )$ where
\begin{align*}
\mathbf{x}_{t}^{f/-} &= e^{\mathbf{A} \tau_{t}} \mathbf{x}_{t-1}^{f/+}\\
   \mathbf{x}_{t}^{f/+} &= \mathbf{x}_{t}^{f/-} + \mathbf{K}_{t} \left( \mathbf{z}_{t} - \mathbf{H} e^{\mathbf{A} \tau_{t}} \mathbf{x}_{t-1}^{f/+} \right) \\
   \mathbf{P}_{t}^{f/-} &= e^{\mathbf{A}\tau_{t}} \mathbf{P}_{t-1}^{f/+} e^{\mathbf{A}^{T}\tau_{t}} + \mathbf{Q}(\tau_{t})\\
    \mathbf{P}_{t}^{f/+} &= \left( \mathbbm{1} - \mathbf{K}_{t} \mathbf{H}  \right) \mathbf{P}_{t}^{f/-}
\end{align*}
where $\mathbf{P}_t$ represents the covariance of $\mathbf{x}_t$ and the ``Gain-matrix'' $\mathbf{K}_{t}$ is
\begin{equation}
    \mathbf{K}_{t} = \mathbf{P}_{t}^{f/-} \mathbf{H}^{T} \left( \mathbf{H P}_{t}^{f/-} \mathbf{H}^{T} + \mathbf{R} \right)^{-1}. 
\end{equation}

\section{Main Results}
\subsection{Alternative derivations of solutions to the Lyapunov-type SDE}
Given the SDE above \ref{eqn:SDE1}, one can solve to yield a time-dependent analytical solution to the evolution of both the mean-state and covariance-matrix. We offer analytical solutions to the Lyapunov-type SDEs in section \ref{sec:SDEsolns}, given with variation of parameters as a derivation distinct from Axelsson et al. that recapitulates the same final solution \cite{cdkf_sols}. The evolution of the covariance matrix of $\mathbf{x}$ in the forward-direction is given by the Lyapunov differential equation
$$
\mathbf{\dot{P}}(t) = \mathbf{A P}(t) + \mathbf{P}(t) \mathbf{A^{T}} + \mathbf{Q_c}.
$$
In Appendix \ref{sec:SDEsolns}, the analytical form for the evolution of the covariance matrix for the motion noise $\mathbf{w}(t)$ is derived as
\begin{align*}
\vech{\mathbf{Q}(t)} &= e^{\mathbf{A_{P}} (t-t_{0})} \mathbf{A}_{P}^{-1} (\mathbbm{1} - e^{-\mathbf{A_{P}}(t - t_{0})}) \vech{\mathbf{Q_c}}\\
\mathbf{A_{P}} &= \mathbf{D^{\dagger}} (\mathbbm{1} \otimes \mathbf{A} + \mathbf{A} \otimes \mathbbm{1}) \mathbf{D}
\end{align*}
where $\vech: \mathcal{S}_{n} \to \mathbb{R}^{\frac{n(n+1)}{2}}$ denotes the half-vectorization operation which stacks the upper-triangular elements of a symmetric matrix in $\mathbb{R}^{n \times n}$, $\mathbf{D^{\dagger}} \in \mathbb{R}^{\frac{n(n+1)}{2} \times n}$ is the elimination-matrix which converts the vectorization of a matrix to its half-vectorization, $\mathbf{D} \in \mathbb{R}^{n \times \frac{n(n+1)}{2}}$ is the unique duplication-matrix which transforms the half-vectorization of a matrix to its vectorization,  and $\otimes$ denotes the Kronecker product. $\vect: \mathbb{R}^{m\times n} \to \mathbb{R}^{mn}$ denotes the vectorization, or matrix column-stacking operation. It is shown that using this covariance, and with a fixed prior uncertainty given by $\mathbf{P}(t_{0})$, the final covariance in the forward-direction for a state is,
$$\mathbf{P}(t) = e^{\mathbf{A}(t-t_{0})} \mathbf{P}(t_{0}) e^{\mathbf{A}^{T} (t-t_{0})} + \mathbf{Q}(t).
$$
Analogous forms are given in the Appendix for the backwards direction. Relying on both of these solutions, we extend the filtering, smoothing, and EM procedure of the discrete-time Kalman filter to continuous-time.

\begin{figure*}[t!]
\setkeys{Gin}{width=1\linewidth}
\begin{minipage}[t]{1\textwidth}
\includegraphics{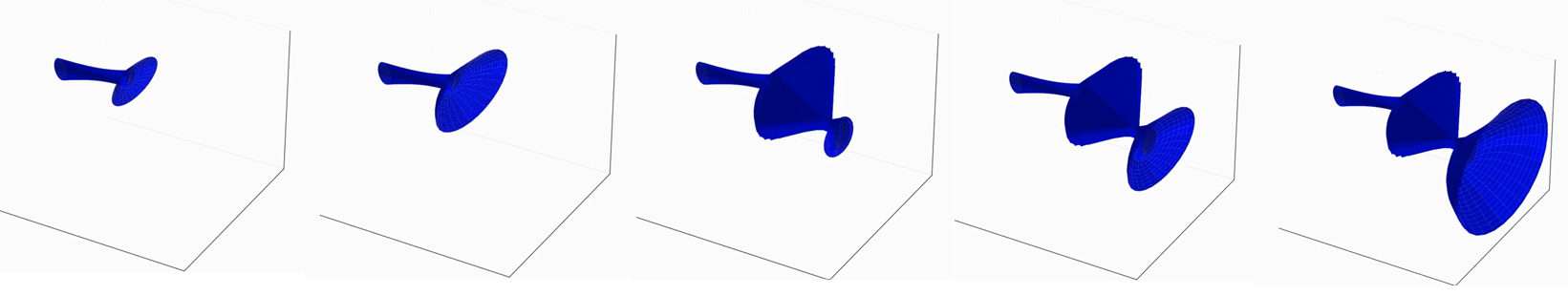} \label{fig:b}
\end{minipage}\hfill\caption{Isocontours of the time-dependent covariance $\bm{Q}(s)$ centered about the state-mean $\bm{x}(s)^{f/-} = e^{\bm{A}s}\bm{x}(t_{k-1})$. Snapshots show the increase in uncertainty of the time-dependent covariance between measurements and collapse to a certain state following measurement.}
\end{figure*}

\subsection{A two-filter form for the backwards pass in the continuous-discrete Kalman filter}\label{sec:EStep}
Smoothing for the Kalman filter, i.e., determining the posterior on a given state $\mathbf{x}_{k}$ given all observations $\mathbf{z}_{1:N}$, $\mathds{P}(\mathbf{x}_{k} | \mathbf{z}_{1:N})$ generally involves one pass of forward-filtering to determine $\mathds{P}(\mathbf{x}_{k}| \mathbf{z}_{1:k})$, and a backwards-pass reliant on the precomputed forwards pass, referred to as Rauch-Tung-Striebel (RTS) smoothing. We introduce an $\alpha, \beta$ form where the smoothed distribution is decomposed as $\mathds{P}(\mathbf{x}_{k} | \mathbf{z}_{1:N}) = \mathds{P}(\mathbf{x}_{k}| \mathbf{z}_{1:k}) \frac{\mathds{P}(\mathbf{z}_{k+1:K}|\mathbf{x}_{k})}{\mathds{P}(\mathbf{z}_{k+1:K})}$, which is computed with a recurrent computation of the likelihood $\frac{\mathds{P}(\mathbf{z}_{k+1:K}|\mathbf{x}_{k})}{\mathds{P}(\mathbf{z}_{k+1:K})}$, where one can compute updates in the reverse-direction in parallel with the forward-update of $\mathds{P}(\mathbf{x}_{k}| \mathbf{z}_{1:k})$. We are able to find a derivation for this owing to the structure of the continuous-time problem where one has that the inverse of the continuous solution $\left(e^{ \mathbf{A} s}\right)^{-1} \mathbf{x}(t_{k-1})=e^{- \mathbf{A} s} \mathbf{x}(t_{k-1})$ always exists. This can be evaluated for all finite $s$ as $|\det  e^{ -\mathbf{A} s} | = e^{-\Tr[\mathbf{As}]} >0$. As such, we compute a backwards pass likelihood \emph{analytically} in continuous time. 

The existing information-filter formulation for smoothing uses the information matrix $(\mathbf{P}_{k}^{b/+})^{-1}$, and the intermediate variable $(\mathbf{P}_{k}^{b/+})^{-1} \mathbf{\mu}_{k}^{b/+}$ to compute a backward pass in parallel with the forward-pass with numerical integration \cite{Lewis2017}. As such, our update can be directly used in information filtering as an analytical solution, yielding a parallel algorithm. Alternatively, one may use this formulation to compute the smoothed posterior following forward-filtering in the same order as RTS. In this work, we provide Algorithm~\ref{alg:backsmooth} as an alternative to RTS smoothing, and prove its equivalence. The posterior-updates in the reverse-direction involve a recursive determination of a Gaussian $\mathcal{N}(\mathbf{x}_{k} | \mathbf{\mu}_{k}^{b/+}, \mathbf{P}_{k}^{b/+})$ which is dependent on a future mean $\mathbf{\mu}_{k}^{b/+}$, and covariance $\mathbf{P}_{k}^{b/+}$. Letting $\mathbf{y} = \mathbf{x}_{k+1}$ and $\mathbf{Q}(\tau_{k+1}) = \mathbf{Q}_{k+1}$, the backward-direction likelihood, $N(\mathbf{x}_{k+1} | \mathbf{\mu}_{k+1}^{b/+}, \mathbf{P}_{k+1}^{b/+})$ is defined by the recurrence:
\vspace{-1mm}
 \begin{align*}
& \mathcal{N}(\mathbf{x}_{k} | \mathbf{\mu}_{k}^{b/+}, \mathbf{P}_{k}^{b/+}) \mathcal{N}(\mathbf{z}_{k+1})= \int \mathcal{N}(\mathbf{y} | e^{\mathbf{A} \tau_{k+1}} \mathbf{x}_{k}, \mathbf{Q}_{k+1})\mathcal{N}(\mathbf{z}_{k+1} | \mathbf{H y}, \mathbf{R}) \mathcal{N}(\mathbf{y} | \mathbf{\mu}_{k+1}^{b/+}, \mathbf{P}_{k+1}^{b/+}) d\mathbf{y} 
\end{align*}\label{eqn:recursive_main}

Proposition~\ref{prop:backwardsmoments} offers the solution to this recurrence.
\begin{prop}\label{prop:backwardsmoments}
    Given system parameters $(\mathbf{Q}_{c}, \mathbf{A}, \mathbf{H})$, and $\bm{V}(t)$, $\bm{Q}(t)$ as time-dependent covariance functions defined in \ref{eqn:Vt}, \ref{eqn:Qt}, then the recursive likelihood \ref{eqn:recursive_main} is distributed with the following mean $\mathbf{\mu}_{k+1}^{b/+}$ and covariance $\mathbf{P}_{k+1}^{b/+}$:
    \small
    \begin{align*}
        & \mathbf{\mu}_{k}^{b/+} = \E_{\sim \mathbf{x}_{k} | \mathbf{z}_{k+1:N}}[\mathbf{x}_{k}] \\
    &= e^{-\mathbf{A} \tau_{k+1}} \mathbf{\mu}_{k+1}^{b/+} + e^{-\mathbf{A} \tau_{k+1}} \mathbf{W}_{k+1} (\mathbf{z}_{k+1} - \mathbf{H \mu}_{k+1}^{b/+}) 
\end{align*}\label{eqn:backmean_main}
    \begin{align*}
    & \mathbf{P}_{k}^{b/+}= \E_{\sim \mathbf{x}_{k} | \mathbf{z}_{k+1:N}}\biggl[\left(\mathbf{x}_{k} - \mathbf{\mu}_{k}^{b/+}\right)\left(\mathbf{x}_{k} - \mathbf{\mu}_{k}^{b/+}\right)^{T}\biggr]  \\
&= e^{- \mathbf{A} \tau_{k+1}} \left( \mathbf{Q}(\tau_{k+1}) + \left( \mathbbm{1} - \mathbf{W}_{k+1} \mathbf{H} \right) \mathbf{P}_{k+1}^{b/+} \right) e^{-\mathbf{A}^{T} \tau_{k+1}} \\
&= \mathbf{V}(\tau_{k+1}) + e^{-\mathbf{A} \tau_{k+1}} \left( \mathbbm{1} - \mathbf{W_{k+1} H} \right) \mathbf{P_{k+1}^{b/+}} e^{-\mathbf{A^{T}} \tau_{k+1}}
\end{align*}\label{eqn:backcovar_main}\normalsize
For
$$\mathbf{W}_{k+1} = \mathbf{P}_{k+1}^{b/+} \mathbf{H}^{T} \left( \mathbf{H} \mathbf{P}_{k+1}^{b/+} \mathbf{H}^{T} + \mathbf{R} \right)^{-1}$$
a backwards-direction gain-matrix.
\end{prop}
The proof of Proposition~\ref{prop:backwardsmoments} is given in Appendix~\ref{proof:backwardsfilter}. We prove in Appendix~\ref{sect:equivalence} that this update is equivalent to the standard RTS set of updates for the posterior, for an appropriate initial condition $\mathbf{P}_{N-1}^{b/+}$ and $\mathbf{\mu}_{N-1}^{b/+}$.
\begin{prop}
    RTS smoothing offers identical first and second moments, $\E_{\sim \mathbf{x}_{k}|\mathbf{z}_{1:N}}{[\mathbf{x}_{k}]}$, $\E_{\sim \mathbf{x}_{k}|\mathbf{z}_{1:N}}{[\mathbf{x}_{k} \mathbf{x}_{k}^{T}]}$, and $\E_{\sim \mathbf{x}_{k},\mathbf{x}_{k-1}|\mathbf{z}_{1:N}}{[\mathbf{x}_{k} \mathbf{x}_{k-1}^{T}]}$ to the $\alpha, \beta$ form for an appropriate initial condition. The correspondence which relates the backwards-updates to standard RTS smoothing is
    \begin{align}
    \mathbf{P}_{k}^{s} =
    \left((\mathbf{P}_{k}^{b/+})^{-1} + (\mathbf{P}_{k}^{f/+})^{-1} \right)^{-1} \\
    \mathbf{\mu}_{k}^{s} = \mathbf{P}_{k}^{s} \left[(\mathbf{P}_{k}^{f/+})^{-1} \mathbf{\mu}_{k}^{f/+} + (\mathbf{P}_{k}^{b/+})^{-1} \mathbf{\mu}_{k}^{b/+} \right].
    \end{align}
\end{prop}
With final smoothed updates computed from the forward and backward pass, one may rely on the following updates for the moments of the smoothed-distribution, with the mean given above, and the cross-correlation
$$
\E_{\sim \mathbf{ x}_{k}, \mathbf{x}_{k-1}|\mathbf{z}_{1:N} }\left[\mathbf{x}_{k} \mathbf{x}_{k-1}^{T} \right]$$
$$ = \mathbf{P}_{k}^{s} (\mathbf{P}_{k}^{f/-})^{-1} e^{\mathbf{A} \tau_{k}} \mathbf{P}_{k-1}^{f/+} + \mathbf{\mu}_{k}^{s} (\mathbf{\mu}_{k-1}^{s})^{T}
$$
and the marginal autocorrelation
$$
\E_{ \mathbf{x}_{k}|\mathbf{z}_{1:N}}\left[\mathbf{x}_{k}\mathbf{x}_{k}^{T} \right] = \mathbf{P}_{k}^{s} + \mathbf{\mu}_{k}^{s} (\mathbf{\mu}_{k}^{s})^{T}.
$$
With all of these quantities defined, we express the new two-filter smoothing procedure in Algorithm~\ref{alg:backsmooth} in the appendix.
\subsection{M-step for the model parameters}
Unlike the discrete-time case, in continuous-time one does not have immediate closed-form updates for $\mathbf{A}$ and $\mathbf{Q}_{c}$. Instead, for $\mathbf{A}$ one must solve a non-linear regression problem involving matrix-exponentials. In the case of the differential covariance $\mathbf{Q}_{c}$, one finds the problem involves a non-linear optimization, and some structure has to be enforced to find an approximating least-squares update. Thus, given the smoothed-posterior over latent state $\mathbf{x}_{k}$ in the form of the moments $\E_{\sim \mathbf{x}_{k}|\mathbf{z}_{1:N}}{[\mathbf{x}_{k}]}$, $\E_{\sim \mathbf{x}_{k}|\mathbf{z}_{1:N}}{[\mathbf{x}_{k} \mathbf{x}_{k}^{T}]}$, and $\E_{\sim \mathbf{x}_{k},\mathbf{x}_{k-1}|\mathbf{z}_{1:N}}{[\mathbf{x}_{k} \mathbf{x}_{k-1}^{T}]}$, we introduce new M-step updates for $\mathbf{A}$ and $\mathbf{Q}_{c}$.

\textbf{The dynamics matrix.} For the dynamics matrix, we offer two sets of updates--one approximate and closed form, and one general involving a non-linear optimization which uses the solution to the former as an initial condition. We first introduce the approximate update for the dynamics-matrix $\mathbf{A}$, which is especially useful for the case in which $\mathbf{A}$ has a small spectral radius, $\rho(\mathbf{A}) < 1$, or the intervals $\tau_{k}$ are small. Defining $\Tr{\left[ \E[\mathbf{x}_{k-1}\mathbf{x}_{k-1}^{T}] \right]} = \sigma_{k-1}$, and the matrices:
$$\bm{M}_{k,k-1} = \E[\mathbf{x}_{k} \mathbf{x}_{k-1}^{T}] - \E[\mathbf{x}_{k-1} \mathbf{x}_{k-1}^{T}]$$
$$
\bm{Z}_{k-1,k-1} = \E[\mathbf{x}_{k-1} \mathbf{x}_{k-1}^{T}]
$$
we show one may approximate $\mathbf{A}$ in closed-form.
\begin{prop}\label{prop:approximatingA_main}
A second-order approximation for an optimization of the form:
\begin{equation*}
\max_{\mathbf{A} \in \mathbb{R}^{n \times n}} -\frac{1}{2} \sum_{k=2}^{N} \left\lVert \mathbf{x_{k}} - e^{\mathbf{A} \tau_{k}} \mathbf{x_{k-1}} \right\rVert_{\mathbf{Q}(\tau_{k})^{-1/2}}^{2} 
\end{equation*}
Under a smoothed-posterior over states $\mathds{P}_{\mathbf{x_{k}}, \mathbf{x_{k-1}} | \mathbf{z_{1:N}}}$, and assuming the commutator $[\mathbf{A}, \mathbf{Q}(\tau_{k})] = 0$ for all $k$, is given by the update:
\small
\begin{equation*}\label{eqn:Acommute_main}
\mathbf{A} = 
\left( 
\sum_{k=2}^{N}
\mathbf{Q}(\tau_{k})^{-1} \bm{M}_{k,k-1} 
\right)
\left(
\sum_{k=2}^{N} \tau_{k} \mathbf{Q}(\tau_{k})^{-1} \bm{Z}_{k-1,k-1} \right)^{-1}
\end{equation*}
\normalsize
In the case they do not commute, a second-order, least-squares approximation is given by:
\small
\begin{equation*}\label{eqn:A1}
\mathbf{A} = \left( \sum_{k=2}^{N} \tau_{k} \sigma_{k-1} \bm{M}_{k,k-1}
\right) \left( \sum_{k=2}^{N}
\tau_{k}^{2} \sigma_{k-1} \bm{Z}_{k-1,k-1} \right)^{-1} 
\end{equation*}
\normalsize
\end{prop}
These updates have the advantage of being fast, accurate for small spectral radius, and usable as an initial condition when $\mathbf{A}$ has any spectral radius. For the general problem, relying on \ref{eqn:A1} as an initial condition, we solve the full non-linear matrix exponential regression problem. We give an analytical form for the gradient of the expected log-likelihood \ref{eqn:exploglikeA} with respect to $\mathbf{A}$ and prove its convergence, enabling a numerical method such as Newton Conjugate-Gradient to optimize a curtailed form of the infinite-series gradient for a root. This is valuable with no restriction on the spectral radius of $\mathbf{A}$ or interval sizes, e.g. for more difficult systems where either the sample time $\tau_{k}$ is large or the spectral radius of $\mathbf{A}$, $\rho(\mathbf{A}) \gg 1$, is large. Defining $\mathbf{V}_{k} =  \mathbf{Q}(\tau_{k})^{-1} \left(\mathbb{E}[\mathbf{x}_{k}\mathbf{x}_{k-1}^{T}] - e^{\mathbf{A} \tau_{k}} \mathbb{E}[\mathbf{x}_{k-1}\mathbf{x}_{k-1}^{T}] \right) $, we use an infinite-series representation of the gradient of the log-likelihood, $\mathbb{E}_{\sim \mathbf{x}_{k}, \mathbf{x}_{k-1}}[\nabla_{\mathbf{A}} \ln{\mathds{P}(\mathbf{x}_{k} | \mathbf{x}_{k-1}, \mathbf{A}, \mathbf{Q_c}}, \tau_{k})]$
\begin{equation}\label{eqn:exploglikeA}
=
\sum_{r=0}^{\infty} \sum_{j=0}^{r} \sum_{k=2}^{N}
\frac{(\tau_{k})^{r+1}}{(r+1)!}(\mathbf{A^{T}})^{j} \mathbf{V}_{k}(\mathbf{A^{T}})^{r-j}
\end{equation}
and prove that \ref{eqn:A1} and \ref{eqn:Acommute} follow as fast and approximate solutions for the dynamics in \ref{prop:approximatingA}, and that the gradient is convergent in \ref{lemma:convergence}.
\begin{lemma}\label{lemma:convergence_main}

The infinite-series gradient $\mathbb{E}_{\sim \mathbf{x_{k}, x_{k-1}}}[\nabla_{\mathbf{A}} \ln{\mathds{P}(\mathbf{x_{k} | x_{k-1}, A, Q_c}, \tau_{k})}]$, under the auto and cross-correlations \ref{eqn:autocor} \ref{eqn:crosscor} computed by Algorithm~\ref{alg:backsmooth}, is convergent.
\end{lemma}
While the series-representation is useful for proving convergence and deriving approximates updates, we also introduce a fully-continuous optimization using \ref{eqn:A1} and \ref{eqn:Acommute_main} as initial conditions, generalizing the update for the SDE dynamics $\mathbf{A}$ for any system, with any spectral radius. This is done by relying on the continuous Fréchet derivative of the matrix-exponential, an operator from matrices $\mathbf{V} \in \mathbb{R}^{n \times n}$ to the derivative of the matrix exponential of $\mathbf{X}$ in the direction of increment $\mathbf{V}$. As such, it is defined by $\mathbf{V} \to \int_{0}^{\tau=1} e^{\mathbf{X} (\tau - s)} \mathbf{V} e^{\mathbf{X} s} ds$, and has existing numerical implementations \cite{AlMohy2009} we rely on in our conjugate-Newton update for an exceptionally effective general update. From this, we see that an alternative form for the gradient, naturally indexed by the time-step $\tau_{k}$, is given by:
$$\mathbb{E}_{\sim \mathbf{x_{k}, x_{k-1}}}[\nabla_{\mathbf{A}} \ln{\mathds{P}(\mathbf{x}_{k} | \mathbf{x}_{k-1}, \mathbf{A}, \mathbf{Q_c}, \tau_{k})}] =$$
\begin{equation*}
\sum_{k=2}^{N} \tau_{k} \int_{0}^{1} e^{\mathbf{A}^{T}\tau_{k} (1 - s)} \bm{V}_{k} e^{\mathbf{A}^{T} \tau_{k} s} ds,
\end{equation*}
For $\bm{V}_{k} = \mathbf{Q}(\tau_{k})^{-1} (\mathbb{E}[\mathbf{x}_{k}\mathbf{x}_{k-1}^{T}] - e^{\mathbf{A} \tau_{k}} \mathbb{E}[\mathbf{x}_{k-1}\mathbf{x}_{k-1}^{T}] )$ as before. Conjugate-Newton is then used with the analytical form of the expected log-likelihood for a given $\mathbf{A}$, using the efficient form given in~\ref{eqn:loglikeA}.

\textbf{The differential covariance.} By analyzing the ELBO, we determine a set of simultaneous equations which must be satisfied by a solution for the covariance $\mathbf{Q_{c}}$, finding that the problem can be framed as a least-squares problem in the continuous-time case. Enforcing the symmetric constraint of the covariance explicitly in the least-squares, an update is also given for the differential covariance, related to its half-vectorization by $\mathbf{Q_{c}} = \left(
\vect{\mathbbm{1}_{n}}^{T} \otimes \mathbbm{1}_{n}
\right) \left( \mathbbm{1}_{n} \otimes \mathbf{D} \vech{\mathbf{Q}_{c}} \right)$. We formulate our optimization in the $\frac{n(n+1)}{2}$-dimensional half-vectorization space of $\mathbf{Q_{c}} \in \mathcal{S}_{n}^{+}$, and in \ref{proof:Q} prove the following result.
\begin{prop}

The unique solution for the half-vectorization of the differential covariance, $\vech{\mathbf{Q_c}}$, of least norm is given by the solution to the linear system
$$
\mathbf{\Tilde{F}^{T} \Tilde{F}} \vech{\mathbf{Q_c}} = \mathbf{\Tilde{F}^{T} \Tilde{Z}}
$$
For block-matrices $\mathbf{\Tilde{F}}$, $\mathbf{\Tilde{Z}}$ defined by:
$$
\mathbf{\Tilde{F}} = 
\begin{bmatrix}
\mathbf{A_{P}}^{-1}(e^{\mathbf{A_{P}} \tau_{k}} - \mathbbm{1})\\
\end{bmatrix}_{k=2}^{N}
$$
$$
\mathbf{\Tilde{Z}} = \begin{bmatrix}
\vech{\E[(\mathbf{x_{k}} - e^{\mathbf{A} \tau_{k}} \mathbf{x_{k-1}})(\mathbf{x_{k}} - e^{\mathbf{A} \tau_{k}} \mathbf{x_{k-1}})^{\mathbf{T}}]}\\
\end{bmatrix}_{k=2}^{N}$$
Equivalently, the closed-form estimator of $\vech{\mathbf{Q_{c}}}$ is given as:
\begin{small}
\[\vech{\mathbf{Q_{c}}} = \frac1{N - 1} \biggl( \sum_{k=2}^{N} \left( e^{\mathbf{A_{P}} \tau_{k}} - \mathbbm{1} \right)^{-1} \mathbf{A_{P}} \vech{\mathbb{E}\left[(\mathbf{x}_{k} - e^{\mathbf{A} \tau_{k}} \mathbf{x}_{k-1})(\mathbf{x}_{k} - e^{\mathbf{A} \tau_{k}} \mathbf{x}_{k-1})^{T} \right]}  \biggr) \]
\end{small}
\end{prop}

From this, inverting the half-vectorization yields a symmetric covariance matrix which serves as the update for $\mathbf{Q}_{c}$, accounting for the innate variability of time-step intervals. A set of updates for the other parameters are given in Appendix~\ref{sect:otherparams}, and are identical to the standard discrete-time case. Given the E-step moments computed using the smoothing procedure of Algorithm~\ref{alg:backsmooth}, one can run the continuous-discrete EM using Algorithm~\ref{alg:cap} in the appendix.

\begin{figure*}[t!]
\setkeys{Gin}{width=0.9\linewidth}
\begin{minipage}[t]{0.5\textwidth}
\includegraphics{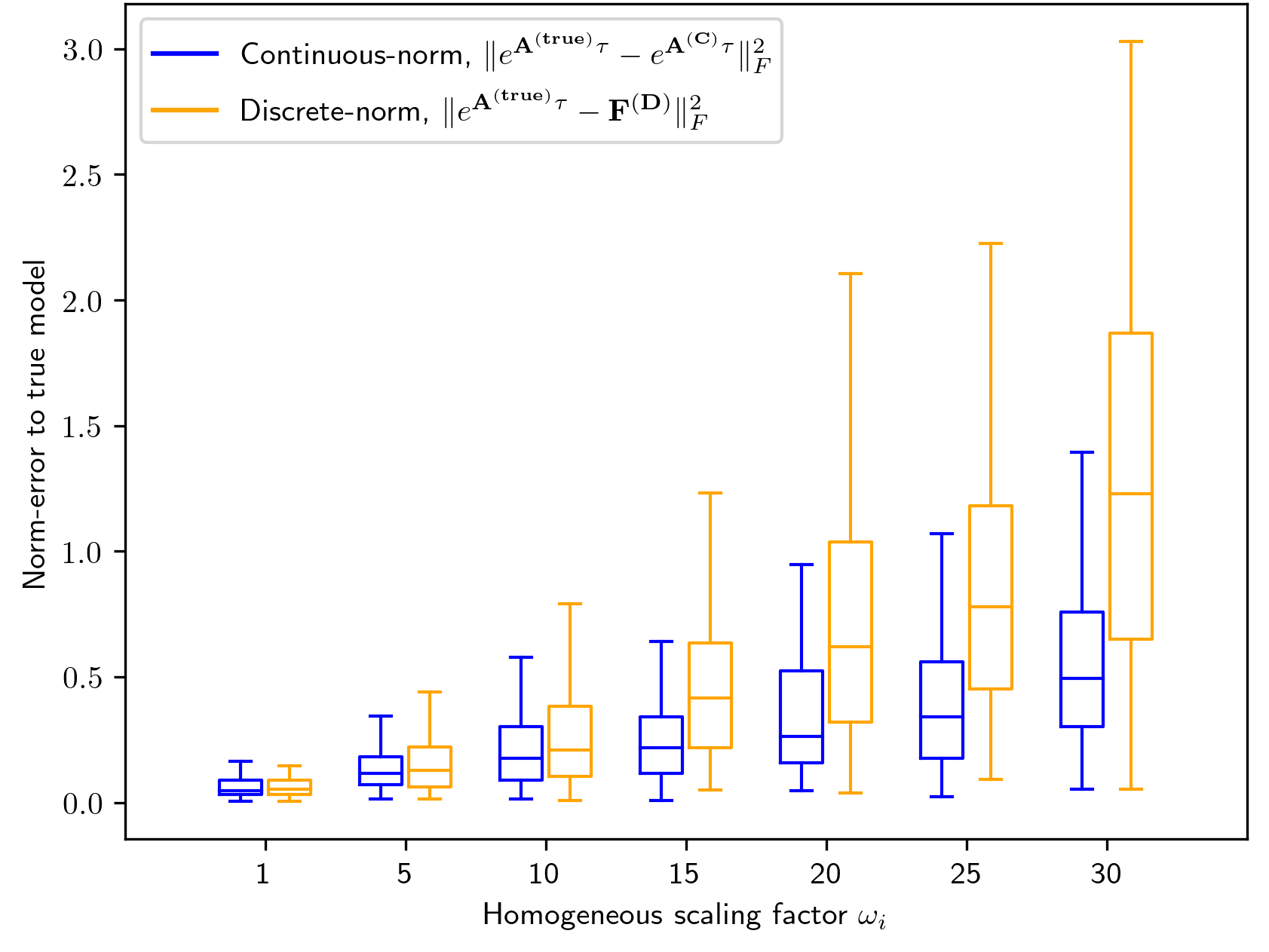} \label{fig:a}
\end{minipage}\hfill
\begin{minipage}[t]{0.5\textwidth}
\includegraphics{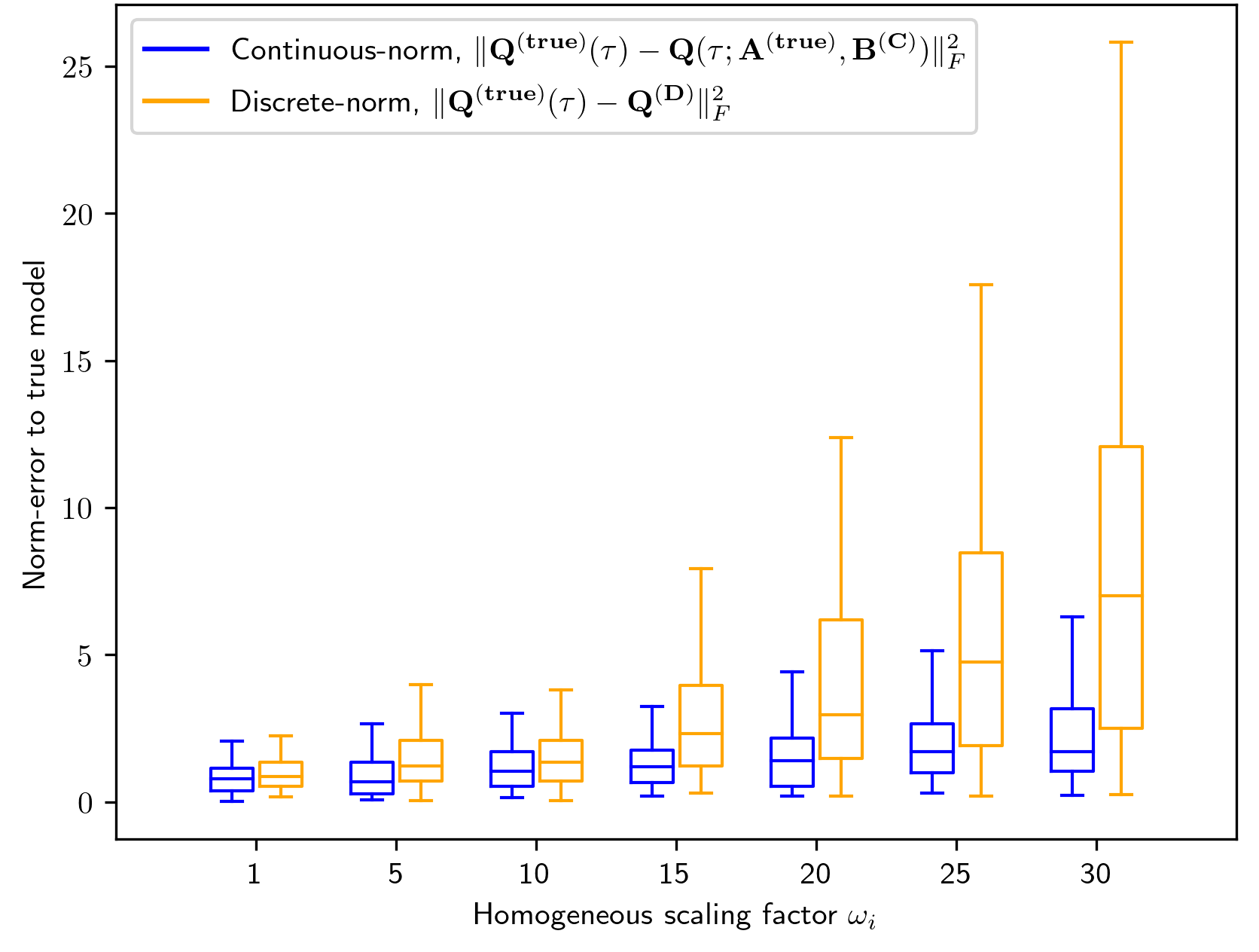} \label{fig:b}
\end{minipage}\hfill\caption{Comparison of model error for uniformly-random observation intervals, given by \ref{eqn:loss1} and \ref{eqn:loss2} in \textbf{a)} and \ref{eqn:loss3} and \ref{eqn:loss4} in \textbf{b)}. Discrete-time parameters were learned using \href{https://github.com/pykalman/pykalman}{pykalman} and the continuous-time parameters using Algorithm~\ref{alg:cap}. Both were capped at 100 EM iterations. Mean and error bars representing the inter-quartile (IQR) ranges and whiskers extending the box by 1.5 $\times$ IQR shown over 100 different simulations of the system. The spectral radius ranges from $\rho(\mathbf{A}) \approx 0.03$ to $\rho(30 \times \mathbf{A}) \approx 1$ (relatively stable).}\label{fig:example}
\end{figure*}

\section{Experimental Results}\label{sect:experiments}
Genetic circuits are a descriptive model of the dynamics of gene-regulation, with features such as multistability and oscillations that have been found, for example, in the bacteriophage $\lambda$-switch and the Cyanobacteria circadian oscillator \cite{Ishiura1998}. We simulate from the stochastic process of a genetic circuit, along with a transformation into a higher-dimensional space with substantial noise covariance in order to mimic observed, noisy recordings. We then apply our procedure to the simulated data to demonstrate \textit{learning} of the underlying dynamics. The promise of this method is to learn \textit{arbitrary} genetic circuit elements in an unsupervised way from noisy, high-dimensional datasets, paving the way for data-driven learning of latent dynamical systems controlling gene expression. Now, we provide some background on genetic circuits and the genetic-toggle switch, an established model for latent gene-regulation.

The toggle-switch circuit element is a bi-stable system which has been successfully been transfected into Escherichia coli (E. coli), with implications for programmable gene-therapy, genetic engineering, and biocomputing \cite{togswitch}. More recent work has demonstrated, using a feedback control loop in transfected E. coli, that the bistable toggle-switch can be maintained near its unstable equilibrium position for extended periods of time and can be made robust \cite{Lugagne2017}. We therefore consider an established synthetic model of a genetic toggle-switch, where we have two repressor proteins 1 and 2, at concentrations $r_{1}$ and $r_{2}$, which mutually repress each other. The model explicitly represents the mRNA concentration of each gene, denoted by $m_{1}$ and $m_{2}$. We rely on the mathematical/biophysical derivation, discussion, and parameters chosen in \cite{togswitch2} and \cite{Hausser2019}. From the central dogma of molecular biology, each protein is produced in a multi-step process (given here for protein 1):
\begin{enumerate}\itemsep0em 
    \item RNA-polymerase attaches to the promoter of gene $1$, with an affinity which depends on whether a repressor protein (like $r_{2}$) prevents binding.
    \item The RNA-polymerase transcribes the mRNA of protein $r_{1}$ at a rate $\alpha_{m}$ into the transcript $m_{1}$.
    \item $r_{1}$ is translated from $m_{1}$ at rate $\alpha_{p}$.
    \item The mRNA for gene 1 is degraded at a rate $\beta_{m}$, and the protein is degraded at a rate $\beta_{p}$.
\end{enumerate}

We assume gene 2 is is regulated in the same way, but with the roles of gene $1$ and $2$ switched. We define a probabilistic differential equation for both genes based on these transcription, translation, and degradation rates for an ensemble of mRNA molecules and proteins. A deterministic differential equation gives a continuous-form of the toggle switch. Non-linear functions $g_{R}(r_{1})$ and $g_{R}(r_{2})$ describe how $r_{1}$ affects the transcription for gene 2, and vice-versa. The promoter is allowed two operator sites for the repressors to bind independently, to repress transcription with a capacity $\omega$ which controls the fold-change of regulation, and with a dissociation constant $K_{R}$ between the repressor and promoter sequence. A biophysical argument yields the promoter activity function 
$$
g_{R}(r) = \frac{1 + \frac{q_{R}}{\omega}(2 + q_{R})}{(1 + q_{R})^{2}}
$$
with $q_{R} = \frac{r}{2 K_{R}}$. Expressed as a coupled system in terms of these rate-parameters, 
\begin{align}
    \frac{d m_{1}}{dt} &= \alpha_{m} g_{R}(r_{2}) - \beta_{m} m_1 \\
    \frac{d r_{1}}{dt} &= \alpha_{p} m_{1} - \beta_{p} r_{1} \\
    \frac{d m_{2}}{dt} &= \alpha_{m} g_{R}(r_{1}) - \beta_{m} m_2 \\
    \frac{d r_{2}}{dt} &= \alpha_{p} m_{2} - \beta_{p} r_{2}.
\end{align}
We can convert these deterministic rate equations into a form expressing discrete quantities, with a distribution evolving using a master equation. This can be extended to a stochastic system with fluctuations in the large number limit, with $\mathbf{\alpha}$ representing a vector of the macroscopic fluctuation of the molecular species, and $\Pi(\mathbf{\alpha}, t)$ representing a probability distribution for the fluctuations evolving according to the Fokker-Planck equation
$$
\frac{\partial \Pi}{ \partial t} = - \sum_{i,j} \mathbf{A}_{i,j} \frac{\partial}{\partial \alpha_{i}} (\alpha_{j} \Pi) + \sum_{i,j} \mathbf{B}_{i,j} \frac{\partial^{2}}{\partial\alpha_{i} \alpha_{j}} \Pi
$$
with $\mathbf{A}$ a matrix controlling the evolution of the fluctuation, $\mathbf{B}$ a matrix of diffusion coefficients representing local fluctuations of the molecular species. The expected fluctuation of $\alpha$, $\langle \alpha \rangle$ (with $\langle . \rangle$ denoting expectation), has the evolution
$$
\frac{d\langle \alpha \rangle}{dt} = \mathbf{A} \langle \alpha \rangle
$$
where the dynamics matrix can be shown to be
\begin{equation}\label{eqn:A}
\mathbf{A} =
\begin{bmatrix}
    -\beta_{p} & b \alpha_{m} g'_{R}(r_{2}) \\
    b \alpha_{m} g'_{R}(r_{1}) & - \beta_{p}
\end{bmatrix}
\end{equation}
for $b = \frac{\alpha_{p}}{\beta_{m}}$, and $g_{R}'(r) = \left(\frac{8(1-\omega)}{\omega}\right) \frac{K_{R}^{2}}{(r + 2 K_{R})^{3}}$. The covariance of $\alpha$, $\mathbf{C}(t)$ can be shown to evolve as
$$
\frac{d\mathbf{C}}{dt} = \mathbf{A C} + \mathbf{C A^{T}} + \mathbf{B}
$$
for a matrix $\mathbf{B} = \mathbf{S} \diag(\nu) \mathbf{S^{T}}$ of diffusion coefficients related to reaction propensities $\nu$ and a stoichiometry matrix $\mathbf{S}$. In the toggle-switch, we have
\begin{equation}
\label{eqn:B}
\mathbf{B} = \begin{bmatrix}
    b^{2} \alpha_{m} g_{R}(r_{2}) + \beta_{p} r_{1} & 0 \\
    0 & b^{2} \alpha_{m} g_{R}(r_{1}) + \beta_{p} r_{2}
\end{bmatrix}
\end{equation}
We also focus on the case that the transformation is identifiable and fix $\mathbf{H}$, $\mathbf{R}$, $\mathbf{P}_{0}$, and $\mathbf{\mu}_{0}$ to validate the novel component of learning the latent SDE, which is the dynamics matrix $\mathbf{A}$ and the diffusion matrix $\mathbf{B}$ (referred to earlier as $\mathbf{Q_{c}}$). We demonstrate our method by learning $\mathbf{A}$ and $\mathbf{B}$ using noisy-measurements of a higher-dimensional, transformed system. This is well-motivated because as transcription factors $r_1$ and $r_2$ might regulate a large number of different genes in a cell (a phenomenon referred to as \textit{trans}-regulation). In such a case, one might have a transformation in $\mathbb{R}^{m \times 2}$ with substantial measurement noise that relates a latent toggle-switch system or genetic circuit unit to a high-dimensional time-varying RNA-seq dataset. We validate this model using a number of realistic rate parameters, shown in \ref{tab:params}. We use the values of the parameters given in \cite{togswitch2}, which are realistic and agree with the existing literature on experimental values. The model depends on $\alpha_{p}$, $\beta_{m}$ through their ratio alone in the parameter $b$. We also use the equilibrium point of the system in \cite{togswitch2}, given by repressor levels at the point $\begin{bmatrix} r_{1} & r_{2}\end{bmatrix}^{\mathbf{T}} = \begin{bmatrix} 11/5 & 341/5 \end{bmatrix}^{\mathbf{T}}$. In this case the non-linear SDE specified by \ref{eqn:A} and \ref{eqn:B} is linearized at the point, analogous to a locally-linearized set of dynamics over a region.

We consider two scenarios of observation time interval irregularity: 1) uniform (which is light tailed, i.e. somewhat regular) and 2) sampled from a heavy-tailed distribution.

\begin{figure*}\label{fig:example2}
\setkeys{Gin}{width=0.9\linewidth}
\begin{minipage}{0.5\textwidth}
\includegraphics{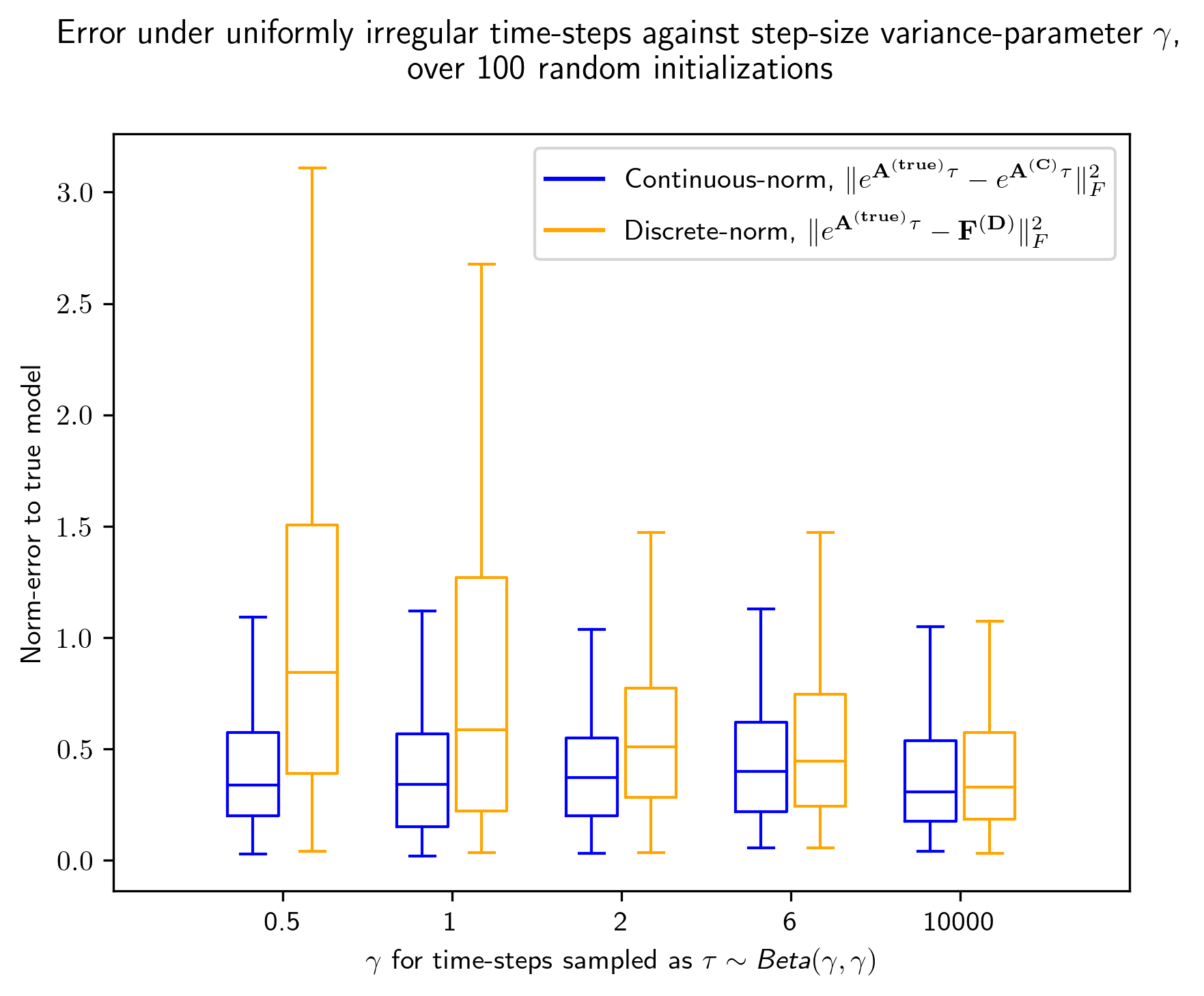} \label{fig:c}
\end{minipage}\hfill
\begin{minipage}{0.5\textwidth}
\includegraphics{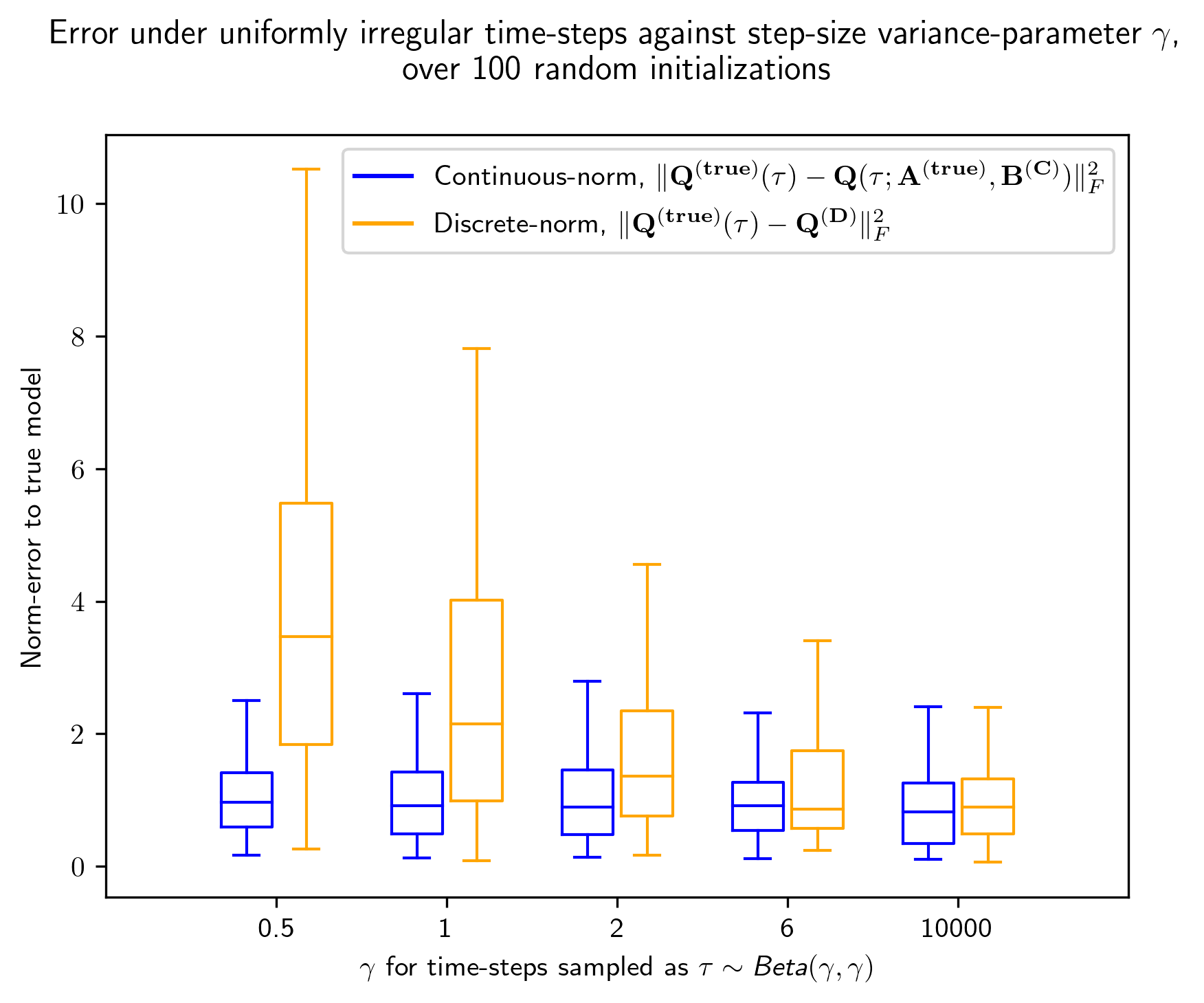} \label{fig:d}
\end{minipage}\hfill\caption{A comparison of the Frobenius-norm error for the example of Beta-distributed observation intervals with variance-controlling parameter $\gamma$. Loss given by \ref{eqn:loss1} and \ref{eqn:loss2} in \textbf{a)} and \ref{eqn:loss3} and \ref{eqn:loss4} in \textbf{b}. Parameters were learned with a maximimum of 100 EM iterations per simulation across 100 random datasets. The Beta distribution is scaled to be in the interval $[0,1/2]$ with 40 time-steps total so that the expected time is 10 minutes, and $\gamma$ ranges from $1/2$ to $10000$, $\mathbf{A}$ is homogeneously scaled so $\rho(\mathbf{A}) = 1$.}\label{fig:example2}
\end{figure*}

\textbf{Comparing the models under uniformly random intervals.} We consider the case where we have a fixed span of time over which the dynamics occur, $[0,T]$, and consider $N-1$ uniformly sampled points $u_{1}, \cdots, u_{N-1} \sim \mathcal{U}(0, T)$ which we sort and break $T$ into $N$ random intervals of lengths $\tau_{1},\tau_{2},...,\tau_{N}$. The discrete-time Kalman filter is learned using EM under the assumption that each observation is spaced with constant $\tau = \frac{T}{N}$, such that the uniformly random breakpoints are still distinct from the discrete-time assumption. It can be shown that a sample-complexity of $N > \frac{C}{a_\text{min}} \max\{\frac{\rho}{a_\text{min}}, 1\} p$ is required for learning a dense linear-SDE (non-latent) where $C$ is some constant, $p$ the dimension of the system, $0 < \rho \leq \lambda_\text{min}\left(-(\mathbf{A} + \mathbf{A^{*}})/2 \right)$, and $a_{\text{min}} \leq |\mathbf{A}_{ij}| p^{1/2}$ \cite{sdesample}. Thus, for some homogeneous scaling of the matrix $\mathbf{A}$ by a scalar factor $\omega$, the sample-complexity is inversely proportional to $\omega$. Owing to the presence of very small eigenvalues for $\mathbf{A}$, where the spectral radius $\rho(\mathbf{A}) \approx 0.034$ in the gene-circuit system and $\text{Re}{(\lambda_{i}(\mathbf{A}))} < 0$ for all eigenvalues $i$, it is a stable system.  To demonstrate the generality of our method to matrices with larger spectral radius, we scale $\mathbf{A}$ by different factors $\omega_{i}$ to demonstrate the difference in performance between the methods as the spectral radius increases. For each $\omega_{i}$, we maintain the ratio of the interval time to the number of samples $\frac{T}{N} = \frac{1}{2}$ so that the average interval is consistent across all examples, but matrices with larger spectral radius rely on fewer samples and those with smaller are trained on more. In particular, for $\omega_{i} \in \{ 1, 5, 10, 15, 20, 25, 30 \}$, we have total interval lengths $T_{i} \in \{ 100,70,60,50,40,30,20 \}$ seconds, and total samples $N_{i} \in \{200,140,120,100,80,60,40\}$. Given the mean-protein lifetime of 50 minutes, we fix the length of the intervals to be of analogous scale. Even for the smallest spectral radius $\rho(\mathbf{A}) \approx 0.034$ the continuous-time model slightly outperforms the discrete time model, and this improvement becomes far more substantial as $\rho(\mathbf{A})$ increases (Figure \ref{fig:example}).

\textbf{Varying the variance of the time-step distribution.} Here we sample observation intervals $\tau$ from Beta$(\tau|\gamma,\gamma)$. As $\gamma \to \infty$, the distribution approaches a delta function $\lim_{\gamma \to \infty} \text{Beta}(\tau|\gamma, \gamma) = \delta(\tau - \frac{1}{2})$, representing a constant step-size. For $\gamma = 1$, we have a uniform distribution over step-sizes, $\text{Beta}(1, 1) = \mathcal{U}(0,1)$. For $\gamma \to 0$, $\text{Beta}(\tau|\gamma, \gamma)$ approaches a Bernoulli distribution with two delta spikes at $\tau = 0$ and $\tau = 1$, corresponding to maximal coefficient of variation CV$=\sqrt{\text{var}[\tau]}/\E[\tau]$ (Figure \ref{fig:gammas}). We consider how varying $\gamma$ affects the capacity of the discrete and continuous-time model to learn the intrinsic dynamics of a system. For $N$ steps of the system we have $\E[T] = \E[\sum_{i=1}^{N} \tau_{i}] = \sum_{i=1}^{N} \E[\tau_{i}] = N \frac{\gamma}{2\gamma} = \frac{N}{2}$. As we vary $\gamma$  the total time that the system runs is the same in expectation as we choose only symmetric Beta-distributions. As the protein lifetime is $\sim50$ minutes, a reasonable $T$ to consider as a small interval of time-change for the system is 10 minutes, used in \ref{fig:example2}.

We compare the learning of the diffusion-coefficients $\mathbf{B}$, and the dynamics matrix $\mathbf{A}$ between our method and the discrete-time Kalman filter for $\gamma \in \{1/2, 1, 2, 6, 10000 \}$ (Figure \ref{fig:example2}). 
This example demonstrates that higher CV for the time intervals yields substantially higher model error in the discrete-time Kalman filter. In comparison, our continuous-time model keeps the error effectively constant for a fixed spectral radius $\rho(\mathbf{A})$ as the CV of the intervals changes. As the distribution approaches a deterministic delta-spike representing a discrete time-interval, the two converge to approximately the same error and variance.

\section{Discussion}\label{sect:discussion}
 In this work, we presented an algorithm for learning the continuous-time Kalman filter, a linear Fokker-Planck SDE. This enables the learning of latent, noisy continuous-time linear dynamical systems without any assumptions on observation interval regularity, and has superior performance relative to discrete-time LDS on continuous-time data. Recent works have addressed the learning of latent time-invariant gene regulation in the context of single-cell datasets \cite{Burdziak2023}, where noise and time-step irregularity in medical datasets presents a major challenge to time-series modeling \cite{Deshpande2022}. Future works might consider the generative noise and sparsity constraints of single-cell RNA-seq as in \cite{Burdziak2023}, in addition to continuous-time assumptions, as a means of learning LTI dynamics of real-world single-cell datasets. More generally, we hope that non-linear methods such as SLDS and SSMs which use discrete-time LDS blocks may extend this work to model sequential and time-series data in continuous-time.

\newpage

\bibliography{references}
\bibliographystyle{icml2024}

\newpage
\appendix
\onecolumn

\begin{figure}[h]
\centering
\includegraphics[scale=0.4]{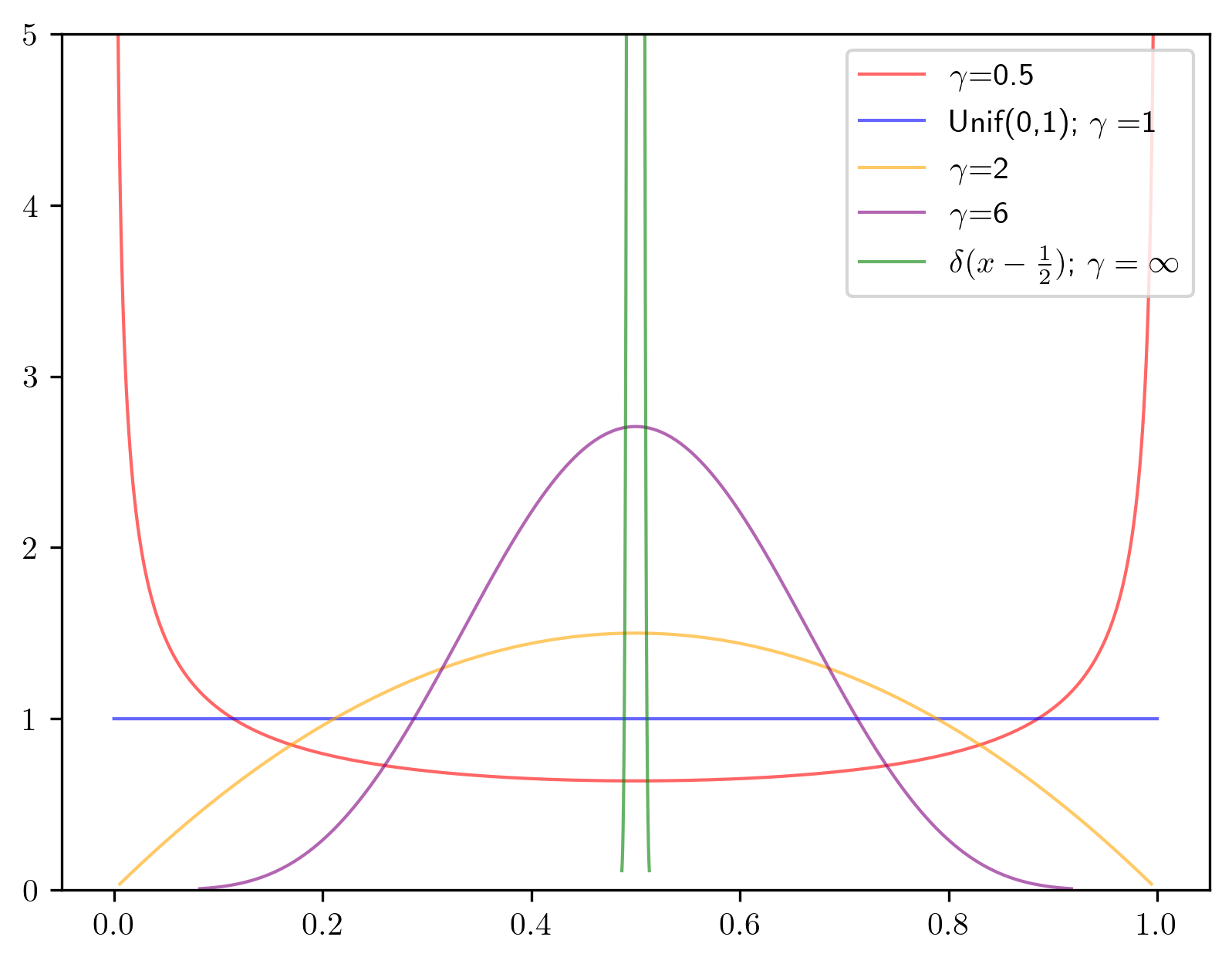}
\caption{Distributions for different values of the time-step variance-parameter $\gamma$, illustrating how varying $\gamma$ will alter the variance of the observation intervals $\tau \sim \textit{Beta}\left( \gamma, \gamma \right)$.}\label{fig:gammas}
\end{figure}

\section{Assessing Model Fit}

We compare the learning of the diffusion-coefficients $\mathbf{B}$ and the dynamics matrix $\mathbf{A}$ between our method and the discrete-time Kalman filter using the \href{https://github.com/pykalman/pykalman}{pykalman} package for an implementation of EM for the discrete-time Kalman filter. The discrete-time Kalman filter, as previously explained, is learned using EM under the assumption that each observation is spaced with constant $\tau = \frac{T}{N}$ for an interval of length $T$ with $N$ measurements. Therefore the divergence between the true dynamics and the discrete-time learned dynamics matrix $\mathbf{F^{(D)}}$ is
\begin{equation}\label{eqn:loss1}
\lVert e^{\mathbf{A^{(true)}} \tau} - \mathbf{F^{(D)}} \rVert_{F}^{2}
\end{equation}

Where the matrix-exponentiation gives the equivalent discrete-time solution for the dynamics, where the learning has been done continuously. The divergence between the true dynamics matrix $\mathbf{A^{(true)}}$ and the continuous-time learned dynamics matrix $\mathbf{A^{(C)}}$ is measured in an analogous manner as
\begin{equation}\label{eqn:loss2}
\lVert e^{\mathbf{A^{(true)}} \tau} - e^{\mathbf{A^{(C)} \tau}} \rVert_{F}^{2}
\end{equation}

The analogous loss for the true diffusion matrix $\mathbf{B^{(true)}}$ is also defined in terms of the integration of the continuous-time quantities, facilitating direct comparison to the discrete-time Kalman filter's learned parameters.  In order to compare the covariances in a unified way, we calculate $\mathbf{Q^{(true)}}(\tau) \triangleq \mathbf{Q}(\tau; \mathbf{A^{(true)}}, \mathbf{B^{(true)}})$ as a $\tau$-integrated true covariance, and compare to the fixed covariance of the discrete-time Kalman filter $\mathbf{Q^{(D)}}$ which assumes noise accumulated over the uniform step-size $\tau$ as
\begin{equation}\label{eqn:loss3}
\lVert \mathbf{Q^{(true)}}(\tau) - \mathbf{Q^{(D)}}\rVert_{F}^{2}
\end{equation}

As we are comparing learning of the covariances alone, the dynamics $\mathbf{A^{(true)}}$ is fixed as known in the continuous time case and $e^{\mathbf{A^{(true)}} \tau}$ is fixed as known in the discrete time case. We compare to the continuous-time learned differential covariance $\mathbf{B^{(C)}}$ by integrating it to $\tau$ to match that of the discrete-time Kalman filter, and compute
\begin{equation}\label{eqn:loss4}
\lVert \mathbf{Q^{(true)}}(\tau) -\mathbf{Q(\tau; \mathbf{A^{(true)}}, \mathbf{B^{(C)}})} \rVert_{F}^{2}
\end{equation}

As such, these errors all involve conversions of the continuous-time quantities to discrete-time ones to match those of the discrete-time Kalman filter, and evaluate how well a continuous-time formulation can learn the true underlying parameters when time-steps are irregular.

The biologically-realistic parameter values of \cite{togswitch2}, which are learned in our experiments, are given in the table below for reference:

\begin{center}\label{tab:params}
\begin{tabular}{||c | c | c | c ||} 
 \hline
 Step & Symbol & Description & Parameter Value \cite{togswitch2} \\ [0.5ex] 
 \hline\hline
 mRNA Transcription & $\alpha_{m}$ & Mean mRNA transcription rate & $0.2$ nM min\textsuperscript{-1} \\ 
 \hline
 mRNA Degradation & $\beta_{m}^{-1}$ & Rate of degradation & See $b$ \\
 \hline
 Protein Translation & $\alpha_{p}$ & Rate of translation & See $b$ \\
 \hline
 Protein Degradation & $\beta_{p}^{-1}$ & Mean protein lifetime & 50 min \\
 \hline
  Repressor-promoter dissociation constant & $K_{R}$ & Rate of repressor dissociation & 5 nM \\
 \hline
 Capacity & $\omega$ & Fold-change due to regulation & 200  \\
 \hline
 Promotor-parameter & $b$ & $b = \alpha_{p}/\beta_{m}$ & 10 nM \\
 \hline
\end{tabular}
\end{center}

\section{Algorithms~\ref{alg:backsmooth} and \ref{alg:cap}}

\begin{algorithm}
\caption{E-Step Backwards Smoothing}\label{alg:backsmooth}
\begin{algorithmic}
\Require Initialized $\mathbf{Q_{c}} \in \mathbb{R}^{n \times n}, \mathbf{Q_{c}} \succ 0$; $\mathbf{P_{0}} \in \mathbb{R}^{n \times n}, \mathbf{P_{0}} \succ 0$; $\mathbf{\mu_{0}} \in \mathbb{R}^{n}$; $\mathbf{A} \in \mathbb{R}^{n \times n}$; $\mathbf{H} \in \mathbb{R}^{m \times n}$
\State $k = N-1$
\State Initialize $\mathbf{P_{N}^{b/+}}$, $\mathbf{\mu_{N}^{b/+}}$
\While{$k \geq 1$}
    \State $W_{k+1} \gets \mathbf{P_{k+1}^{b/+} H^{T}} \left( \mathbf{H P_{k+1}^{b/+} H^{T}} + \mathbf{R} \right)^{-1}$
    \State $\mathbf{P_{k}^{b/+}} \gets e^{- \mathbf{A} \tau_{k+1}} \left( \mathbf{Q}(\tau_{k+1}) + \left( \mathbbm{1} - \mathbf{W_{k+1} H} \right) \mathbf{P_{k+1}^{b/+}} \right) e^{-\mathbf{A^{T}} \tau_{k+1}}$
    \State $\mathbf{\mu_{k}^{b/+}} \gets e^{-\mathbf{A} \tau_{k+1}} \mathbf{\mu_{k+1}^{b/+}} + e^{-\mathbf{A} \tau_{k+1}} \mathbf{W_{k+1}} (\mathbf{z_{k+1} - H \mu_{k+1}^{b/+}})$
    \If{$\mathbf{P_{k}^{f/+}}$, $\mathbf{\mu_{k}^{f/+}}$ Pre-computed}
    \State $\mathbf{P_{k}^{s}} \gets
((\mathbf{P_{k}^{b/+}})^{-1} + (\mathbf{P_{k}^{f/+}})^{-1})^{-1}$
    \State $\mathbf{\mu_{k}^{s}} \gets \mathbf{P_{k}^{s}} [(\mathbf{P_{k}^{f/+}})^{-1} \mathbf{\mu_{k}^{f/+}} + (\mathbf{P_{k}^{b/+}})^{-1} \mathbf{\mu_{k}^{b/+}}]$
    \State $\E_{\sim \mathbf{ x_{k}, x_{k-1}|z_{1:N} }}[\mathbf{x_{k} x_{k-1}^{T}}] \gets \mathbf{P_{k}^{s}} (\mathbf{P_{k}^{f/-}})^{-1} e^{\mathbf{A} \tau_{k}} \mathbf{P_{k-1}^{f/+}} + \mathbf{\mu_{k}^{s} (\mu_{k-1}^{s})^{T}}$
    \State $\E_{\sim \mathbf{x_{k}|z_{1:N}}}[\mathbf{x_{k}x_{k}^{T}}] \gets \mathbf{P_{k}^{s}} + \mathbf{\mu_{k}^{s} (\mu_{k}^{s})^{T}}$
    \State $k \gets k-1$
    \EndIf
\EndWhile
\end{algorithmic}
\end{algorithm}

\begin{algorithm}
\caption{Continuous-Discrete EM}\label{alg:cap}
\begin{algorithmic}
\Require Tolerance threshold $\varepsilon > 0$
\State Randomly Initialize or Fix: $\mathbf{Q_{c}} \in \mathbb{R}^{n \times n}, \mathbf{Q_{c}} \succ 0$; $\mathbf{P_{0}} \in \mathbb{R}^{n \times n}, \mathbf{P_{0}} \succ 0$; $\mathbf{\mu_{0}} \in \mathbb{R}^{n}$; $\mathbf{A} \in \mathbb{R}^{n \times n}$; $\mathbf{H} \in \mathbb{R}^{m \times n}$
\State $\log\mathcal{L}_{0} \gets \infty$, $i \gets 1$
\While{$|\log\mathcal{L}\left(\mathbf{z_{1:N}}| \Theta^{(i)} \right) - \log\mathcal{L}\left(\mathbf{z_{1:N}}| \Theta^{(i-1)} \right) | \geq \varepsilon$:}
    \State Compute $\mathbf{\mu_{k}^{s}}$, $\E_{\sim \mathbf{ x_{k}, x_{k-1}|z_{1:N} }}[\mathbf{x_{k} x_{k-1}^{T}}]$, $\E_{\sim \mathbf{x_{k}|z_{1:N}}}[\mathbf{x_{k} x_{k}^{T}}]$ from Smoother
    \State $\mathbf{\mu_{0}^{(i)}}$, $\mathbf{P_{0}^{(i)}}$ $\gets \E[\mathbf{x_{1}}]$, $\E[\mathbf{x_{1}x_{1}^{T}}] -\mathbf{\mu_{0} \mu_{0}^{T}}$
    \If{$[\mathbf{A}, \mathbf{Q}(\tau_{k})] = 0\text{ } \forall k$} \State /* Case 1: $\mathbf{A}$ and $\mathbf{Q}(\tau_{k})$ commute */
    \State $
\mathbf{A^{(i)}} \gets 
\left( 
\sum_{k=2}^{N}
\mathbf{Q}(\tau_{k})^{-1}\left(\E_{\sim x}[\mathbf{x_{k} x_{k-1}^{T}}] - \E_{\sim x_{k-1}}[\mathbf{x_{k-1} x_{k-1}^{T}}] \right) 
\right)
\left(
\sum_{k=2}^{N} \tau_{k} \mathbf{Q}(\tau_{k})^{-1} \E_{\sim x_{k-1}} [\mathbf{x_{k-1} x_{k-1}^{T}} ]\right)^{-1}
$
    \Else
    \State /* Case 2: General $\mathbf{A}$ and $\mathbf{Q}(\tau_{k})$ */
    \State $\mathbf{A^{(i)}} \gets \left( \sum_{k=2}^{N} \tau_{k}\Tr{\left[ \E[\mathbf{x_{k-1}x_{k-1}^{T}}] \right]} \left(
\E[\mathbf{x_{k} x_{k-1}^{T}}] - \E[\mathbf{x_{k-1} x_{k-1}^{T}}]
\right)
\right) \left( \sum_{k=2}^{N}
\tau_{k}^{2} \Tr{\left[\E[\mathbf{x_{k-1} x_{k-1}^{T}}]\right]} \E[\mathbf{x_{k-1} x_{k-1}^{T}}]
\right)^{-1}$
    \EndIf
    \If{Optimizing Gradient}
    \State /* Refining higher-order terms */
    \State $\mathbf{A^{(0)}} \gets \mathbf{A^{(i)}}$\Comment{Set either previous iterate of Newton or approximation above as initial condition}
    \State $\mathbf{A^{(i)}} \gets$ \textit{Newton-CG}($\mathbb{E}[\nabla_{\mathbf{A}} \ln{\mathds{P}(\mathbf{x_{k} | x_{k-1}, A, Q_c}, \tau_{k})}]$, $\mathbb{E}[\ln{\mathds{P}(\mathbf{x_{k} | x_{k-1}, A, Q_c}, \tau_{k})}]$, $\mathbf{A^{(0)}}$)
    \EndIf
    \State $\mathbf{\vech{\mathbf{Q_{c}}}^{(i)}}\gets\left( N - 1 \right)^{-1} \biggl( \sum_{k=2}^{N} \left( e^{\mathbf{A_{P}} \tau_{k}} - \mathbbm{1} \right)^{-1} \mathbf{A_{P}} \vech{\mathbb{E}_{\mathbf{x_{k},x_{k-1}|z_{1:N}}}\left[(\mathbf{x_{k}} - e^{\mathbf{A} \tau_{k}} \mathbf{x_{k-1}})(\mathbf{x_{k}} - e^{\mathbf{A} \tau_{k}} \mathbf{x_{k-1}})^{\mathbf{T}} \right]} \biggr)$
    \State $\mathbf{H^{(i)}}$, $\mathbf{R^{(i)}}$ $\gets \left( \sum_{k=1}^{N} \mathbf{z_{k}} \E[\mathbf{x_{k}}]^{\mathbf{T}} \right) \left( \sum_{k=1}^{N} \E[\mathbf{x_{k} x_{k}^{T}}] \right)^{-1}$, $N^{-1} \sum_{k=1}^{N} \E[\mathbf{(z_{k} - H x_{k})(z_{k} - H x_{k})^{T}}]$
    \State $i \gets i + 1$
\EndWhile
\end{algorithmic}
\end{algorithm}

\section{Solving the Lyapunov-type SDE}
\label{sec:SDEsolns}

Given the stochastic differential equation $\mathbf{\dot{x}}(t) = \mathbf{A x}(t) + \mathbf{w}(t)$, the homogeneous solution for the ODE is given as $\mathbf{x}(t) = e^{\mathbf{A}(t-t_{0})} \mathbf{x}(t_{0})$. In the general framework for the method of variation of parameters, the constant term above is reformulated as some function of t, $\mathbf{u}(t)$, which is to be determined such that it depends on both our homogeneous and particular solutions. We therefore assume
$$
\mathbf{x}(t) = e^{\mathbf{A}(t-t_{0})} \mathbf{u}(t) = e^{\mathbf{A}(t-t_{0})} (\mathbf{x}(t_{0}) + \mathbf{g}(t))
$$
$$
\frac{d}{dt} e^{\mathbf{A}(t-t_{0})} \mathbf{u}(t) = \mathbf{A} e^{\mathbf{A}(t-t_{0})} \mathbf{u}(t) + e^{\mathbf{A}(t-t_{0})} \frac{d \mathbf{u}(t)}{dt} = \mathbf{A x}(t) + e^{\mathbf{A}(t-t_{0})} \frac{d\mathbf{g}}{dt}
$$

From this, we can straightforwardly solve for $\mathbf{g}(t)$ using \ref{eqn:SDE1}, \ref{eqn:SDE2} from the expression
$$
\frac{d\mathbf{g}}{dt} = e^{-\mathbf{A}(t-t_{0})} \mathbf{w}(t)
$$

where we then integrate $\mathbf{g}(t)$ as
$$
\mathbf{g}(t) = \int_{t_{0}}^{t} e^{-\mathbf{A}(s - t_{0})} \mathbf{w}(s) ds = e^{\mathbf{A} t_{0}} \int_{t_{0}}^{t} e^{-\mathbf{A} s} \mathbf{w}(s) ds 
$$

Relying on this, we immediately arrive at the final form of the solution for the latent dynamics:
$$
\mathbf{x}(t) = e^{\mathbf{A}(t-t_{0})} (\mathbf{x} (t_{0}) + e^{\mathbf{A} t_{0}} \int_{t_{0}}^{t} e^{-\mathbf{A}s} \mathbf{w}(s) ds) = e^{\mathbf{A}(t-t_{0})} \mathbf{x}(t_{0}) + e^{\mathbf{A}t} \int_{t_{0}}^{t} e^{-\mathbf{A}s} \mathbf{w}(s) ds
$$

From this, if one discretizes the evolution between times $t_{k-1}$ to $t_{k}$, the solution for the state at time $t_{k}$ is given as:
$$
\mathbf{x_{k}} = e^{\mathbf{A} \tau_{k}}\mathbf{x_{k-1}} + \int_{t_{k-1}}^{t_{k}} e^{\mathbf{A}(t_{k}-s)} \mathbf{w}(s) ds
$$

Under the Gaussian white-noise assumption, the noise is assumed to obey a differential $d\mathbf{w}(s)$ dependent on time. The time-dependent equation becomes:
$$
\mathbf{x_{k}} = e^{\mathbf{A} \tau_{k}} \mathbf{x_{k-1}} + \int_{t_{k-1}}^{t_{k}} e^{\mathbf{A}(t_{k}-s)} d\mathbf{w}(s)
$$

We introduce the standard notation $f/-$ indicating an a priori estimate of the system state in the forward direction without measurement, and $f/+$ indicates the update given to the state with a measurement. Taking the expectation of this quantity and noting $d\mathbf{w}(s)$ is zero-mean and $e^{\mathbf{A}(t_{k}-s)}$ is a deterministic function of time $s$, one may invoke It\^{o}'s lemma to find:
$$
\mathbf{x_{k}^{f/-}} = \E_{\sim \mathbf{x_{k-1}^{f}}} \left[e^{\mathbf{A} \tau_{k}}\mathbf{x_{k-1}} + \int_{t_{k-1}}^{t_{k}} e^{\mathbf{A}(t_{k}-s)} d\mathbf{w}(s) \right] = e^{\mathbf{A} \tau_{k}} \mathbf{x_{k-1}^{f/+}}
$$

 The notation introduced also extends naturally to $b/-$, $b/+$ as the a priori and a posteriori states in the backwards-pass. By analogous reasoning in the backwards direction, we find:
$$
\mathbf{x_{k}^{b/-}} = e^{-\mathbf{A} \tau_{k+1}} \mathbf{x_{k+1}^{b/+}}
$$

We also seek an analytical solution to the forward and backward covariance dynamics of the SDE, following Axelsson et. al. \cite{cdkf_sols}. Given the differential equation:
$$
\mathbf{\dot{P}}(t) = \mathbf{A P}(t) + \mathbf{P}(t) \mathbf{A^{T}} + \mathbf{Q_c}
$$

The covariances $\mathbf{P}, \mathbf{Q_{c}} \in \mathcal{S}_{n}^{+}$ are defined by their $\frac{n(n+1)}{2}$ upper-triangular elements. Thus, we can convert this into a vector-differential equation by taking the half-vectorization. We introduce $\mathbf{D^{\dagger}} \in \mathbb{R}^{\frac{n(n+1)}{2} \times n}$ as the elimination-matrix which converts the vectorization of a matrix to its half-vectorization, and $\mathbf{D} \in \mathbb{R}^{n \times \frac{n(n+1)}{2}}$ as the unique duplication-matrix which transforms the half-vectorization of a matrix to its vectorization. The symbol $\otimes$ denotes the Kronecker product, $\vect: \mathbb{R}^{m\times n} \to \mathbb{R}^{mn}$ denotes the vectorization, or matrix column-stacking operation, and $\vech: \mathcal{S}_{n} \to \mathbb{R}^{\frac{n(n+1)}{2}}$ denotes the half-vectorization operation which stacks the upper-triangular elements of a symmetric matrix. Taking the half-vectorization of the Lyapunov-type ODE for the covariance matrix, we find:
$$
\vech{\mathbf{\dot{P}}} = \vech{\mathbf{AP}} + \vech{\mathbf{P A^{T}}} + \vech{\mathbf{Q}} = 
\mathbf{D^{\dagger}} \left(\vect{\mathbf{AP}} + \vect{\mathbf{P A^{T}}} \right) + \vech{\mathbf{Q_c}}
$$
Invoking the identity $\vect{\mathbf{AP}} = (\mathbbm{1} \otimes \mathbf{A}) \vect{\mathbf{P}}$ and $\vect{\mathbf{PA^{T}}} = (\mathbf{A} \otimes \mathbbm{1}) \vect{\mathbf{P}}$, we have that this equates to:
$$
\vech{\mathbf{\dot{P}}} = \mathbf{D^{\dagger}} \left( \left(\mathbbm{1} \otimes \mathbf{A} + \mathbf{A} \otimes \mathbbm{1} \right) \vect{\mathbf{P}} \right) + \vech{\mathbf{Q_c}}
$$
Using $\vect{\mathbf{P}} = \mathbf{D} \vech{\mathbf{P}}$, with $\mathbf{D}$ representing the duplication matrix:
$$
\vech{\mathbf{\dot{P}}} = \mathbf{D^{\dagger}} \left(\mathbbm{1} \otimes \mathbf{A} + \mathbf{A} \otimes \mathbbm{1} \right) \mathbf{D} \vech{\mathbf{P}} + \vech{\mathbf{Q_c}}
$$

Solving the above differential equation for the homogeneous solution, without the particular component for $\vech{\mathbf{Q_c}}$, will later allow us to determine the general solution via the method of variation of parameters.

First, as $\mathbf{D D^{\dagger}} = \mathbbm{1}$, one sees that \cite{matcalc}:
$$
e^{\mathbf{D^{\dagger}} (\mathbbm{1} \otimes \mathbf{A} + \mathbf{A} \otimes \mathbbm{1})t \mathbf{D}} = \mathbf{D^{\dagger}} e^{(\mathbbm{1} \otimes \mathbf{A} + \mathbf{A} \otimes \mathbbm{1})t} \mathbf{D} = \mathbf{D^{\dagger}} (e^{\mathbf{A}t} \otimes e^{\mathbf{A}t}) \mathbf{D}
$$

Following conventional notation, the matrix $\mathbf{D^{\dagger}} (\mathbbm{1} \otimes \mathbf{A} + \mathbf{A} \otimes \mathbbm{1}) \mathbf{D}$ will be referred to as $\mathbf{A_{P}}$.

In the differential equation, the solution for the homogeneous term is seen to be:
$$
\vech{\mathbf{P}(t)} = \mathbf{D^{\dagger}} \left(e^{\mathbf{A}(t-t_{0})} \otimes e^{\mathbf{A}(t-t_{0})} \right) \mathbf{D} \vech{\mathbf{P}(t_{0})} = \mathbf{D^{\dagger}} \left(e^{\mathbf{A}(t-t_{0})} \otimes e^{\mathbf{A}(t-t_{0})} \right) \vect{\mathbf{P}(t_{0})}$$
$$= \mathbf{D^{\dagger}} \vect{e^{\mathbf{A}(t-t_{0})} \mathbf{P}(t_{0}) e^{\mathbf{A^{T}} (t-t_{0})}} = \vech{e^{\mathbf{A}(t-t_{0})} \mathbf{P}(t_{0}) e^{\mathbf{A^{T}} (t-t_{0})}}
$$
Which implies the matrix form for the homogeneous solution is:
$$
\mathbf{P}(t) = e^{\mathbf{A}(t-t_{0})} \mathbf{P}(t_{0}) e^{\mathbf{A^{T}}(t - t_{0})}
$$

Arguing in the same vein, with $s = T - t$ in the reverse direction, the homogeneous solution has the form:
$$
\mathbf{P}(t) = e^{-\mathbf{A}(s-s_{0})}\mathbf{P}(s_{0})e^{-\mathbf{A^{T}}(s-s_{0})}
$$

Noting that this only reflects the effects of the transition matrix $\mathbf{A}$, and has yet to reflect the time-dependent contribution of accumulated noise as reflected in our $\mathbf{Q}(t)$ matrix, we solve for the general solution by variation of parameters. As before, this involves the assumption that $\mathbf{U}(t) = \mathbf{P}(t_{0}) + \mathbf{V}(t)$ is some time-dependent matrix function which depends on both the homogeneous and particular solutions to our Lyapunov differential equation.
$$
\mathbf{P}(t) = e^{\mathbf{A}(t-t_{0})} \mathbf{U}(t) e^{\mathbf{A^{T}}(t - t_{0})} = e^{\mathbf{A}(t-t_{0})} (\mathbf{P}(t_{0}) + \mathbf{V}(t)) e^{\mathbf{A^{T}} (t-t_{0})}$$

Substituting $\mathbf{Q}(t) = e^{\mathbf{A}(t-t_{0})} \mathbf{V}(t) e^{\mathbf{A^{T}} (t-t_{0})}$ as our solution to the accumulated motion noise, we have the following expression for the growth in our covariance
$$\mathbf{P}(t) = e^{\mathbf{A}(t-t_{0})} \mathbf{P}(t_{0}) e^{\mathbf{A^{T}} (t-t_{0})} + \mathbf{Q}(t)
$$
where, differentiating the form given by variation of parameters, we find the differential expression:
$$
\mathbf{\dot{P}} = \mathbf{AP} + \mathbf{PA^{T}} + e^{\mathbf{A}(t-t_{0})} \mathbf{\dot{V}} e^{\mathbf{A^{T}}(t - t_{0})}
$$
From this, we see that we must solve the following equation for $\mathbf{V}(t)$ to find the time-dependent noise term for the general a priori error covariance $\mathbf{P}(t)$:
$$
\mathbf{Q_c} = e^{\mathbf{A}(t-t_{0})} \mathbf{\dot{V}} e^{\mathbf{A^{T}}(t - t_{0})}
$$
As the matrix exponential is always invertible, we find from the above expression that
$$
\mathbf{\dot{V}} = e^{-\mathbf{A}(t-t_{0})} \mathbf{Q_c} e^{-\mathbf{A^{T}}(t - t_{0})}
$$
and, integrating the expression in time, find a solution for the anonymous function $\mathbf{V}(t)$ as:
$$
\mathbf{V}(t) = \int_{t_{0}}^{t} e^{-\mathbf{A}(s-t_{0})} \mathbf{Q_c} e^{-\mathbf{A^{T}}(s-t_{0})} ds
$$
We can provide a derivation for the form of the time-dependent covariance directly, noting the recent definition of $\mathbf{Q}$:
$$
\mathbf{Q}(t) = e^{\mathbf{A}(t-t_{0})} \mathbf{V}(t) e^{\mathbf{A^{T}} (t-t_{0})} = e^{\mathbf{A}(t-t_{0})}\left(\int_{t_{0}}^{t} e^{-\mathbf{A}(s-t_{0})} \mathbf{Q_c} e^{-\mathbf{A^{T}}(s-t_{0})} ds\right) e^{\mathbf{A^{T}} (t-t_{0})}$$
$$=
\int_{t_{0}}^{t} e^{\mathbf{A}(t-s)} \mathbf{Q_c} e^{\mathbf{A^{T}}(t-s)} ds
$$
To render this computationally simpler, rather than integrating the above equation directly, the original differential equation for $\dot{\mathbf{V}}$ can alternatively be re-expressed using vectorization identities to cast it into a vector differential equation amenable to a direct solution. This parallels the previous argumentation by the which the variation of parameters solution was derived for the state estimate. As such, we consider the half-vectorization of the ODE for $\dot{\mathbf{V}}$ as
$$
\vech{\mathbf{\dot{V}}} = \vech{e^{-\mathbf{A}(t-t_{0})}\mathbf{Q_c}e^{-\mathbf{A^{T}}(t-t_{0})}} = \mathbf{D^{\dagger}} \vect{e^{-\mathbf{A}(t-t_{0})} \mathbf{Q_c} e^{-\mathbf{A^{T}} (t - t_{0})}}
$$
$$
\vech{\mathbf{\dot{V}}} = e^{-\mathbf{D^{\dagger}} (\mathbbm{1} \otimes \mathbf{A} + \mathbf{A} \otimes \mathbbm{1} ) \mathbf{D} (t-t_{0})} \vech{\mathbf{Q_c}} = e^{-\mathbf{A_{P}} (t-t_{0})} \vech{\mathbf{Q_c}}
$$
from which, we may integrate the system as follows:
$$
\vech{\mathbf{V}(t)} = e^{\mathbf{A_{P}} t_{0}} \int_{t_{0}}^{t} e^{-\mathbf{A_{P}} s} ds \vech{\mathbf{Q_c}} = e^{\mathbf{A_{P}} t_{0}} \int_{t_{0}}^{t} \frac{d}{ds} \left( -\mathbf{A_{P}}^{-1} e^{-\mathbf{A_{P}} s} \right) ds \vech{\mathbf{Q_c}}
$$
The solution follows immediately from the above expression, giving us a half-vectorized time-dependent solution as:
\begin{equation}
\label{eqn:Vt}
\vech{\mathbf{V}(t)} = \mathbf{A_{P}}^{-1} \left(\mathbbm{1} - e^{-\mathbf{A_{P}}(t - t_{0})} \right) \vech{\mathbf{Q_c}}
\end{equation}
As we are relying on variation of parameters, we can use this to directly find the the general solution to the time-dependent noise covariance. Thus, we can straightforwardly use the solution for $\mathbf{V}(t)$ as
$$
\vech{\mathbf{Q}(t)} = \vech{e^{\mathbf{A}(t-t_{0})} \mathbf{V}(t) e^{\mathbf{A^{T}} (t - t_{0})}} = \mathbf{D^{\dagger}} \vect{e^{\mathbf{A}(t-t_{0})} \mathbf{V}(t) e^{\mathbf{A^{T}} (t - t_{0})}}
$$
where, after the application of more vectorization identities, we find:
$$
\vech{\mathbf{Q}(t)} = e^{\mathbf{D^{\dagger}} (\mathbbm{1} \otimes \mathbf{A} + \mathbf{A} \otimes \mathbbm{1}) \mathbf{D} (t-t_{0})} \vech{\mathbf{V}} = e^{\mathbf{A_{P}} (t-t_{0})} \vech{\mathbf{V}}
$$
$$
\vech{\mathbf{Q}(t)} = e^{\mathbf{A_{P}} (t-t_{0})} \mathbf{A_{P}}^{-1} \left(\mathbbm{1} - e^{-\mathbf{A_{P}}(t - t_{0})} \right) \vech{\mathbf{Q_c}}
$$
After distributing terms using the commutativity of the matrix exponential, we find the final form of the solution for our covariance-matrix function $\mathbf{Q}(t)$:
$$
\vech{\mathbf{Q}(t)} = \mathbf{A_{P}}^{-1} \left(e^{\mathbf{A_{P}}(t-t_{0})} - \mathbbm{1} \right) \vech{\mathbf{Q_c}}
$$
The reasoning in the backwards direction for the differential equation is no different. To be explicit, the solution is:
\begin{equation}
\label{eqn:Qt}
\vech{\mathbf{Q}(s)} = \mathbf{A_{P}}^{-1} \left(\mathbbm{1} - e^{-\mathbf{A_{P}}(s-s_{0})} \right) \vech{\mathbf{Q_c}}
\end{equation}
Also being explicit about the exact form the time-dependent covariance function $\mathbf{Q}(t)$ takes, one may use the well-defined inverse of the vectorization operation to give a matrix-form for the solution as:
$$
\mathbf{Q}(t) = \vect^{-1}{\biggl( \mathbf{D} \vech{\mathbf{Q}(t)} \biggr)} = \left(
\vect{\mathbbm{1}_{n}}^{T} \otimes \mathbbm{1}_{n}
\right) \left( \mathbbm{1}_{n} \otimes \mathbf{D} \vech{\mathbf{Q}(t)} \right)
$$
Between time points $t_{k-1}$ and $t_{k}$, we must re-integrate the differential equation between each set of measurements to reflect the a posteriori amendment to the error covariance. The discretized evolution in terms of a sample time difference $\tau_{k}$ is therefore given as:
$$
\vech{\mathbf{Q}(t_{k})} = \mathbf{A_{P}}^{-1} (e^{\mathbf{A_{P}} \tau_{k}} - \mathbbm{1}) \vech{\mathbf{Q_c}}
$$
With an analogous form given for the backwards direction, one uses both sets of covariance updates to specify a time-dependent uncertainty, e.g. whose eigenvalues $\lambda_{i}\left[\mathbf{Q}(t)\right]$ might be an increasing function of time and represent a covariance whose isocontours are expanding in space about the mean state $e^{\mathbf{A}\tau_{k}} \mathbf{x_{k-1}^{f/+}}$. This might express an increasing quantity of uncertainty as the system evolves and before a measurement occurs, and is useful in generalizing the constant covariance matrix $\mathbf{Q}$ found in the discrete-time Kalman filter to continuous-time systems with time-dependent uncertainty.

\section{Two-filter evaluation of the backwards posterior for the Continuous-Discrete Smoother}

Unlike in the standard Kalman filter, the continuous-discrete variant relies on a dynamics matrix $\mathbf{A}$ whose solution in reverse-time, given by the inverse matrix exponential $e^{- \mathbf{A} t} x_{0} = \left(e^{ \mathbf{A} t}\right)^{-1} x_{0}$, can be evaluated for all times $t$ as $|\det  e^{ \mathbf{A} t} | = e^{\Tr[\mathbf{At}]} >0$. Given this assumption, one can explicitly compute an analytical backwards-pass to evaluate the likelihood on $\mathbf{z_{k+1:K}}$ given $\mathbf{x_{k}}$ in a two-filter formulation which differs from the standard Rauch-Tung-Striebel formulation in that one may compute a backwards-pass independently from the result of the forwards pass. We offer a Bayesian derivation which shows this dual formulation is identical to the standard RTS formulation for the continuous-discrete filter for an appropriate initialization.
\begin{equation}
\label{eqn:twofilterlike}
\mathds{P}(\mathbf{x_{k}} | \mathbf{z_{1:N}}) = \frac{\mathds{P}(\mathbf{x_{k}, z_{1:N}})}{\mathds{P}(\mathbf{z_{1:N}})} = \frac{\mathds{P}(\mathbf{x_{k}, z_{1:k}})\mathds{P}(\mathbf{z_{k+1:K}}|\mathbf{x_{k}})}{ 
\mathds{P}(\mathbf{z_{1:K}})} = \mathds{P}(\mathbf{x_{k}}| \mathbf{z_{1:k}}) \frac{\mathds{P}(\mathbf{z_{k+1:K}}|\mathbf{x_{k}})}{\mathds{P}(\mathbf{z_{k+1:K}})} \end{equation}
As we had already determined $\mathds{P}(\mathbf{x_{k}}|\mathbf{z_{1:k}})$, we isolate the other term in the expression.
$$
\frac{\mathds{P}(\mathbf{z_{k+1:K}}|\mathbf{x_{k}})}{\mathds{P}(\mathbf{z_{k+1:K}}|\mathbf{z_{1:k}})} = \mathds{P}(\mathbf{z_{k+1:K}})^{-1} \int_{\mathbf{x_{k+1}}} \mathds{P}(\mathbf{z_{k+1:K}, x_{k+1}}|\mathbf{x_{k}}) d\mathbf{x_{k+1}}$$
$$\frac{\mathds{P}(\mathbf{z_{k+1:K}}|\mathbf{x_{k}})}{\mathds{P}(\mathbf{z_{k+1:K}} | \mathbf{z_{1:k}})} \mathds{P}(\mathbf{z_{k+1}}|\mathbf{z_{1:k}}) = \int_{\mathbf{x_{k+1}}} \mathds{P}(\mathbf{x_{k+1}} | \mathbf{x_{k}}) \mathds{P}(\mathbf{z_{k+1}} | \mathbf{x_{k+1}}) \frac{\mathds{P}(\mathbf{z_{k+2:K}} | \mathbf{x_{k+1}})}{\mathds{P}(\mathbf{z_{k+2:K}}|\mathbf{z_{1:k+1}})} d\mathbf{x_{k+1}}
$$
We see this involves a recursive determination for a likelihood on future measurements $\mathbf{z_{k+1:K}}$ given $\mathbf{x_{k}}$, normalized by a marginal density over $\mathbf{z_{k+1:K}|z_{1:k}}$.

Defining the parameters of this recursive term as
$$
\frac{\mathds{P}(\mathbf{z_{k+1:K}|x_{k}})}{\mathds{P}(\mathbf{z_{k+1:K}|z_{1:k}})} = \mathcal{N}( \mathbf{x_{k}} | \mathbf{x_{k}^{b/+}}, \mathbf{\mathbf{P_{k}^{b/+}}})
$$
and substituting the terms we determined previously, we find:
$$\mathcal{N}(\mathbf{x_{k}} | \mathbf{\mu_{k}^{b/+}}, \mathbf{P_{k}^{b/+}}) \mathcal{N}(\mathbf{z_{k+1}}) =$$
$$\int_{\mathbf{x_{k+1}}} \mathcal{N}(\mathbf{x_{k+1}} | e^{\mathbf{A} \tau_{k+1}} \mathbf{x_{k}}, \mathbf{Q}(\tau_{k+1})) \mathcal{N}(\mathbf{z_{k+1} | H x_{k+1}, R}) N(\mathbf{x_{k+1} | \mu_{k+1}^{b/+}, \mathbf{P_{k+1}^{b/+}}}) d\mathbf{x_{k+1}}
$$
Using this recursive definition, we seek to find a posterior, represented by $N(\mathbf{x_{k+1} | \mu_{k+1}^{b/+}, \mathbf{P_{k+1}^{b/+}}})$, which uses purely information from future states in order to derive a parallel $\alpha, \beta$ form for the Kalman filter.

\begin{prop}
    The backward-direction likelihood, $N(\mathbf{x_{k+1} | \mu_{k+1}^{b/+}, \mathbf{P_{k+1}^{b/+}}})$ which can be used for smoothing independent of the forward estimates by \ref{eqn:twofilterlike}, is defined by the recursive relation:
    \begin{multline}
    \label{eqn:recurrence}
    \mathcal{N}(\mathbf{x_{k}} | \mathbf{\mu_{k}^{b/+}}, \mathbf{P_{k}^{b/+}}) \mathcal{N}(\mathbf{z_{k+1}}) = \\
    \int_{\mathbf{x_{k+1}}} \mathcal{N}(\mathbf{x_{k+1}} | e^{\mathbf{A} \tau_{k+1}} \mathbf{x_{k}}, \mathbf{Q}(\tau_{k+1})) \mathcal{N}(\mathbf{z_{k+1} | H x_{k+1}, R}) N(\mathbf{x_{k+1} | \mu_{k+1}^{b/+}, \mathbf{P_{k+1}^{b/+}}}) d\mathbf{x_{k+1}}
    \end{multline} 
    The distribution of this likelihood, which depends on a future mean state $\mathbf{\mu_{k+1}^{b/+}}$ and future covariance $\mathbf{P_{k+1}^{b/+}}$, has the following first-moment:
    \begin{equation}
    \E_{\sim \mathbf{x_{k}} | \mathbf{z_{k+1:N}}}[\mathbf{x_{k}}] \triangleq \mathbf{\mu_{k}^{b/+}} = e^{-\mathbf{A} \tau_{k+1}} \mathbf{\mu_{k+1}^{b/+}} + e^{-\mathbf{A} \tau_{k+1}} \mathbf{W_{k+1}} (\mathbf{z_{k+1} - H \mu_{k+1}^{b/+}})
\end{equation}
    Additionally, it is distributed with the covariance:
    \begin{equation}\label{eqn:backcovar}
    \E_{\sim \mathbf{x_{k}} | \mathbf{z_{k+1:N}}}\biggl[\left(\mathbf{x_{k}} - \mathbf{\mu_{k}^{b/+}}\right)\left(\mathbf{x_{k}} - \mathbf{\mu_{k}^{b/+}}\right)^{\mathbf{T}}\biggr] \triangleq
\mathbf{P_{k}^{b/+}} = e^{- \mathbf{A} \tau_{k+1}} \left( \mathbf{Q}(\tau_{k+1}) + \left( \mathbbm{1} - \mathbf{W_{k+1} H} \right) \mathbf{P_{k+1}^{b/+}} \right) e^{-\mathbf{A^{T}} \tau_{k+1}}\end{equation}
$$= \mathbf{V}(\tau_{k+1}) + e^{-\mathbf{A} \tau_{k+1}} \left( \mathbbm{1} - \mathbf{W_{k+1} H} \right) \mathbf{P_{k+1}^{b/+}} e^{-\mathbf{A^{T}} \tau_{k+1}}
$$

For $\mathbf{V}(t)$, $\mathbf{Q}(t)$ as defined in \ref{eqn:Vt} and \ref{eqn:Qt} representing time-dependent functions for evaluating the covariance-matrix, and $\mathbf{W_{k+1}} = \mathbf{P_{k+1}^{b/+} H^{T}} \left( \mathbf{H P_{k+1}^{b/+} H^{T}} + \mathbf{R} \right)^{-1}$ representing a backwards-direction gain-matrix.
\end{prop}

\begin{proof}\label{proof:backwardsfilter}
From this, we find that the joint distribution of $\begin{bmatrix} \mathbf{z_{k+1}^{T}} & \mathbf{x_{k}^{T}} & \mathbf{x_{k+1}^{T}} \end{bmatrix}^{\mathbf{T}}$ is given with precision $\mathbf{\Lambda_{k}}$:
$$
\mathbf{\Lambda_{k}} =
\begin{bmatrix}
\mathbf{R^{-1}} & \mathbf{0} & \mathbf{R^{-1} H} \\
\mathbf{0} & e^{\mathbf{A^{T}} \tau_{k+1}} \mathbf{Q}(\tau_{k+1})^{-1} e^{\mathbf{A} \tau_{k+1}} & e^{\mathbf{A^{T}} \tau_{k+1}} \mathbf{Q}(\tau_{k+1})^{-1} \\ \mathbf{H^{T}R^{-1}} & \mathbf{Q}(\tau_{k+1})^{-1} e^{\mathbf{A} \tau_{k+1}} & \mathbf{Q}(\tau_{k+1})^{-1} + \mathbf{H^{T} R^{-1} H} + (\mathbf{P_{k+1}^{b/+}})^{-1}
\end{bmatrix}
$$
Taking the inverse of this symmetric block-precision matrix by recursive application of Schur's complement, we find the covariance $\mathbf{\Sigma_{k}}$ as:
$$
\mathbf{\Lambda_{k}}^{-1} = \mathbf{\Sigma_{k}} =
\begin{bmatrix}
\mathbf{R + H P_{k+1}^{b/+} H^{T}} & \mathbf{H P_{k+1}^{b/+}} e^{-\mathbf{A^{T}} \tau_{k+1}} & -\mathbf{H P_{k+1}^{b/+}}\\
e^{-\mathbf{A} \tau_{k+1}} \mathbf{P_{k+1}^{b/+} H^{T}} & e^{-\mathbf{A} \tau_{k+1}} (\mathbf{Q}(\tau_{k+1}) + \mathbf{P_{k+1}^{b/+}}) e^{-\mathbf{A^{T}} \tau_{k+1}} & -e^{-\mathbf{A} \tau_{k+1}} \mathbf{P_{k+1}^{b/+}} \\ 
- \mathbf{P_{k+1}^{b/+} H^{T}} & -\mathbf{P_{k+1}^{b/+}} e^{-\mathbf{A^{T}} \tau_{k+1}} & \mathbf{P_{k+1}^{b/+}} 
\end{bmatrix}
$$
While this can offer a covariance marginalized over $\begin{bmatrix} \mathbf{x_{k}} & \mathbf{z_{k+1}} \end{bmatrix}^{\mathbf{T}}$ or over $\begin{bmatrix} \mathbf{x_{k}} & \mathbf{x_{k+1}} \end{bmatrix}^{\mathbf{T}}$ of our joint distribution, we seek to marginalize out only the future latent state $\mathbf{x_{k+1}}$, to yield a joint distribution over observation $\mathbf{z_{k+1}}$ and state $\mathbf{x_{k}}$. We aim to factorize this as a prior over state $\mathbf{z_{k+1}}$ and a conditional on $\mathbf{x_{k}}$ linear with respect to measurement $\mathbf{z_{k+1}}$. As such, we find the conditional mean and covariance for $\mathbf{x_{k+1}}$:
$$
\mathbf{\Sigma_{x_{k+1}|x_{k}, z_{k+1}}} = \left(\mathbf{Q}(\tau_{k+1})^{-1} + \mathbf{H^{T} R^{-1} H} + (\mathbf{P_{k+1}^{b/+}})^{-1} \right)^{-1}
$$
$$
\mathbf{\mu_{x_{k+1}|x_{k}, z_{k+1}}} = \mathbf{\Sigma_{x_{k+1}|x_{k}, z_{k+1}}} \left[\mathbf{Q}(\tau_{k+1})^{-1}e^{\mathbf{A} \tau_{k+1}} \mathbf{x_{k}} + \mathbf{H^{T}R^{-1}z_{k+1}} + \mathbf{\left(P_{k+1}^{b/+} \right)^{-1} x_{k+1}^{b/+}} \right]
$$
We complete the square to yield a distribution $\mathcal{N}(\mathbf{x_{k+1} | \mu_{x_{k+1}|x_{k},z_{k+1}}, \Sigma_{x_{k+1}|x_{k}, z_{k+1}}})$ under the integral over the future state $\mathbf{x_{k+1}}$, and can integrate the distribution. This proportionality is expressed as: $\mathcal{N}(\mathbf{x_{k}} | \mathbf{x_{k}^{b/+}}, \mathbf{P_{k}^{b/+}}) \propto$
$$
\exp{\left( -\frac{1}{2} \mathbf{x_{k}^{T}} e^{\mathbf{A^{T}} \tau_{k+1}}\mathbf{Q}(\tau_{k+1})^{-1} e^{\mathbf{A} \tau_{k+1}} \mathbf{x_{k}} +\frac{1}{2} \mathbf{\mu_{x_{k+1}|x_{k},z_{k+1}}^{T} \Sigma_{x_{k+1}|x_{k}, z_{k+1}}^{-1} \mu_{x_{k+1}|x_{k},z_{k+1}} }\right)}$$
$$\times \int_{\mathbf{x_{k+1}}} \mathcal{N}(\mathbf{x_{k+1} | \mu_{x_{k+1}|x_{k},z_{k+1}}, \Sigma_{x_{k+1}|x_{k}, z_{k+1}}}) d\mathbf{x_{k+1}}
$$

So, integrating out the look-ahead state $\mathbf{x_{k+1}}$ one has the following proportionality relation:
$$
\mathcal{N}(\mathbf{x_{k} | x_{k}^{b/+}, \mathbf{P_{k}^{b/+}}}) \propto
\exp{\left( -\frac{1}{2} \mathbf{x_{k}^{T}} e^{\mathbf{A^{T}} \tau_{k+1}}\mathbf{Q}(\tau_{k+1})^{-1} e^{\mathbf{A} \tau_{k+1}} \mathbf{x_{k}} +\frac{1}{2} \mathbf{\mu_{x_{k+1}|x_{k},z_{k+1}}^{T} \Sigma_{x_{k+1}|x_{k}, z_{k+1}}^{-1} \mu_{x_{k+1}|x_{k},z_{k+1}}}\right)}
$$

More generally for the left-hand side of the recurrence relation expressed in \ref{eqn:recurrence}, we have that:
$$\mathcal{N}(\mathbf{x_{k} | x_{k}^{b/+}, \mathbf{P_{k}^{b/+}}}) \mathcal{N}(\mathbf{z_{k+1}}) \propto$$
$$\exp{\left( -\frac{1}{2} \mathbf{x_{k}^{T}} e^{\mathbf{A^{T}} \tau_{k+1}}\mathbf{Q}(\tau_{k+1})^{-1} e^{\mathbf{A} \tau_{k+1}} \mathbf{x_{k}} - \frac{1}{2} \mathbf{z_{k+1}^{T} R^{-1} z_{k+1}} +\frac{1}{2} \mathbf{\mu_{x_{k+1}|x_{k},z_{k+1}}^{T} \Sigma_{x_{k+1}|x_{k}, z_{k+1}}^{-1} \mu_{x_{k+1}|x_{k},z_{k+1}} }\right)}$$

Integrating the conditional distribution on $\mathbf{x_{k+1}}$ to 1, we see the result of the marginalization over $\mathbf{x_{k+1}}$ introduces correlations between observation $\mathbf{z_{k+1}}$ and latent state $\mathbf{x_{k}}$. Re-notating $\mathbf{\Sigma_{x_{k+1}|x_{k}, z_{k+1}}}$ as $\mathbf{G}$ for brevity, we find the precision of the new joint distribution over $\begin{bmatrix} 
\mathbf{z_{k+1}^{T}} & \mathbf{x_{k}^{T}} \end{bmatrix}^{\mathbf{T}}$ as:

$$
\mathbf{\Tilde{\Lambda}_{k}} = 
\begin{bmatrix}
\mathbf{R^{-1} - R^{-1} H G H^{T} R^{-1}} & \mathbf{R^{-1} H G} \mathbf{Q}(\tau_{k+1})^{-1} e^{\mathbf{A} \tau_{k+1}} \\
e^{\mathbf{A^{T}} \tau_{k+1}} \mathbf{Q}(\tau_{k+1})^{-1} \mathbf{G H^{T} R^{-1}} & e^{\mathbf{A^{T}} \tau_{k+1}}\mathbf{Q}(\tau_{k+1})^{-1} e^{\mathbf{A} \tau_{k+1}} - e^{\mathbf{A^{T}} \tau_{k+1}} \mathbf{Q}(\tau_{k+1})^{-1} \mathbf{G} \mathbf{Q}(\tau_{k+1})^{-1} e^{\mathbf{A} \tau_{k+1}}
\end{bmatrix}
$$

Focusing on the distribution over $\mathbf{x_{k}}$, we seek to calculate the state covariance over $\mathbf{x_{k}}$, $\mathbf{\Tilde{\Sigma}_{x_{k} | z_{k+1}}} = \mathbf{\Tilde{\Lambda}_{x_{k},x_{k}}}^{-1}$, using our precision over the joint distribution.

To find the distribution on $\mathbf{x_{k}}$, we find the marginal covariance on $\mathbf{x_{k}}$ and the mean on $\mathbf{x_{k}}$ given this covariance, in a form in which $\mathbf{z_{k+1}}$ has yet to be marginalized:

$$
\mathbf{\Lambda_{x_{k},x_{k}}}^{-1} = \left(e^{\mathbf{A^{T}} \tau_{k+1}}\mathbf{Q}(\tau_{k+1})^{-1} e^{\mathbf{A} \tau_{k+1}} - e^{\mathbf{A^{T}} \tau_{k+1}} \mathbf{Q}(\tau_{k+1})^{-1} \mathbf{G} \mathbf{Q}(\tau_{k+1})^{-1} e^{\mathbf{A} \tau_{k+1}} \right)^{-1}
$$
$$
= e^{-\mathbf{A} \tau_{k+1}} \left( \mathbf{Q}(\tau_{k+1})^{-1} - \mathbf{Q}(\tau_{k+1})^{-1} \mathbf{G} \mathbf{Q}(\tau_{k+1})^{-1} \right)^{-1} e^{-\mathbf{A^{T}} \tau_{k+1}}
$$
$$
= e^{-\mathbf{A} \tau_{k+1}}\left(\mathbf{Q}(\tau_{k+1}) + \left((\mathbf{P_{k+1}^{b/+}})^{-1} + \mathbf{H^{T}R^{-1}H} \right)^{-1} \right)e^{-\mathbf{A^{T}} \tau_{k+1}}
$$
Therefore, we have that
$$
\mathbf{P_{k}^{b/+}} = \mathbf{V}(\tau_{k+1}) + e^{-\mathbf{A} \tau_{k+1}} \left(\left( \mathbf{P_{k+1}^{b/+}} \right)^{-1} + \mathbf{H^{T}R^{-1}H } \right)^{-1}e^{-\mathbf{A^{T}} \tau_{k+1}} 
$$

and defining $((\mathbf{P_{k+1}^{b/+}})^{-1} + \mathbf{H^{T}R^{-1}H})^{-1} = \mathbf{S_{k+1}}^{-1} = \mathbf{P_{k+1}^{b/+}} - \mathbf{P_{k+1}^{b/+} H^{T} (H \mathbf{P_{k+1}^{b/+}} H^{T} + R)^{-1} H} = (\mathbbm{1} - \mathbf{W_{k+1} H}) \mathbf{P_{k+1}^{b/+}}$, we see this also equates to

$$
\mathbf{P_{k}^{b/+}} = e^{- \mathbf{A} \tau_{k+1}} \left( \mathbf{Q}(\tau_{k+1}) + \left( \mathbbm{1} - \mathbf{W_{k+1} H} \right) \mathbf{P_{k+1}^{b/+}} \right) e^{-\mathbf{A^{T}} \tau_{k+1}}$$
$$= \mathbf{V}(\tau_{k+1}) + e^{-\mathbf{A} \tau_{k+1}} \left( \mathbbm{1} - \mathbf{W_{k+1} H} \right) \mathbf{P_{k+1}^{b/+}} e^{-\mathbf{A^{T}} \tau_{k+1}}
$$

This recursively expresses the propagation of the covariance matrix expressing our uncertainty on the future states, dependent on a differential equation propagated backwards in time. The backwards-direction covariance $\mathbf{P_{k}^{b/+}}$, defining a marginal distribution on the current latent state given the future observations, naturally depends on the time-dependent covariance $\mathbf{Q}(\tau_{k+1})$ accumulated between the current and future latent state and the recurrent uncertainty on the future latent states expressed by $\mathbf{P_{k+1}^{b/+}}$, as one would expect.

Then, we find $\mathbf{\mu_{x_{k}}(z_{k+1})} \triangleq \mathbf{\mu_{k}^{b/+}}$ as:
\begin{multline}\nonumber
\mathbf{\mu_{x_{k}}(z_{k+1})} = 
e^{-\mathbf{A} \tau_{k+1}}\left( \mathbf{Q}(\tau_{k+1}) + \left((\mathbf{P_{k+1}^{b/+}})^{-1} + \mathbf{H^{T}R^{-1}H} \right)^{-1} \right) e^{-\mathbf{A^{T}} \tau_{k+1}} \\
\times \left[e^{\mathbf{A^{T}} \tau_{k+1}} \mathbf{Q}(\tau_{k+1})^{-1} \mathbf{G (H^{T} R^{-1} z_{k+1}} + (\mathbf{P_{k+1}^{b/+}})^{-1} \mathbf{\mu_{k+1}^{b/+}}) \right]
\end{multline}
$$
= e^{-\mathbf{A} \tau_{k+1}}\left(\mathbf{Q}(\tau_{k+1}) + \left((\mathbf{P_{k+1}^{b/+}})^{-1} + \mathbf{H^{T}R^{-1}H} \right)^{-1} \right) \mathbf{Q}(\tau_{k+1})^{-1} \mathbf{G \left(H^{T} R^{-1} z_{k+1} + (P_{k+1}^{b/+})^{-1} \mu_{k+1}^{b/+} \right) }
$$

To simplify the above expression, we see that:
$$
\left(\mathbf{Q}(\tau_{k+1}) + \mathbf{S_{k+1}}^{-1} \right) \mathbf{Q}(\tau_{k+1})^{-1} \left(\mathbf{Q}(\tau_{k+1})^{-1} + \mathbf{S_{k+1}} \right)^{-1}
$$
$$
= \left(\mathbf{Q}(\tau_{k+1}) + \mathbf{S_{k+1}}^{-1} \right) \mathbf{Q}(\tau_{k+1})^{-1} \left(\mathbf{Q}(\tau_{k+1}) - \mathbf{Q}(\tau_{k+1}) \left(\mathbf{S_{k+1}}^{-1} + \mathbf{Q}(\tau_{k+1}) \right)^{-1} \mathbf{Q}(\tau_{k+1}) \right)
$$
$$
= \left(\mathbf{Q}(\tau_{k+1}) + \mathbf{S_{k+1}}^{-1} \right) \left(\mathbbm{1} - \left(\mathbf{S_{k+1}}^{-1} + \mathbf{Q}(\tau_{k+1}) \right)^{-1} \mathbf{Q}(\tau_{k+1}) \right) = \mathbf{S_{k+1}}^{-1}
$$

Thus:

$$
\mathbf{\mu_{x_{k}}(z_{k+1})} = e^{-\mathbf{A} \tau_{k+1}} \left((\mathbf{P_{k+1}^{b/+}})^{-1} + \mathbf{H^{T}R^{-1}H} \right)^{-1} \mathbf{\left(H^{T} R^{-1} z_{k+1} + (P_{k+1}^{b/+})^{-1} \mu_{k+1}^{b/+} \right)}
$$

From this, after some simplification and defining a gain matrix as before, we find that:
$$= e^{-\mathbf{A} \tau_{k+1}} \mathbf{\mu_{k+1}^{b/+}} + e^{-\mathbf{A} \tau_{k+1}} \mathbf{W_{k+1}} (\mathbf{z_{k+1} - H \mu_{k+1}^{b/+}})
$$
\end{proof}
Thus, in the backwards direction the mean is given by the solution for a differential equation starting at the future mean state $\mathbf{\mu_{k+1}^{b/+}}$ propagated in the backwards direction added to a gain term propagated analogously. This gain-term involves an innovation error $\mathbf{\nu_{k}^{b}} = (\mathbf{z_{k+1} - H \mu_{k+1}^{b/+}})$ which depends on both the observation and the projection of the previous backwards-mean from the latent space into the observation space. This term reflects the measurement-update in the reverse-direction.

\section{Joint Posterior Marginals in a Two-Filter Form}\label{sect:equivalence}

We compare the $\alpha, \gamma$ (RTS) and newly introduced $\alpha, 
\beta$ (two-filter) form by first considering the joint-posterior marginal between states $\mathbf{x_{k-1}, x_{k}}$, and comparing how the results of the backwards posterior factor into the final joint-posterior, and then establishing equivalence with the established $\alpha, \gamma$ form.

\begin{prop}
    The $\alpha, \gamma$ form of smoothing offers identical first and second moments, $\E_{\sim \mathbf{x_{k}|z_{1:N}}}{[\mathbf{x_{k}}]}$, $\E_{\sim \mathbf{x_{k}|z_{1:N}}}{[\mathbf{x_{k} x_{k}^{T}}]}$, and $\E_{\sim \mathbf{x_{k},x_{k-1}|z_{1:N}}}{[\mathbf{x_{k} x_{k-1}^{T}}]}$ to the $\alpha, \beta$ form for an appropriate set of initial conditions, and is therefore equivalent to the $\alpha, 
    \beta$ form.
\end{prop}

\begin{proof}

We begin by considering the joint-distribution between a current and previous state $\mathbf{x_{k}}$ and $\mathbf{x_{k-1}}$, conditional on all of the data observations:
$$
\mathds{P}(\mathbf{x_{k-1},x_{k}|z_{1:N}}) = \frac{\mathds{P}(\mathbf{x_{k-1}, x_{k}, z_{1:N}})}{\mathds{P}(\mathbf{z_{1:N}})} = \frac{\mathds{P}(\mathbf{x_{k-1}, x_{k}, z_{1:N}})}{\mathds{P}(\mathbf{z_{1:N}})}$$
$$= \left( \frac{\mathds{P}(\mathbf{z_{k+1:N}|x_{k}})}{\mathds{P}(\mathbf{z_{k+1:N}|z_{1:k}})} \right) \left( \frac{\mathds{P}(\mathbf{z_{k}|x_{k}})\mathds{P}(\mathbf{x_{k}|x_{k-1}})\mathds{P}(\mathbf{x_{k-1}|z_{1:k-1}})}{\mathds{P}(\mathbf{z_{k}|z_{1:k-1}})} \right)
$$
Letting $\mathbf{P_{k}^{f/-}} = e^{\mathbf{A} \tau_{k}} \mathbf{P_{k-1}^{f/+}} e^{\mathbf{A^{T}} \tau_{k}} + \mathbf{Q}(\tau_{k})$, a quantity derived in the forward filter, we find:
$$= \left( \frac{\mathcal{N}(\mathbf{x_{k}| \mu_{k}^{b/+}, P_{k}^{b/+}})\mathcal{N}(\mathbf{z_{k} | H x_{k}, R})\mathcal{N}(\mathbf{x_{k}}|e^{\mathbf{A} \tau_{k}} \mathbf{x_{k-1}}, \mathbf{Q}(\tau_{k}))\mathcal{N}(\mathbf{x_{k-1}}|\mathbf{\mu_{k-1}^{f/+}}, \mathbf{P_{k-1}^{f/+}})}{\mathcal{N}(\mathbf{z_{k}}|\mathbf{H} e^{\mathbf{A} \tau_{k}} \mathbf{\mu_{k-1}^{f/+}}, \mathbf{H P_{k}^{f/-} H^{T} + R})} \right)
$$
\begin{multline}\nonumber
= \left(\mathcal{N}(\mathbf{x_{k}| \mu_{k}^{b/+}, P_{k}^{b/+}}) \mathcal{N}(\mathbf{x_{k-1}|\mu_{k-1}^{f/+}, P_{k-1}^{f/+}}) \mathcal{N}(\mathbf{x_{k}}|e^{\mathbf{A} \tau_{k}} \mathbf{x_{k-1}}, \mathbf{Q}(\tau_{k}))\right) \times \\
\left( \frac{\mathcal{N}(\mathbf{z_{k} | H x_{k}, R})}{\mathcal{N}(\mathbf{z_{k}}|\mathbf{H} e^{\mathbf{A} \tau_{k}} \mathbf{\mu_{k-1}^{f/+}}, \mathbf{H P_{k}^{f/-} H^{T} + R})} \right)
\end{multline}
We seek to simplify this to find a cross-covariance between states $\mathbf{x_{k-1}}$ and $\mathbf{x_{k}}$, which enables us to define a joint posterior in our Expectation-Maximization procedure. To handle the presence of distributions over $\mathbf{z_{k}}$, we aim to factorize the term on the right-hand side as a distribution over $\mathbf{z_{k}}$ and a distribution on $\mathbf{x_{k}}$.

$$\left( \frac{\mathcal{N}(\mathbf{z_{k} | H x_{k}, R})}{\mathcal{N}(\mathbf{z_{k}}| \mathbf{H} e^{\mathbf{A} \tau_{k}} \mathbf{\mu_{k-1}^{f/+}}, \mathbf{H P_{k}^{f/-} H^{T} + R})} \right)
$$

We find the conditional covariance on $\mathbf{x_{k}}$ given $\mathbf{z_{k}}$ as:
$$
\mathbf{\Sigma_{z_{k}|x_{k}}} = \left(\mathbf{R^{-1} - (H P_{k}^{f/-} H^{T} + R)^{-1}} \right)^{-1}
$$

And after collecting terms linear in $\mathbf{z_{k}}$, we find a conditional mean of:
$$
\mathbf{\mu_{z_{k}|x_{k}}} = \left(\mathbf{R^{-1} - \left(H P_{k}^{f/-} H^{T} + R \right)^{-1} } \right)^{-1} \left(\mathbf{R^{-1} H x_{k} - \left(H \mathbf{P_{k}^{f/-}} H^{T} + R \right)}^{-1} \mathbf{H} e^{\mathbf{A} \tau_{k}} \mathbf{\mu_{k-1}^{f/+}} \right)
$$

Defining a matrix $\mathbf{J_{k}} = (\mathbbm{1} - \mathbf{R (H P_{k}^{f/-} H^{T} + R})^{-1})^{-1} = (\mathbf{H P_{k}^{f/-} H^{T} + R}) (\mathbf{H \mathbf{P_{k}^{f/-}} H^{T}})^{-1}$, this simplifies to:
$$
\mathbf{H} e^{\mathbf{A} \tau_{k}} \mathbf{\mu_{k-1}^{f/+}} + \mathbf{J_{k} H} (\mathbf{x_{k}} - e^{\mathbf{A} \tau_{k}} \mathbf{\mu_{k-1}^{f/+}})
$$

As such, we have that the following for $\mathbf{z_{k}}$:
$$\left( \frac{\mathcal{N}(\mathbf{z_{k}} | \mathbf{H x_{k}}, \mathbf{R})}{\mathcal{N}(\mathbf{z_{k}}|\mathbf{H} e^{\mathbf{A} \tau_{k}} \mathbf{\mu_{k-1}^{f/+}}, \mathbf{H P_{k}^{f/-} H^{T} + R})} \right) \propto \mathcal{N}(\mathbf{z_{k}} | \mathbf{\mu_{z_{k}|x_{k}}}, \mathbf{\Sigma_{z_{k}|x_{k}}}) \exp\left(+\frac{1}{2}\mathbf{\mu_{z_{k}|x_{k}}^{T} \Sigma_{z_{k}|x_{k}}^{-1} \mu_{z_{k}|x_{k}}^{T}} \right)$$

Focusing on the distribution on $\mathbf{x_{k}}$ and particularly isolating the quadratic form relevant to the covariance on $\mathbf{x_{k}}$, we have:
$$
\left( \frac{\mathcal{N}(\mathbf{z_{k}} | \mathbf{H x_{k}}, \mathbf{R})}{\mathcal{N}(\mathbf{z_{k}} |\mathbf{H} e^{\mathbf{A} \tau_{k}} \mathbf{\mu_{k-1}^{f/+}}, \mathbf{H P_{k}^{f/-}} \mathbf{H^{T}} + \mathbf{R})} \right) \propto_{\mathbf{x_{k}}} \exp\left(-\frac{1}{2} 
\mathbf{x_{k}^{T} H^{T} R^{-1} H x_{k}} \right) + \Theta(\mathbf{x_{k}})
$$

From this, we see that the covariance of the joint distribution on $\mathbf{x_{k}}$ and $\mathbf{x_{k-1}}$ can be given as:
$$
\mathbf{\Lambda_{x_{k}, x_{k-1}}} =
\begin{bmatrix}
\mathbf{Q}(\tau_{k})^{-1} + (\mathbf{P_{k}^{b/+}})^{-1} + \mathbf{H^{T}R^{-1}H} & \mathbf{Q}(\tau_{k})^{-1} e^{\mathbf{A} \tau_{k}} \\
e^{\mathbf{A^{T}} \tau_{k}} \mathbf{Q}(\tau_{k})^{-1} & e^{\mathbf{A^{T}} \tau_{k}} \mathbf{Q}(\tau_{k})^{-1} e^{\mathbf{A} \tau_{k}} + (\mathbf{P_{k-1}^{f/+}})^{-1}
\end{bmatrix}
$$

Inverting the precision matrix, we find the following covariance across the joint distribution. We find that the marginal distribution on $\mathbf{x_{k}}$ is simply given by the smoothed covariance:
$$\mathbf{\left(\Sigma_{x_{k}, x_{k-1}} \right)_{1,1}} = \left((\mathbf{P_{k}^{b/+}})^{-1} + \left(\mathbf{Q}(\tau_{k}) + e^{\mathbf{A} \tau_{k}} \mathbf{P_{k-1}^{f/+}} e^{\mathbf{A^{T}} \tau_{k}} \right)^{-1} + \mathbf{H^{T}}\mathbf{R}^{-1}\mathbf{H} \right)^{-1}$$
$$= \left((\mathbf{P_{k}^{b/+}})^{-1} + (\mathbf{P_{k}^{f/+}})^{-1} \right)^{-1} = \mathbf{P_{k}^{s}} $$

Constructing the rest of the matrix and substituting $\mathbf{P_{k}^{s}}$ for the covariance-average of $\mathbf{P_{k}^{f/+}}$ and $\mathbf{P_{k}^{b/+}}$, we find:
$$
\mathbf{\Sigma_{x_{k},x_{k-1}}} =
\begin{bmatrix}
\mathbf{P_{k}^{s}} & [-\mathbf{P_{k}^{s}} (\mathbf{P_{k}^{f/-}})^{-1} e^{\mathbf{A} \tau_{k}} \mathbf{P_{k-1}^{f/+}}] \\
[-\mathbf{P_{k-1}^{f/+}} e^{\mathbf{A^{T}} \tau_{k}} (\mathbf{P_{k}^{f/-}})^{-1} \mathbf{P_{k}^{s}}] & [\mathbf{P_{k-1}^{f/+}} + \mathbf{P_{k-1}^{f/+}} e^{\mathbf{A^{T}} \tau_{k}} (\mathbf{P_{k}^{f/-}})^{-1} (\mathbf{P_{k}^{s}} - \mathbf{P_{k}^{f/-}}) (\mathbf{P_{k}^{f/-}})^{-1} e^{\mathbf{A} \tau_{k}} \mathbf{P_{k-1}^{f/+}}]
\end{bmatrix}
$$

Which gives the same form one would find in the $\alpha, \gamma$ derivation where the marginal covariance on $\mathbf{x_{k-1}}$ involves a recursive solution for the smoothed covariance using the results of the forward filter and $\mathbf{P_{k}^{s}}$ from the backwards direction representing the recursive RTS covariance. As such, the two-forms are equivalent up to an initial condition on $\mathbf{P_{N}^{b/+}}$.

To determine the marginal mean under the posterior, we complete the square over $\mathbf{x_{k-1}}$ and integrate the full joint distribution over $\mathbf{x_{k-1}}$ to leave a term linear in $\mathbf{x_{k}}$. In particular, we have the following proportionality in $\mathbf{x_{k-1}}$ from the full joint:

\begin{multline}\nonumber
\propto_{\mathbf{x_{k-1}}} \exp\left(-\frac{1}{2} \mathbf{x_{k-1}^{T}} \left( 
e^{\mathbf{A^{T}} \tau_{k}} \mathbf{Q}(\tau_{k})^{-1} e^{\mathbf{A} \tau_{k}} + \left( \mathbf{P_{k-1}^{f/+}} \right)^{-1} \right) \mathbf{x_{k-1}} \right. \\
\left.
 + 2 \mathbf{x_{k-1}^{T}} \left( e^{\mathbf{A^{T}} \tau_{k}} \mathbf{Q}(\tau_{k})^{-1} \mathbf{x_{k}} + \left( \mathbf{P_{k-1}^{f/+}}
\right)^{-1} \mathbf{\mu_{k-1}^{f/+}} \right)
\right)
\end{multline}
Defining:
$$\mathbf{\Lambda_{k-1,k-1}} =  e^{\mathbf{A^{T}} \tau_{k}} \mathbf{Q}(\tau_{k})^{-1} e^{\mathbf{A} \tau_{k}} + \left( \mathbf{P_{k-1}^{f/+}} \right)^{-1}$$
and
$$\mathbf{\mu_{x_{k-1}|x_{k}}} = \mathbf{\Lambda_{k-1,k-1}^{-1}}\left( e^{\mathbf{A^{T}} \tau_{k}} \mathbf{Q}(\tau_{k})^{-1} \mathbf{x_{k}} + \left( \mathbf{P_{k-1}^{f/+}}
\right)^{-1} \mathbf{\mu_{k-1}^{f/+}} \right)$$
we have:
$$
= \mathcal{N}(\mathbf{x_{k-1} | \mu_{x_{k-1|k}}, \Lambda_{k-1,k-1}^{-1}}) \exp\left(+\frac{1}{2} \mathbf{\mu_{x_{k-1}|x_{k}}^{T} \Lambda_{k-1,k-1}^{-1} \mu_{x_{k-1}|x_{k}}} \right)
$$
Thus, integrating over $\mathbf{x_{k-1}}$, the linear terms in the marginal over $\mathbf{x_{k}}$ are:
$$
2 \mathbf{x_{k}^{T}} \left(
\mathbf{Q}(\tau_{k})^{-1} e^{\mathbf{A} \tau_{k}} \mathbf{\Lambda^{-1}_{k-1,k-1}} \left(\mathbf{P_{k-1}^{f/+}}\right)^{-1} \mathbf{\mu_{k-1}^{f/+}} + \mathbf{H^{T} R^{-1} z_{k}} + \left(\mathbf{P_{k}^{b/+}}\right)^{-1} \mathbf{\mu_{k}^{b/+}}
\right)
$$
$$
= 2 \mathbf{x_{k}^{T}} \left( \left(\mathbf{P_{k}^{b/+}}\right)^{-1} \mathbf{\mu_{k}^{b/+}} +
\left( \mathbf{P_{k}^{f/-}} \right)^{-1} e^{\mathbf{A}\tau_{k}} \mathbf{\mu_{k-1}^{f/+}} + \mathbf{H R^{-1} z_{k}}
\right) = 2 \mathbf{ x_{k}^{T}} \left( \left(\mathbf{P_{k}^{b/+}}\right)^{-1} \mathbf{\mu_{k}^{b/+}} +
\left(\mathbf{P_{k}^{f/+}} \right)^{-1} \mathbf{\mu_{k}^{f/+}}
\right)
$$
Therefore, using the marginal covariance on $\mathbf{x_{k}}$, we find the mean of the marginal distribution on $\mathbf{x_{k}}$ is given by:
$$
\E_{\sim \mathbf{x_{k}|z_{1:N}}}[\mathbf{x_{k}}] = \mathbf{\mu_{k}^{s}} = \mathbf{P_{k}^{s}} \left[(\mathbf{P_{k}^{f/+}})^{-1} \mathbf{\mu_{k}^{f/+}} + (\mathbf{P_{k}^{b/+}})^{-1} \mathbf{\mu_{k}^{b/+}} \right]$$
$$= \left((\mathbf{P_{k}^{f/+}})^{-1} + (\mathbf{P_{k}^{b/+}})^{-1} \right)^{-1}\left[(\mathbf{P_{k}^{f/+}})^{-1} \mathbf{\mu_{k}^{f/+}} + (\mathbf{P_{k}^{b/+}})^{-1} \mathbf{\mu_{k}^{b/+}} \right]
$$

\end{proof}

 The $\alpha, \beta$ (or two-filter) form is therefore equivalent to the $\alpha, \gamma$ form in the continuous case, and one may analytically equate the two for an appropriate initialization of $\mathbf{P_{N}^{b/+}}$, $\mathbf{\mu_{N}^{b/+}}$. Therefore, if one were interested purely in the confidence estimates from future to previous states, without running both forward filtering and subsequent smoothing, one may accomplish it with the $\beta$-pass alone.

This yields the intuitive result that the smoothed mean is a covariance-weighted averaging of the forward and backward mean. These identically match the form of the smoothed mean and covariance calculated in the information-filter, generally presented in differential form, showing the assumptions are justified by a Bayesian investigation in the continuous-discrete case. Moreover, this shows that the backward-pass precision matrix of the information formulation corresponds to the inverse of the exact covariance $\mathbf{P_{k}^{b/+}}$ presented above, and can easily be computed analytically.

From $\mathbf{\left( \Sigma_{x_{k},x_{k-1}} \right)_{1,2}} = \E[\mathbf{(x_{k} - \mu_{k}^{s})(x_{k-1} - \mu_{k-1}^{s})^{T}}] = \mathbf{P_{k}^{s}} (\mathbf{P_{k}^{f/-}})^{-1} e^{\mathbf{A} \tau_{k}} \mathbf{P_{k-1}^{f/+}}$ we can find the auto-correlation between the current and previous state:
\begin{equation}\label{eqn:crosscor}
\E_{\sim \mathbf{ x_{k}, x_{k-1}|z_{1:N} }}[\mathbf{x_{k} x_{k-1}^{T}}] = \mathbf{P_{k}^{s}} (\mathbf{P_{k}^{f/-}})^{-1} e^{\mathbf{A} \tau_{k}} \mathbf{P_{k-1}^{f/+}} + \mathbf{\mu_{k}^{s} (\mu_{k-1}^{s})^{T}}
\end{equation}

And the marginal auto-correlation for a given state:
\begin{equation}\label{eqn:autocor}
\E_{\sim \mathbf{x_{k}|z_{1:N}}}[\mathbf{x_{k}x_{k}^{T}}] = \mathbf{P_{k}^{s}} + \mathbf{\mu_{k}^{s} (\mu_{k}^{s})^{T}}
\end{equation}

Our expectation-maximization procedure will depend on these second and first-order moments alone.

By equivalence with the $\alpha, \gamma$ form, we have that the covariance at the final state is:

$$
\mathbf{P_{N}^{s} = \left( 
\left(P_{N}^{f/+}\right)^{-1} + \left(P_{N}^{b/+}\right)^{-1}
\right)^{-1} = \left(P_{N}^{f/+}\right)^{-1}}
$$

This implies that the initial backward covariance $\mathbf{P_{N}^{b/+}} = e^{- \mathbf{A} 0} \left( \mathbf{Q}(0) + \left( \mathbbm{1} - \mathbf{W_{N+1} H} \right) \mathbf{P_{N+1}^{b/+}} \right) e^{- \mathbf{A^{T}} 0} = \left(\mathbbm{1} - \mathbf{W_{N+1} H} \right) \mathbf{P_{N+1}^{b/+}} = \left( \left(\mathbf{P_{N+1}^{b/+}}\right)^{-1} + \mathbf{H^{T} R^{-1} H} \right)^{-1}$ is undefined, unless we work with the information-matrix or introduce a prior on the final state. Therefore, assuming neither, the initial step for which the backward process is defined is the $(N-1)$\textsuperscript{th}, for $\mathbf{P_{N-1}^{b/+}}$ and $\mathbf{\mu_{N-1}^{b/+}}$. With an additional initialization of $\mathbf{\mu_{N}^{s} = \mu_{N}^{f/+}}$, as is standard, and one $\alpha, \gamma$ smoothing step, the initial conditions for the backward process are given as:
$$
\mathbf{P_{N-1}^{b/+} =
\left( 
\left( P_{N-1}^{s} \right)^{-1} - \left( P_{N-1}^{f/+} \right)^{-1}
\right)^{-1}}
$$
$$
\mathbf{
\mu_{N-1}^{b/+} = P_{N-1}^{b/+} \left(
\left( P_{N-1}^{s} \right)^{-1} \mu_{N-1}^{s} - \left( P_{N-1}^{f/+} \right)^{-1} \mu_{N-1}^{f/+}
\right)}
$$
From these, one may apply the $\alpha, \beta$ procedure to compute the backwards posterior for times $t_{k} \leq t_{N-1}$. It clearly holds that $\mathbf{P_{k}^{b/+}} \succ 0$ from the updates in \ref{eqn:backcovar}. This is clear as $\mathbf{F_{k+1}} = \left( \mathbf{Q}(\tau_{k+1}) + \left( \mathbbm{1} - \mathbf{W_{k+1} H} \right) \mathbf{P_{k+1}^{b/+}} \right) \succ \mathbf{0}$ for $\mathbf{Q}(\tau_{k+1}) \succ \mathbf{0}$ a covariance, $\mathbf{P_{k+1}^{b/+}} \succ \mathbf{0}$ assumed by induction, and $\mathbf{W_{k+1}}$ representing a Kalman-gain. Thus a unique square-root of $\mathbf{F_{k+1}} = \mathbf{F_{k+1}}^{1/2} \mathbf{F_{k+1}}^{1/2}$ exists, and $\mathbf{y^{T}}\mathbf{P_{k}^{b/+}} \mathbf{y} = \mathbf{y^{T}} e^{- \mathbf{A} \tau_{k+1}} \mathbf{F_{k+1}}^{1/2} \mathbf{F_{k+1}}^{1/2} e^{-\mathbf{A^{T}} \tau_{k+1}} \mathbf{y} = \langle\mathbf{F_{k+1}}^{1/2} e^{-\mathbf{A^{T}} \tau_{k+1}} \mathbf{y}, \mathbf{F_{k+1}}^{1/2} e^{-\mathbf{A^{T}} \tau_{k+1}} \mathbf{y} \rangle > 0$ holds $\forall \mathbf{y} \in \mathbb{R}^{n}$.

The appeal of the two-filter approach is that one may compute the forward and backward densities in parallel, which would imply that an initialization like the one above would limit the added speed-up of the two-filters. Potential resolution can come with either the use of the information-filtering framework with $(\mathbf{P_{N}^{f/+}})^{-1} = \mathbf{0}$ and the analytical updates for the information matrix implied above, with the introduction of a prior over the final state, or with the use of artificial normalizing distributions \cite{twofilter2} \cite{twofilternorm}. The information-filter, in particular, uses the information matrix $\left(\mathbf{P_{k}^{b/+}} \right)^{-1}$, and the intermediate variable $\left(\mathbf{P_{k}^{b/+}} \right)^{-1} \mathbf{\mu_{k}^{b/+}}$ to compute a backward pass independently (i.e. in parallel) from the forward-pass. The derivations above imply an analytical form for both, implying one can directly compute a forward and backward pass posterior in parallel \textit{and} in closed-form for the continuous-time case.

\section{Expectation maximization for the continuous-time parameters}

We write the full likelihood of the model as follows, without assuming that the number of time-points are equal, or assuming that the difference between the time points is the same.
$$
\mathds{P}(\mathbf{x_{1:N}, z_{1:N}| A, H, Q_c, R, P_{0}, \mu_{0}}, \tau_{1:N-1}) =
$$
$$\mathds{P}(\mathbf{x_{1}}| \mathbf{\mu_{0}, P_{0}}) \mathds{P}(\mathbf{z_{1} | x_{1}, R, H}) \left( \prod_{k=2}^{N} \mathds{P}(\mathbf{z_{k} | x_{k}, R, H}) \mathds{P}(\mathbf{x_{k}|x_{k-1}, A, Q_c}, \tau_{k}) \right)
$$
Taking the log-likelihood, we find:
$$\ln{\mathds{P}(\mathbf{x_{1:N}, z_{1:N}} | \mathbf{\Theta})} = $$
$$ \ln{\mathds{P}(\mathbf{x_{1}| \mu_{0}, P_{0}}, \tau_{1})} + \sum_{k=2}^{N} \ln{\mathds{P}(\mathbf{x_{k}|x_{k-1}, A, Q_c}, \tau_{k})} + \sum_{k=1}^{N} \ln{\mathds{P}(\mathbf{z_{k}|x_{k}, R, H})} $$
Under our EM-procedure, we introduce the analytically-derived marginal-posterior over $\mathbf{x_{k}|z_{1:N}}$ for the spatial auto-correlations $\mathbf{x_{k} x_{k}^{T}}$ and spatial means $\mathbf{x_{k}}$, as well as a joint-posterior $\mathbf{x_{k}, x_{k-1}|z_{1:N}}$ for the cross-correlations $\mathbf{x_{k} x_{k-1}^{T}}$. For each of the derivations below, we first rely on the proportionality of the expected log-likelihood in terms of each parameter $\theta$ of the model, and offer a derivation for a continuous-time update for the differential parameters of the SDE.

\section{An M-Step optimization for the homogeneous dynamics matrix}

In order to learn the dynamics matrix $\mathbf{A}$ reflected in the SDE, $\mathbf{\dot{x}}(t) = \mathbf{A x}(t) + \mathbf{w}(t)$, we assume an update which has been made for the current posterior relying on the discussion of smoothing above. With the posterior fixed, one must either find an approximation for $\mathbf{A}$ or solve a non-linear regression problem of the general form $\max_{\mathbf{A} \in \mathbb{R}^{n \times n}} -\frac{1}{2} \sum_{i=2}^{N} \lVert \mathbf{x_{i}} - e^{\mathbf{A} \tau_{i}} \mathbf{x_{i-1}} \rVert_{\mathbf{Q_{i}}^{-1/2}}^{2}$ for $\tau_{i} > 0$, $\mathbf{Q_{i}} \succ 0$ from the expected log-likelihood given as our objective. We offer a discussion of both, introducing a reasonable approximation (particularly for matrices with small spectral radius) which may also be used as an initial condition for a generalized update involving numeric methods. We also introduce a general update for arbitrary matrices $\mathbf{A}$ relying on an implementation of Fréchet derivative for the matrix-exponential.

\begin{prop}\label{prop3}
The expectation of the gradient $\mathbb{E}_{\sim \mathbf{x_{k}, x_{k-1}}}[\nabla_{\mathbf{A}} \ln{\mathds{P}(\mathbf{x_{k} | x_{k-1}, A, Q_c}, \tau_{k})}]$, whose root yields an M-step optimizer, is given by:
$$
\mathbb{E}_{\sim \mathbf{x_{k}, x_{k-1}}}[\nabla_{\mathbf{A}} \ln{\mathds{P}(\mathbf{x_{k} | x_{k-1}, A, Q_c}, \tau_{k})}] =$$
$$
\sum_{r=0}^{\infty} \sum_{j=0}^{r} \sum_{k=2}^{N} 
\frac{(\tau_{k})^{r+1}}{(r+1)!}(\mathbf{A}^{T})^{j} \left( \mathbf{Q}(\tau_{k})^{-1} (\E_{\sim\mathbf{x_{k}, x_{k-1}|z_{1:N}}}[\mathbf{x_{k} x_{k-1}^{T}}] - e^{\mathbf{A} \tau_{k}} \E_{\sim\mathbf{x_{k-1}|z_{1:N}}}[\mathbf{x_{k-1}} x_{k-1}^{T}]) \right) (\mathbf{A^{T}})^{r-j}
$$
\end{prop}

\begin{proof}
We find the optimization for the dynamics matrix reduces to the minimization of a sum of the following quadratic forms:
$$
\ln{\mathds{P}(\mathbf{x_{k}|x_{k-1}, A, Q_c}, \tau_{k})} \propto_{\mathbf{A}} -\frac{1}{2}(\mathbf{x_{k}} - e^{\mathbf{A} \tau_{k}} \mathbf{x_{k-1}})^\mathbf{{T}} \mathbf{Q}(\tau_{k})^{-1} (\mathbf{x_{k}} - e^{\mathbf{A} \tau_{k}} \mathbf{x_{k-1}})
$$
This involves a matrix-exponential non-linear regression for $\mathbf{A}$. We find the directional derivative of this term acting on $\mathbf{V}$ as:
$$
-\frac{1}{2}D \left( (\mathbf{x_{k}} - e^{\mathbf{A} \tau_{k}} \mathbf{x_{k-1}})^\mathbf{{T}} \mathbf{Q}(\tau_{k})^{-1} (\mathbf{x_{k}} - e^{\mathbf{A} \tau_{k}} \mathbf{x_{k-1}})
\right) \circ (\mathbf{V}) $$
\begin{multline}
\nonumber
= -\frac{1}{2} \biggl( D (\mathbf{x_{k}} - e^{\mathbf{A} \tau_{k}} \mathbf{x_{k-1}})^\mathbf{{T}} \circ (\mathbf{V}) \mathbf{Q}(\tau_{k})^{-1} \left( \mathbf{x_{k}} - e^{\mathbf{A} \tau_{k}} \mathbf{x_{k-1}} \right) \\
+ (\mathbf{x_{k}} - e^{\mathbf{A} \tau_{k}} \mathbf{x_{k-1}})^\mathbf{{T}} \mathbf{Q}(\tau_{k})^{-1} D \left( \mathbf{x_{k}} - e^{\mathbf{A} \tau_{k}} \mathbf{x_{k-1}} \right) \circ (\mathbf{V}) \biggr)
\end{multline}
$$
= (\mathbf{x_{k}} - e^{\mathbf{A} \tau_{k}} \mathbf{x_{k-1}})^\mathbf{{T}} \mathbf{Q}(\tau_{k})^{-1} D \left(e^{\mathbf{A} \tau_{k}} \mathbf{x_{k-1}} \right) \circ (\mathbf{V})
$$
Where, the derivative of the matrix-exponential is \cite{matcalc}
$$D(e^{\mathbf{X}}) \circ (\mathbf{V}) = \sum_{r=0}^{\infty} \frac{1}{(r+1)!} \sum_{j=0}^{r} \mathbf{X}^{j} \mathbf{V} \mathbf{X}^{r-j}$$
And with $\mathbf{X} = \mathbf{A} \tau_{k}$, the derivative of our original quadratic form with increment $\mathbf{V}$ is given as:
$$
=
(\mathbf{x_{k}} - e^{\mathbf{A} \tau_{k}} \mathbf{x_{k-1}})^{\mathbf{T}} \mathbf{Q}(\tau_{k})^{-1}
\sum_{r=0}^{\infty} \frac{1}{(r+1)!} \sum_{j=0}^{r} (\mathbf{A} \tau_{k})^{j} \tau_{k} \mathbf{V} (\mathbf{A} \tau_{k})^{r-j} \mathbf{x_{k-1}}
$$
$$
= (\mathbf{x_{k}} - e^{\mathbf{A} \tau_{k}} \mathbf{x_{k-1}})^{\mathbf{T}} \mathbf{Q}(\tau_{k})^{-1}
\sum_{r=0}^{\infty} \frac{(\tau_{k})^{r+1}}{(r+1)!} \sum_{j=0}^{r} \mathbf{A}^{j} \mathbf{V} \mathbf{A}^{r-j} \mathbf{x_{k-1}}
$$
Taking the vectorization, we find:
$$
\vect{\left( 
\sum_{r=0}^{\infty} \frac{(\tau_{k})^{r+1}}{(r+1)!} \sum_{j=0}^{r} (\mathbf{x_{k}} - e^{\mathbf{A} \tau_{k}} \mathbf{x_{k-1}})^{\mathbf{T}} \mathbf{Q}(\tau_{k})^{-1} \mathbf{A}^{j} \mathbf{V} \mathbf{A}^{r-j} \mathbf{x_{k-1}} \right)}$$ 
$$=  \sum_{r=0}^{\infty} \frac{(\tau_{k})^{r+1}}{(r+1)!} \sum_{j=0}^{r} \vect{\left( (\mathbf{x_{k}} - e^{\mathbf{A} \tau_{k}} \mathbf{x_{k-1}})^{\mathbf{T}} \mathbf{Q}(\tau_{k})^{-1} \mathbf{A}^{j} \mathbf{V} \mathbf{A}^{r-j} \mathbf{x_{k-1}} \right)}
$$
$$=  \sum_{r=0}^{\infty} \frac{(\tau_{k})^{r+1}}{(r+1)!} \sum_{j=0}^{r} \left( \left( \mathbf{A}^{r-j} \mathbf{x_{k-1}} \right)^{\mathbf{T}} \otimes \left((\mathbf{x_{k}} - e^{\mathbf{A} \tau_{k}} \mathbf{x_{k-1}})^{\mathbf{T}} \mathbf{Q}(\tau_{k})^{-1} \mathbf{A}^{j} \right) \right)\vect{\mathbf{V}} = Dg(\mathbf{A}) \vect{\mathbf{V}}
$$
Therefore, solving for $\mathbf{A}$ reduces to finding the root of the following expression:
$$
\sum_{r=0}^{\infty} \frac{(\tau_{k})^{r+1}}{(r+1)!} \sum_{j=0}^{r} \left( \left( \mathbf{A}^{r-j} \mathbf{x_{k-1}} \right)^{\mathbf{T}} \otimes \left((\mathbf{x_{k}} - e^{\mathbf{A} \tau_{k}} \mathbf{x_{k-1}})^{\mathbf{T}} \mathbf{Q}(\tau_{k})^{-1} \mathbf{A}^{j} \right) \right) = \mathbf{0}
$$
This directly implies that:
$$
\sum_{r=0}^{\infty} \frac{(\tau_{k})^{r+1}}{(r+1)!} \sum_{j=0}^{r} \left( \mathbf{A}^{r-j} \mathbf{x_{k-1}} \right) \otimes \left( (\mathbf{A}^{j})^{\mathbf{T}} \mathbf{Q}(\tau_{k})^{-1} (\mathbf{x_{k}} - e^{\mathbf{A} \tau_{k}} \mathbf{x_{k-1}}) \right) = \mathbf{0}
$$
And, as $\mathbf{u} \otimes \mathbf{v} = \mathbf{0} \iff \mathbf{u v^{T}} = \mathbf{0}$ for $\mathbf{u}, \mathbf{v} \in \mathbb{R}^{n}$, we have that:
$$
\sum_{r=0}^{\infty} \frac{(\tau_{k})^{r+1}}{(r+1)!} \sum_{j=0}^{r} \mathbf{A}^{r-j} (\mathbf{x_{k} x_{k-1}^{T}} - e^{\mathbf{A} \tau_{k}} \mathbf{x_{k-1} x_{k-1}^{T}})^{\mathbf{T}} \mathbf{Q}(\tau_{k})^{-1} \mathbf{A}^{j}  = \mathbf{0}
$$
This can analogously be found without the use of Kronecker products by identifying the gradient directly with trace-rotation on the quadratic form:
$$
(\mathbf{x_{k}} - e^{\mathbf{A} \tau_{k}} \mathbf{x_{k-1}})^{\mathbf{T}} \mathbf{Q}(\tau_{k})^{-1}
\sum_{r=0}^{\infty} \frac{(\tau_{k})^{r+1}}{(r+1)!} \sum_{j=0}^{r} \mathbf{A}^{j} \mathbf{V} \mathbf{A}^{r-j} \mathbf{x_{k-1}} $$
$$= \sum_{r=0}^{\infty} \sum_{j=0}^{r} \Tr{\left[ \mathbf{A}^{r-j} \mathbf{x_{k-1}} (\mathbf{x_{k}} - e^{\mathbf{A} \tau_{k}} \mathbf{x_{k-1}})^{\mathbf{T}} \mathbf{Q}(\tau_{k})^{-1} \frac{(\tau_{k})^{r+1}}{(r+1)!}  \mathbf{A}^{j} \mathbf{V} \right]}
$$
$$
= \langle \sum_{r=0}^{\infty} \sum_{j=0}^{r} \left(\mathbf{A}^{r-j} \mathbf{x_{k-1}} (\mathbf{x_{k}} - e^{\mathbf{A} \tau_{k}} \mathbf{x_{k-1}})^{\mathbf{T}} \mathbf{Q}(\tau_{k})^{-1}
 \frac{(\tau_{k})^{r+1}}{(r+1)!}  \mathbf{A}^{j} \right)^{\mathbf{T}}, \mathbf{V} \rangle_{F}$$
$$ = \langle \nabla_{\mathbf{A}} \ln{\mathds{P}( \mathbf{x_{k} | x_{k-1}, A, Q_c}, \tau_{k})}, \mathbf{V} \rangle_{F}
$$
Which yields the result:
$$
\nabla_{\mathbf{A}} \ln{\mathds{P}(\mathbf{x_{k} | x_{k-1}, A, Q_c}, \tau_{k})} =
\sum_{r=0}^{\infty} \sum_{j=0}^{r} 
\frac{(\tau_{k})^{r+1}}{(r+1)!}(\mathbf{A^{T}})^{j} \left( \mathbf{Q}(\tau_{k})^{-1} (\mathbf{x_{k}} - e^{\mathbf{A} \tau_{k}} \mathbf{x_{k-1}}) \mathbf{x_{k-1}^{T}} \right) (\mathbf{A^{T}})^{r-j} = 0
$$
Introducing the sum over time-steps, and taking the expectation of this quantity under the current posterior, the result follows.
\end{proof}
A sensible solution would satisfy the following condition for non-zero matrices $\mathbf{A}$, for time-steps $\tau_{k} > 0$:
$$
\sum_{k=2}^{N} \mathbf{Q}(\tau_{k})^{-1} \left( \E_{\sim \mathbf{x_{k},x_{k-1}}}[\mathbf{x_{k} x_{k-1}^{T}}] - e^{\mathbf{A} \tau_{k}} \E_{\sim \mathbf{x_{k-1}}}[\mathbf{x_{k-1} x_{k-1}^{T}}] \right) = 0
$$
Given the assumption that $\dim\text{range}\mathbf{A}=n$, and thus $\dim\text{range}(\mathbf{A^{T}})^{j} = \dim\text{range}(\mathbf{A^{T}})^{r-j} = n$, it must hold that
$$(\mathbf{A^{T}})^{r-j} \mathbf{X} (\mathbf{A^{T}})^{j} \mathbf{v} = 0 \iff (\mathbf{A^{T}})^{j} \mathbf{v} \in \text{null}((\mathbf{A^{T}})^{r-j} \mathbf{X}) = \text{null}(\mathbf{X})$$
So for this to hold in general over $\mathbb{R}^{n}$, it is reasonable to solve for $\mathbf{A}$ requiring
$$\mathbf{X} = \sum_{k=2}^{N} \mathbf{Q}(\tau_{k})^{-1} \left( \E_{\sim \mathbf{x_{k},x_{k-1}}}[\mathbf{x_{k} x_{k-1}^{T}}] - e^{\mathbf{A} \tau_{k}} \E_{\sim \mathbf{x_{k-1}}}[\mathbf{x_{k-1} x_{k-1}^{T}}] \right) = \mathbf{0}$$
While this is not the most general result, it yields a strong initial condition on $\mathbf{A}$ which can be used in a numerical optimization which further refines $\mathbf{A}$ to its true value.

Noting that the matrix exponential is given by a converging series, we restrict our attention to the first terms for a tractable and fast approximate solution for $\mathbf{A}$, i.e. to use as an initial condition for a numerical optimization. This more general numeric optimization using conjugate-Newton is introduced in the context of a non-linear matrix-exponential regression for fitting higher-order terms with no restriction on the matrix, its spectral radius, or its stability. However, we note the particular approximation below remains valuable on its own as an efficient routine for learning matrices with $\rho(\mathbf{A}) < 1$, i.e. stabilizable or non-stabilizable matrices with small spectral radius. These matrices are of interest for a number of reasons. In particular, stabilizable matrices where $\rho(e^{\mathbf{A}}) < 1$, or $\Re(\lambda_{i}(\mathbf{A})) < 0$. Learning matrices constrained to be explicitly stabilizable is an explored area \cite{stableLDS}.
A typical simplifying assumption demonstrating the value of stabilizable matrices is to consider a noiseless case, where using a diagonalization of the matrix exponential, we see $\mathbf{x}(t) = e^{\mathbf{A} t} \mathbf{x_{0}} = e^{\mathbf{P \Lambda} t \mathbf{P}^{-1}} \mathbf{x_{0}} = \mathbf{P} e^{\mathbf{\Lambda} t} \mathbf{P}^{-1} \mathbf{x_{0}}$ yields an unbounded state in the limit $\lim_{t \to \infty} \mathbf{x}(t)$ for eigenvalues of $e^{\mathbf{\Lambda}}$ greater than one. Therefore, if one assumes a system involves a stabilizable matrix of small spectral radius, such as the genetic circuit presented in \ref{sect:experiments}, it is necessarily the case that the system state is bounded in the limit, and any higher-order terms vanish quickly with $t$, with small error given by the higher-order terms of the Taylor series.

\begin{prop}\label{prop:approximatingA}
A second-order approximation for a minimization of the form given by:
\begin{equation}
\max_{\mathbf{A} \in \mathbb{R}^{n \times n}} \left\{ -\frac{1}{2} \sum_{k=2}^{N} \lVert \mathbf{x_{k}} - e^{\mathbf{A} \tau_{k}} \mathbf{x_{k-1}} \rVert_{\mathbf{Q}(\tau_{k})^{-1/2}}^{2} \right\}
\end{equation}
Under an expectation over a smoothed posterior on states $\sim \mathbf{x_{k}}, \mathbf{x_{k-1}} | \mathbf{z_{1:N}}$, and assuming the matrices $\mathbf{A}, \mathbf{Q}(\tau_{k})$ commute for all $k$, is given by the update:
$$
\mathbf{A} = 
\left( 
\sum_{k=2}^{N}
\mathbf{Q}(\tau_{k})^{-1}\left(\E_{\sim \mathbf{x_{k},x_{k-1}}}[\mathbf{x_{k} x_{k-1}^{T}}] - \E_{\sim \mathbf{x_{k-1}}}[\mathbf{x_{k-1} x_{k-1}^{T}}] \right) 
\right)
\left(
\sum_{k=2}^{N} \tau_{k} \mathbf{Q}(\tau_{k})^{-1} \E_{\sim \mathbf{x_{k-1}}} [\mathbf{x_{k-1} x_{k-1}^{T}} ]\right)^{-1}
$$
In the case they do not commute, a mixed second-order, least-squares approximation is given by:
$$\mathbf{A} = \left( \sum_{k=2}^{N} \tau_{k}\Tr{\left[ \E[\mathbf{x_{k-1}x_{k-1}^{T}}] \right]} \left(
\E[\mathbf{x_{k} x_{k-1}^{T}}] - \E[\mathbf{x_{k-1} x_{k-1}^{T}}]
\right)
\right) \left( \sum_{k=2}^{N}
\tau_{k}^{2} \Tr{\left[\E[\mathbf{x_{k-1} x_{k-1}^{T}}]\right]} \E[\mathbf{x_{k-1} x_{k-1}^{T}}]
\right)^{-1}$$
With $\lVert \mathbf{r} \rVert_{\mathbf{Q_{i}}^{-1/2}}$ denoting the norm given by $\mathbf{r^{T} Q_{i}^{-1/2} Q_{i}^{-1/2} r} = \mathbf{r^{T} Q_{i}^{-1} r}$, for $\mathbf{Q_{i}^{-1}} \succ \mathbf{0}$ a positive-definite precision-matrix (which always has a unique, positive-definite square-root $\mathbf{Q_{i}^{-1/2}}$ by the real spectral theorem).
\end{prop}

\begin{proof}
We consider the following condition required of $\mathbf{A}$ from before:
$$
\sum_{k=2}^{N} \mathbf{Q}(\tau_{k})^{-1} \E_{\sim \mathbf{x_{k},x_{k-1}}}[\mathbf{x_{k} x_{k-1}^{T}}] = \sum_{k=2}^{N} \mathbf{Q}(\tau_{k})^{-1} e^{\mathbf{A} \tau_{k}} \E_{\sim \mathbf{x_{k-1}}}[\mathbf{x_{k-1} x_{k-1}^{T}}] $$
With a simple second-order approximation of the matrix-exponential's Taylor series:
$$\approx 
\sum_{k=2}^{N} \mathbf{Q}(\tau_{k})^{-1}
\E_{\sim \mathbf{x_{k-1}}}[\mathbf{x_{k-1} x_{k-1}^{T}}] + \sum_{k=2}^{N} \mathbf{Q}(\tau_{k})^{-1} \mathbf{A} \tau_{k} \E_{\sim \mathbf{x_{k-1}}}[\mathbf{x_{k-1} x_{k-1}^{T}}]
$$
Which, if the commutator $[\mathbf{Q}(\tau_{k})^{-1}, \mathbf{A} ] = 0$, yields an analytical update for $\mathbf{A}$ (for $\mathbf{Q}(\tau_{k})$ fixed from the previous step):
\begin{equation}\label{eqn:Acommute}
\mathbf{A} = 
\left( 
\sum_{k=2}^{N}
\mathbf{Q}(\tau_{k})^{-1}\left(\E_{\sim \mathbf{x_{k},x_{k-1}}}[\mathbf{x_{k} x_{k-1}^{T}}] - \E_{\sim \mathbf{x_{k-1}}}[\mathbf{x_{k-1} x_{k-1}^{T}}] \right) 
\right)
\left(
\sum_{k=2}^{N} \tau_{k} \mathbf{Q}(\tau_{k})^{-1} \E_{\sim \mathbf{x_{k-1}}} [\mathbf{x_{k-1} x_{k-1}^{T}} ]\right)^{-1}
\end{equation}
In the case the matrices do not commute, we note that $\mathbf{Q}(\tau_{k}) \succ 0$ represents a positive-definite and invertible matrix. Therefore the following also represents a numerically reasonable update for $\mathbf{A}$, derived from least-squares and naturally independent of $\mathbf{Q}(\tau_{k})^{-1}$:
$$
\mathbf{A} = \left( \sum_{k=2}^{N} \tau_{k}\Tr{\left[ \E[\mathbf{x_{k-1}x_{k-1}^{T}}] \right]} \left(
\E[\mathbf{x_{k} x_{k-1}^{T}}] - \E[\mathbf{x_{k-1} x_{k-1}^{T}}]
\right)
\right) \left( \sum_{k=2}^{N}
\tau_{k}^{2} \Tr{\left[\E[\mathbf{x_{k-1} x_{k-1}^{T}}]\right]} \E[\mathbf{x_{k-1} x_{k-1}^{T}}]
\right)^{-1}
$$
One sees this by formulating the least-squares problem from the equation:
\begin{equation}
\sum_{k=2}^{N} \mathbf{Q}(\tau_{k})^{-1} \left( \E_{\sim \mathbf{x_{k},x_{k-1}}}[\mathbf{x_{k} x_{k-1}^{T}}] -
\E_{\sim \mathbf{x_{k-1}}}[\mathbf{x_{k-1} x_{k-1}^{T}}] \right)  = \sum_{k=2}^{N} \mathbf{Q}(\tau_{k})^{-1} \mathbf{A} \tau_{k} \E_{\sim \mathbf{x_{k-1}}}[\mathbf{x_{k-1} x_{k-1}^{T}}]
\end{equation}
By defining the matrices $\mathbf{\Tilde{Y}}$, $\mathbf{\Tilde{W}}$ for the least-squares problem $\mathbf{\Tilde{Y}} = \mathbf{\Tilde{W} A^{T}} $:
$$
\mathbf{\Tilde{Y}} =
\begin{bmatrix}
    \left( \mathbf{x_{2} x_{1}^{T}} -
\mathbf{x_{1} x_{1}^{T}} \right)^{\mathbf{T}}
    & ... &
    \left( \mathbf{x_{N} x_{N-1}^{T}} -
\mathbf{x_{N-1} x_{N-1}^{T}} \right)^{\mathbf{T}}
\end{bmatrix}
$$
$$
\mathbf{\Tilde{W}} =
\begin{bmatrix}
    \left( \mathbf{x_{1} x_{1}^{T}} \right)^{\mathbf{T}} \tau_{2}
    & ... &
    \left( \mathbf{x_{N-1} x_{N-1}^{T}} \right)^{\mathbf{T}} \tau_{N}
\end{bmatrix}
$$
Which yields a solution of the form $\left(\mathbf{\Tilde{W}^{T} \Tilde{W}} \right) \mathbf{A}^{T} = \mathbf{\Tilde{W}^{T} \Tilde{Y}}$, from which the conclusion follows in expectation.
\end{proof}

One may note some similarities of the above expression with the standard EM update for the Kalman filter. The distinction present is that this involves a difference between the joint and marginal between states, and relies on time-step normalization to yield a continuous motion matrix in differential form. In practice, we avoid the inversion of the right matrix, instead using a linear-system solver on the normal-equations defined by $\mathbf{\Tilde{Y}} = \mathbf{\Tilde{W} A^{T}} $.

To generalize this update, one may rely on $\mathbf{A^{(0)}} \triangleq \mathbf{A}$ for $\mathbf{A}$ defined as above as an initialization, and use the gradient $\nabla_{\mathbf{A}} \ln{\mathds{P}(\mathbf{x_{k} | x_{k-1}, A, Q_c}, \tau_{k})}$ to update $\mathbf{A}$ with numerical methods such as BFGS or Newton's. In particular, as the infinite series given above representing the gradient is convergent, one may choose some $q$ terms to approximate this gradient, and to then use the approximated gradient to iteratively refine $\mathbf{A^{(i)}}$ at each Expectation-step of EM.

\begin{lemma}\label{lemma:convergence}
The expectation of the gradient $\mathbb{E}_{\sim \mathbf{x_{k}, x_{k-1}}}[\nabla_{\mathbf{A}} \ln{\mathds{P}(\mathbf{x_{k} | x_{k-1}, A, Q_c}, \tau_{k})}]$ is given by a convergent series.
\end{lemma}
\begin{proof}
Under the posterior given by the EM-procedure, the gradient is given as:
$$\mathbb{E}_{\sim \mathbf{x_{k}, x_{k-1}}}[\nabla_{\mathbf{A}} \ln{\mathds{P}(\mathbf{x_{k} | x_{k-1}, A, Q_c}, \tau_{k})}]$$
$$
=
\sum_{r=0}^{\infty} \sum_{j=0}^{r} \sum_{k=2}^{N}
\frac{(\tau_{k})^{r+1}}{(r+1)!}(\mathbf{A}^{T})^{j} \left( \mathbf{Q}(\tau_{k})^{-1} (\mathbb{E}[\mathbf{x_{k}x_{k-1}^{T}}] - e^{\mathbf{A} \tau_{k}} \mathbb{E}[\mathbf{x_{k-1}x_{k-1}^{T}}])  \right) (\mathbf{A^{T}})^{r-j}
$$
For $\mathcal{M}_{n} = \{ \mathbf{M} \in \mathbb{R}^{n \times n} \}$, and a sequence $\{ \mathbf{A_{r}} \} \subset \mathcal{M}_{n}$, the series $\sum_{r=0}^{\infty} \mathbf{A_{r}}$ converges if there is some matrix-norm $\lVert . \rVert$ such that $\sum_{r=0}^{\infty} \lVert \mathbf{A_{r}} \rVert$ is convergent. Considering the Frobenius-norm, $\lVert \mathbf{M} \rVert_{F} = \Tr\mathbf{M^{T} M}$, we see the above series is bounded as:
$$ \lVert \mathbb{E}_{\sim \mathbf{x_{k}, x_{k-1}}}[\nabla_{\mathbf{A}} \ln{\mathds{P}(\mathbf{x_{k} | x_{k-1}, A, Q_c}, \tau_{k})}] \rVert_{F}$$
$$
\leq 
\sum_{r=0}^{\infty} \sum_{j=0}^{r} \sum_{k=2}^{N}
\frac{(\tau_{k})^{r+1}}{(r+1)!} \lVert(\mathbf{A}^{T})^{j} \left( \mathbf{Q}(\tau_{k})^{-1} (\mathbb{E}[\mathbf{x_{k}x_{k-1}^{T}}] - e^{\mathbf{A} \tau_{k}} \mathbb{E}[\mathbf{x_{k-1}x_{k-1}^{T}}])  \right) (\mathbf{A^{T}})^{r-j}
\rVert_{F}
$$
$$
\leq \sum_{r=0}^{\infty}  \sum_{k=2}^{N}
\sum_{j=0}^{r}
\frac{(\tau_{k})^{r+1}}{(r+1)!} \lVert \mathbf{A} \rVert_{F}^{r} \lVert \left( \mathbf{Q}(\tau_{k})^{-1} (\mathbb{E}[\mathbf{x_{k}x_{k-1}^{T}}] - e^{\mathbf{A} \tau_{k}} \mathbb{E}[\mathbf{x_{k-1}x_{k-1}^{T}}])  \right)
\rVert_{F}
$$
$$
= \sum_{k=2}^{N}
\sum_{r=0}^{\infty} \left(\frac{(\tau_{k} \lVert \mathbf{A} \rVert_{F})^{r}}{(r)!} \right)
 \tau_{k} \lVert \left( \mathbf{Q}(\tau_{k})^{-1} (\mathbb{E}[\mathbf{x_{k}x_{k-1}^{T}}] - e^{\mathbf{A} \tau_{k}} \mathbb{E}[\mathbf{x_{k-1}x_{k-1}^{T}}])  \right)
\rVert_{F}
$$
As the matrix-exponential $e^{\mathbf{A} \tau_{k}}$ is itself a convergent series, $\mathbf{Q}(\tau_{k})^{-1} \in \mathcal{S}_{n}^{+}$ represents a fixed covariance-matrix, $\tau_{k} \in \mathbb{R}^{+}$ $ \forall k$, and the auto and cross-correlations have finite moments given by \ref{eqn:autocor} and \ref{eqn:crosscor}, it clearly holds that:
$$ \lVert \mathbb{E}_{\sim \mathbf{x_{k}, x_{k-1}}}[\nabla_{\mathbf{A}} \ln{\mathds{P}(\mathbf{x_{k} | x_{k-1}, A, Q_c}, \tau_{k})}] \rVert_{F} \leq$$
$$
 \sum_{k=2}^{N}
e^{\lVert A \rVert_{F} \tau_{k}}
 \tau_{k} \lVert \left( \mathbf{Q}(\tau_{k})^{-1} (\mathbb{E}[\mathbf{x_{k}x_{k-1}^{T}}] - e^{\mathbf{A} \tau_{k}} \mathbb{E}[\mathbf{x_{k-1}x_{k-1}^{T}}])  \right)
\rVert_{F} < +\infty
$$
So that the gradient is convergent, and it is justified to take curtailed approximation of $q$-terms.
\end{proof}

\begin{corollary}
As a corollary of \ref{prop3}, an equivalent form for the gradient, which we have shown is convergent, is given as:
\begin{equation}
    \sum_{k=2}^{N} \tau_{k} \int_{0}^{1} e^{(\mathbf{A^{T}}\tau_{k}) (1 - s)} \left( \mathbf{Q}(\tau_{k})^{-1} \left(\mathbb{E}[\mathbf{x_{k}x_{k-1}^{T}}] - e^{\mathbf{A} \tau_{k}} \mathbb{E}[\mathbf{x_{k-1}x_{k-1}^{T}}] \right)  \right) e^{(\mathbf{A^{T}} \tau_{k}) s} ds
\end{equation}
\begin{equation}
= \sum_{k=2}^{N} \tau_{k} \int_{0}^{1} e^{(\mathbf{A^{T}}\tau_{k}) (1 - s)} \mathbf{V_{k}} e^{(\mathbf{A^{T}} \tau_{k}) s} ds
\end{equation}
 For $\mathbf{V_{k}} = \left( \mathbf{Q}(\tau_{k})^{-1} \left(\mathbb{E}[\mathbf{x_{k}x_{k-1}^{T}}] - e^{\mathbf{A} \tau_{k}} \mathbb{E}[\mathbf{x_{k-1}x_{k-1}^{T}}] \right)  \right)$. 
\end{corollary}

\begin{proof}

We identify the Fréchet derivative, given by the linear map $V \to \int_{0}^{\tau=1} e^{\mathbf{X} (\tau - s)} V e^{\mathbf{X} s} ds = D(e^{\mathbf{X}}) \circ (V) = \sum_{r=0}^{\infty} \frac{1}{(r+1)!} \sum_{j=0}^{r} \mathbf{X}^{j} V \mathbf{X}^{r-j}$ in the gradient as:
$$
=\sum_{r=0}^{\infty} \sum_{j=0}^{r} \sum_{k=2}^{N}
\frac{(\tau_{k})^{r+1}}{(r+1)!}(\mathbf{A^{T}})^{j} \left( \mathbf{Q}(\tau_{k})^{-1} \left(\mathbb{E}[\mathbf{x_{k}x_{k-1}^{T}}] - e^{\mathbf{A} \tau_{k}} \mathbb{E}[\mathbf{x_{k-1}x_{k-1}^{T}}] \right)  \right) (\mathbf{A^{T}})^{r-j}
$$
$$
= \sum_{k=2}^{N} \tau_{k} \sum_{r=0}^{\infty} \frac{1}{(r+1)!} \sum_{j=0}^{r} (\mathbf{A^{T}} \tau_{k})^{r-j} \left( \mathbf{Q}(\tau_{k})^{-1} \left(\mathbb{E}[\mathbf{x_{k}x_{k-1}^{T}}] - e^{\mathbf{A} \tau_{k}} \mathbb{E}[\mathbf{x_{k-1}x_{k-1}^{T}}] \right)  \right) (\mathbf{A^{T}} \tau_{k})^{r-j}
$$
$$
= \sum_{k=2}^{N} \tau_{k} \int_{0}^{1} e^{(\mathbf{A^{T}}\tau_{k}) (1 - s)} \left( \mathbf{Q}(\tau_{k})^{-1} \left(\mathbb{E}[\mathbf{x_{k}x_{k-1}^{T}}] - e^{\mathbf{A} \tau_{k}} \mathbb{E}[\mathbf{x_{k-1}x_{k-1}^{T}}] \right)  \right) e^{(\mathbf{A^{T}} \tau_{k}) s} ds = \sum_{k=2}^{N} \tau_{k} \int_{0}^{1} e^{(\mathbf{A^{T}}\tau_{k}) (1 - s)} \mathbf{V_{k}} e^{(\mathbf{A^{T}} \tau_{k}) s} ds
$$
\end{proof}
 It follows from this that one may compute the gradient using existing numeric methods for the Fréchet derivative of the matrix exponential $\mathbf{A^{T}} \tau_{k}$ in the direction of increment $\mathbf{V_{k}}$ \cite{AlMohy2009}. We rely on this as the gradient in our implementation using Conjugate-Newton for general matrices $\mathbf{A}$ of arbitrary spectral radius.

We note that when Newton's method is invoked for the dynamics matrix $\mathbf{A}$, we also require the expected log-likelihood under the posterior given by $\sum_{k=2}^{N} \mathbb{E}_{\sim \mathbf{x_{k}, x_{k-1}}}[\ln{\mathds{P}(\mathbf{x_{k} | x_{k-1}, A, Q_c}, \tau_{k})}]$ as a function of $\mathbf{A}$, where:
$$\sum_{k=2}^{N} \mathbb{E}_{\sim \mathbf{x_{k}, x_{k-1}}}[\ln{\mathds{P}(\mathbf{x_{k} | x_{k-1}, A, Q_c}, \tau_{k})}] =$$
$$
\sum_{k=2}^{N} \mathbb{E}_{\sim \mathbf{x_{k}, x_{k-1}}}\left[\left(\mathbf{x_{k}} - e^{\mathbf{A}\tau_{k}} \mathbf{x_{k-1}} \right)^{\mathbf{T}} \mathbf{Q}(\tau_{k})^{-1} \left(\mathbf{x_{k}} - e^{\mathbf{A}\tau_{k}} \mathbf{x_{k-1}} \right) \right]$$
$$= \sum_{k=2}^{N} \mathbb{E}_{\sim \mathbf{x_{k}, x_{k-1}}}\biggl[\Tr\left[\left(\mathbf{x_{k}} - e^{\mathbf{A}\tau_{k}} \mathbf{x_{k-1}} \right)\left(\mathbf{x_{k}} - e^{\mathbf{A}\tau_{k}} \mathbf{x_{k-1}} \right)^{\mathbf{T}} \mathbf{Q}(\tau_{k})^{-1}\right]\biggr]
$$
\begin{equation}\label{eqn:loglikeA}
= \Tr\biggl[\sum_{k=2}^{N} \biggl( \mathbb{E}_{\sim \mathbf{x_{k}}}[\mathbf{x_{k} x_{k}^{T}}] - e^{\mathbf{A} \tau_{k}} \mathbb{E}_{\sim \mathbf{x_{k}, x_{k-1}}}[\mathbf{x_{k-1} x_{k}^{T}}] - \mathbb{E}_{\sim \mathbf{x_{k}, x_{k-1}}}[\mathbf{x_{k} x_{k-1}^{T}}]e^{\mathbf{A^{T}} \tau_{k}} + e^{\mathbf{A} \tau_{k}} \mathbb{E}_{\sim \mathbf{x_{k-1}}}[\mathbf{x_{k-1} x_{k-1}^{T}}] e^{\mathbf{A^{T}} \tau_{k}} \biggr)\mathbf{Q}(\tau_{k})^{-1} \biggr]
\end{equation}
Which is readily computable given the moments derived from the smoothing procedure, and is used in our conjugate-Newton update.

\section{An M-Step optimization for the homogeneous dynamics covariance matrix}

From before, the covariance is given as a time-dependence function following the SDE given by $d\mathbf{x}(t) = \mathbf{A x}(t) dt + d\mathbf{w}(t)$, with second-order conditions on the noise-differential of $\E[d\mathbf{w}(t)d\mathbf{w}(t)^{\mathbf{T}}] = \mathbf{Q_c} dt$ and $\E[\mathbf{w}(t)\mathbf{w}(s)^{\mathbf{T}}] = \mathbf{Q_c} \delta(t - s)$, for $\delta$ denoting a Dirac-delta. The analytical solution for this covariance-matrix, given a set of fixed dynamics represented by $\mathbf{A}$, is given in half-vectorized form as: $\vech{\mathbf{Q}(t)} = \mathbf{A_{P}}^{-1} (\mathbbm{1} - e^{-\mathbf{A_{P}}(\tau-\tau_{0})} ) \vech{\mathbf{Q_c}}$. As such, we see that in the sliced-form, one has a time-dependent matrix pre-multiplying some constant differential covariance matrix. This offers the possibility of converting a more difficult optimization for a continuously evolving set of matrices which changes with each step, $\{ \mathbf{Q}(\tau_{k}) \}_{k=1}^{N}$, to one which can be solved numerically.

\begin{prop}
The solution for the optimal differential covariance $\mathbf{Q_{c}} = \vect^{-1}{\left( \mathbf{D} \vech{\mathbf{Q}(t)} \right)} = \left(
\vect{\mathbbm{1}}^{T} \otimes \mathbbm{1}
\right) \left( \mathbbm{1} \otimes \mathbf{D} \vech{\mathbf{Q}(t)} \right)$, with only symmetry explicit, is given by solving the following linear system:
$$
\mathbf{\Tilde{F}^{T} \Tilde{F}} \vech{\mathbf{Q_c}} = \mathbf{\Tilde{F}^{T} \Tilde{Z}}
$$
For a unique half-vectorization $\vech{\mathbf{Q_{c}}}$ of least-norm, and matrices $\mathbf{\Tilde{F}}$, $\mathbf{\Tilde{Z}}$ defined by:
$$
\mathbf{\Tilde{F}} = 
\begin{bmatrix}
\mathbf{A_{P}}^{-1}(e^{\mathbf{A_{P}} \tau_{2}} - \mathbbm{1})\\
... \\
\mathbf{A_{P}}^{-1}(e^{\mathbf{A_{P}} \tau_{N}} - \mathbbm{1})
\end{bmatrix}
$$
$$
\mathbf{\Tilde{Z}} = \begin{bmatrix}
\vech{\E[(\mathbf{x_{2}} - e^{\mathbf{A} \tau_{2}} \mathbf{x_{1}})(\mathbf{x_{2}} - e^{\mathbf{A} \tau_{2}} \mathbf{x_{1}})^{\mathbf{T}}]}\\
... \\
\vech{\E[(\mathbf{x_{N}} - e^{\mathbf{A} \tau_{N-1}} \mathbf{x_{N-1}})(\mathbf{x_{N}} - e^{\mathbf{A} \tau_{N}} \mathbf{x_{N-1}})^{\mathbf{T}}]}
\end{bmatrix}$$
Across $N$ total time-steps, a set of fixed time differences $\{ \tau_{k}: \tau_{k} > 0 \}_{k=2}^{N}$, and a fixed dynamics matrix $\mathbf{A}$.
Equivalently, one finds the following closed-form expression for the approximator of $\vech{\mathbf{Q_{c}}}$:
$$
\vech{\mathbf{Q_{c}}} = \left( N - 1 \right)^{-1} \biggl( \sum_{k=2}^{N} \left( e^{\mathbf{A_{P}} \tau_{k}} - \mathbbm{1} \right)^{-1} \mathbf{A_{P}} \vech{\mathbb{E}_{\mathbf{x_{k},x_{k-1}|z_{1:N}}}\left[(\mathbf{x_{k}} - e^{\mathbf{A} \tau_{k}} \mathbf{x_{k-1}})(\mathbf{x_{k}} - e^{\mathbf{A} \tau_{k}} \mathbf{x_{k-1}})^{\mathbf{T}} \right]} \biggr)
$$
\end{prop}

\begin{proof}\label{proof:Q}
Starting with the proportionality of the log-likelihood with respect to the differential covariance $\mathbf{Q_{c}}$, one sees:
$$
\ln{\mathds{P}(\mathbf{x_{k}|x_{k-1}, A, Q_c}, \tau_{k})} \propto_{\mathbf{Q_c}} -\frac{1}{2} \ln{|\mathbf{Q}(\tau_{k})|} -\frac{1}{2} (\mathbf{x_{k}} - e^{\mathbf{A} \tau_{k}} \mathbf{x_{k-1}})^{\mathbf{T}} \mathbf{Q}(\tau_{k})^{-1} (\mathbf{x_{k}} - e^{\mathbf{A} \tau_{k}} \mathbf{x_{k-1}})$$
$$= -\frac{1}{2} \ln{|\mathbf{Q}(\tau_{k})|} - \frac{1}{2} \Tr\left[\mathbf{Q}(\tau_{k})^{-1} (\mathbf{x_{k}} - e^{\mathbf{A} \tau_{k}} \mathbf{x_{k-1}}) (\mathbf{x_{k}} - e^{\mathbf{A} \tau_{k}} \mathbf{x_{k-1}})^{\mathbf{T}} \right] = h
$$
Letting $\mathbf{Z_{k}} = (\mathbf{x_{k}} - e^{\mathbf{A} \tau_{k}} \mathbf{x_{k-1}}) (\mathbf{x_{k}} - e^{\mathbf{A} \tau_{k}} \mathbf{x_{k-1}})^{\mathbf{T}}$, and $\mathbf{Q}(\tau_{k}) = \mathbf{Q_{k}}$, we take the derivative as:
$$
dh(\mathbf{Q_{k}}; d\mathbf{Q_{k}}(\mathbf{Q_c}; d\mathbf{Q_c})) = -\frac{1}{2} \Tr\left[\mathbf{Q_{k}}^{-1} d\mathbf{Q_{k}}\right] + \frac{1}{2} \Tr\left[\mathbf{Q_{k}}^{-1} \mathbf{Z_{k}} \mathbf{Q_{k}}^{-1} d\mathbf{Q_{k}}\right]
$$
Therefore:
$$
dh(\mathbf{Q_{k}}; d\mathbf{Q_{k}}(\mathbf{Q_c}; d\mathbf{Q_c})) = \frac{1}{2} \Tr \left[(\mathbf{Q_{k}}^{-1} \mathbf{Z_{k}} \mathbf{Q_{k}}^{-1} - \mathbf{Q_{k}}^{-1} )d\mathbf{Q_{k}} \right]
$$
Given the following expression:
$$\sum_{k=2}^{N} \Tr \left[ (\mathbf{Q_{k}}^{-1}\mathbf{Z_{k}}\mathbf{Q_{k}}^{-1} - \mathbf{Q_{k}}^{-1})d\mathbf{Q_{k}} \right] = 0$$
We seek to re-express $d\mathbf{Q_{k}}$ in terms of the differential $d\mathbf{Q_c}$ to find the derivative of the expression with respect to our differential covariance.

First, we note that $\vect{d\mathbf{Q_{k}}} = \mathbf{D A_{P}}^{-1} \left( e^{\mathbf{A_{P}} \tau_{k}} - \mathbbm{1} \right) d\vech{\mathbf{Q_c}}$. We may re-express the following term from the minimization in terms of the vectorization of $d\mathbf{Q_{k}}$, a known quantity.
$$\sum_{k=2}^{N} \Tr \left[ (\mathbf{Q_{k}}^{-1}\mathbf{Z_{k}}\mathbf{Q_{k}}^{-1} - \mathbf{Q_{k}}^{-1})d\mathbf{Q_{k}} \right] = \sum_{k=2}^{N} \vect{(d\mathbf{Q_{k}})}^{T} \left( \mathbf{Q_{k}}^{-1} \otimes \mathbf{Q_{k}}^{-1} \right) \vect{\left(\mathbf{Z_{k}} - \mathbf{Q_{k}}\right)} = 0$$
As $\mathbf{Z_{k}}$ and $\mathbf{Q_{k}} \in \mathbb{R}^{n \times n}$ are symmetric, there exists a unique $n^{2} \times \frac{n(n+1)}{2}$ duplication matrix $\mathbf{D}$ which relates the half-vectorizations of the matrices to their vectorizations \cite{matcalc}. Thus, we can equivalently express this as:
$$
\sum_{k=2}^{N} \vect{(d\mathbf{Q_{k}})}^{T} \left( \mathbf{Q_{k}}^{-1} \otimes \mathbf{Q_{k}}^{-1} \right) \vect{\left(\mathbf{Z_{k}} - \mathbf{Q_{k}}\right)} = \sum_{k=2}^{N} \vect{(d\mathbf{Q_{k}})}^{T} \left( \mathbf{Q_{k}}^{-1} \otimes \mathbf{Q_{k}}^{-1} \right) \mathbf{D} \vech{\left(\mathbf{Z_{k}} - \mathbf{Q_{k}}\right)}
$$
Expanding $d\mathbf{Q_{k}}$, we find:
$$
\vech{(d\mathbf{Q_c})}^{T} \left( \sum_{k=1}^{N} (e^{\mathbf{A_{P}^{T} }\tau_{k}} - \mathbbm{1}) \mathbf{A_{P}^{-T} \mathbf{D^{T}}} \left( \mathbf{Q_{k}}^{-1} \otimes \mathbf{Q_{k}}^{-1} \right) \mathbf{D} \vech{\left(\mathbf{Z_{k}} - \mathbf{Q_{k}}\right)} \right) = 0
$$
Which implies the first order condition that:
$$
\left( \sum_{k=2}^{N} (e^{\mathbf{A_{P}^{T}}\tau_{k}} - \mathbbm{1}) \mathbf{A_{P}^{-T} \mathbf{D^{T}}} \left( \mathbf{Q_{k}}^{-1} \otimes \mathbf{Q_{k}}^{-1} \right) \mathbf{D} \vech{\left(\mathbf{Z_{k}} - \mathbf{Q_{k}}\right)} \right) = 0
$$
By theorem 3.13 in \cite{matcalc}, we see that $\mathbf{D^{T}} \left( \mathbf{Q_{k}}^{-1} \otimes \mathbf{Q_{k}}^{-1} \right) \mathbf{D}$ must be non-singular, such that $ \nul \mathbf{D^{T}} \left( \mathbf{Q_{k}}^{-1} \otimes \mathbf{Q_{k}}^{-1} \right) \mathbf{D} = \{ \mathbf{0} \}$. This implies the simultaneous equation $\vech{\left(\mathbf{Z_{k}} - \mathbf{Q_{k}}\right)} = \mathbf{0}$ satisfied for all $k$ is a sensible solution, and no analytical means of finding $\vech{\mathbf{Q_c}}$ exists for general $N$. Therefore, we approximate $\vech{\mathbf{Q_c}}$ with least squares using the following minimization:
$$
\min_{\vech{\mathbf{Q_c}}} \left( \frac{1}{2} \sum_{k=2}^{N} \lVert{ \vech{\E[\mathbf{Z_{k}}]} - \mathbf{A_{P}}^{-1} (e^{\mathbf{A_{P}} \tau_{k}} - \mathbbm{1}) \vech{\mathbf{Q_c}} \rVert}_{2}^{2} \right)
$$
We can directly cast this into a least-squares framework, where we can define:
$$
\mathbf{\Tilde{Z}} = \begin{bmatrix}
\vech{\E[\mathbf{Z_{2}}]}\\
... \\
\vech{\E[\mathbf{Z_{N}}]}
\end{bmatrix}$$
$$
\mathbf{\Tilde{F}} = 
\begin{bmatrix}
\mathbf{A_{P}}^{-1}(e^{\mathbf{A_{P}} \tau_{2}} - \mathbbm{1})\\
... \\
\mathbf{A_{P}}^{-1}(e^{\mathbf{A_{P}} \tau_{N}} - \mathbbm{1})
\end{bmatrix}
$$
And use a numerical solver for the normal equation in terms of $\vech{\mathbf{Q_c}}$:
$$
\mathbf{\Tilde{F}^{T} \Tilde{F}} \vech{\mathbf{Q_c}} = \mathbf{\Tilde{F}^{T} \Tilde{Z}}
$$
It is a standard result that there exists a unique solution to the normal-equation for which $\vech{\mathbf{Q_c}}$ is the vector of least-norm in the least-squares solution manifold, where the component of $\vech{\mathbf{Q_c}} \in \nul \mathbf{\Tilde{F}^{T}}\mathbf{\Tilde{F}}$ is identically zero.

In other words, the final, closed-form approximating solution for $\vech{\mathbf{Q_{c}}}$ is simply given by the solution to $\min_{\vech{\mathbf{Q_c}}} \left( \frac{1}{2} \sum_{k=2}^{N} \lVert{ (e^{\mathbf{A_{P}} \tau_{k}} - \mathbbm{1})^{-1} \mathbf{A_{P}} \vech{\E[\mathbf{Z_{k}}]} -  \vech{\mathbf{Q_c}} \rVert}_{2}^{2} \right)$, simply given by:
$$
\vech{\mathbf{Q_{c}}} = \left( N - 1 \right)^{-1} \biggl( \sum_{k=2}^{N} \left( e^{\mathbf{A_{P}} \tau_{k}} - \mathbbm{1} \right)^{-1} \mathbf{A_{P}} \vech{\mathbb{E}_{\mathbf{x_{k},x_{k-1}|z_{1:N}}}\left[(\mathbf{x_{k}} - e^{\mathbf{A} \tau_{k}} \mathbf{x_{k-1}})(\mathbf{x_{k}} - e^{\mathbf{A} \tau_{k}} \mathbf{x_{k-1}})^{\mathbf{T}} \right]} \biggr)
$$
\end{proof}

As $\E[\mathbf{Z_{k}}] = \E[(\mathbf{x_{k}} - e^{\mathbf{A} \tau_{k}} \mathbf{x_{k-1}})(\mathbf{x_{k}} - e^{\mathbf{A} \tau_{k}} \mathbf{x_{k-1}})^{\mathbf{T}}] \succ \mathbf{0}$, this regression for the half-vectorization of $\mathbf{Q_c}$ matches it against a positive-definite matrix computed under the current posterior, generally yielding a positive-definite result for our motion-model covariance. In the rarer case that $\mathbf{Q_{c}}$ is not positive-definite--i.e. when the condition-number of the covariance matrix $\left( \frac{\lambda_{max}(\mathbf{Q_{c}})}{\lambda_{min}(\mathbf{Q_{c}})} \right)$ is very large and the optimization finds small negative eigenvalues (unlikely, due to the setup of the regression, but theoretically possible as there is no strict definiteness constraint), we simply find the closest-approximating positive semi-definite matrix with a minimization of the form $\min_{\mathbf{Q} \in \mathcal{S}_{n}^{+}} \lVert \mathbf{Q} - \mathbf{Q_{c}} \rVert_{F}^{2}$. This has an established solution of the form $\mathbf{Q_{c}^{(new)}} = \mathbf{U} \Lambda_{+} \mathbf{U^{T}}$, where $\mathbf{U}$, $\Lambda$ come from an eigendecomposition of the result of the optimization above for $\mathbf{Q_{c}}$. We also have $\Lambda_{+}$ as a diagonal matrix containing only the positive eigenvalues of $\mathbf{Q_{c}}$, and add a small stabilizing factor of the form $\sigma^{2} \mathbbm{1}_{n}$ to ensure positive-definiteness.

In the case that $\mathbf{Q_{c}}$ has diagonal structure (e.g. stemming from an inductive assumption of a white-noise diffusion process), one may explicitly update as $\vech{\mathbf{Q_c}} \gets \vech{\mathbf{Q_c}} \odot \vech{\mathbbm{1}_{n}}$ from the update above.

\section{M-Step optimizations for the other parameters}\label{sect:otherparams}

\subsection*{M-step for the observation matrix}

The updates for the observation matrix and observation covariance are given in much the same vein as for the standard Kalman Filter, as the relation $\mathbf{z}(t_{k}) = \mathbf{H x}(t_{k}) + \mathbf{v}$ has no continuous-time dependence and is represented analogously in the expected log-likelihood as in the standard discrete-time Kalman filter.

To recapitulate this result, each term in the expected log-likelihood depending on the observation matrix $\mathbf{H}$ is of the form:
$$\ln{\mathds{P}(\mathbf{z_{k}|x_{k}, R, H})} \propto_{\mathbf{H}} \mathbf{(z_{k} - H x_{k})^{T} R^{-1} (z_{k} - H x_{k})}
$$
$$ = \Tr\left[ \mathbf{R^{-1}(z_{k} - H x_{k})(z_{k} - H x_{k})^{T} } \right] = \mathbf{g}
$$
Letting $\mathbf{e} = \mathbf{z_{k} - H x_{k}}$, we see:
$$
d\mathbf{g}(\mathbf{e; de}) = \Tr[\mathbf{R}^{-1}d(\mathbf{ee^{T}})] = 2 \Tr[\mathbf{R}^{-1} d\mathbf{e}(\mathbf{H; dH}) \mathbf{e^{T}}] = -2 \Tr[\mathbf{x_{k} (z_{k} - Hx_{k})^{T} R^{-1} dH}] = d\mathbf{g}(\mathbf{H; dH})
$$
With $d\mathbf{e}(\mathbf{H;dH}) = -d\mathbf{H} \mathbf{x_{k}}$. Therefore we get the first-order condition for the observation matrix as:
$$
\mathbf{0} = \mathbf{(z_{k} - Hx_{k})x_{k}^{T}}
$$
Now, introducing the sum over time-steps and taking the expectation of this quantity, we find:
$$
N^{-1} \sum_{k=1}^{N} \E [\mathbf{(z_{k} - H x_{k}) x_{k}^{T}}] = N^{-1} \left( \sum_{k=1}^{N} \mathbf{z_{k}} \E[\mathbf{x_{k}}]^{\mathbf{T}} - \mathbf{H} \left( \sum_{k=1}^{N} \E[\mathbf{x_{k} x_{k}^{T}}] \right) \right) = \mathbf{0}
$$
Therefore, our final update for H is given as:
$$
\mathbf{H} = \left( \sum_{k=1}^{N} \mathbf{z_{k}} \E[\mathbf{x_{k}}]^{\mathbf{T}} \right) \left( \sum_{k=1}^{N} \E[\mathbf{x_{k} x_{k}^{T}}] \right)^{-1}
$$
If we also want to account for observation sparsity, for $\lambda \geq 0$, an alternate update is given as:
$$
\min_{\mathbf{H}} \left( \lVert \mathbf{Z - H X} \rVert_{F}^{2} + \lambda \lVert \mathbf{H} \rVert_{1} \right)
$$
With the regression in the Frobenius norm, and $\mathbf{Z}$ and $\mathbf{X}$ given as: $\mathbf{Z} = \begin{bmatrix}
\mathbf{z_{1}} &
... &
\mathbf{z_{N}}
\end{bmatrix}$, and $\mathbf{X} = \begin{bmatrix}
\mathbf{x_{1}} &
... &
\mathbf{x_{N}}
\end{bmatrix}$.
Supposing we drop the regularization, take the derivative of $\Tr[\mathbf{(Z - H X)^{T}(Z - H X)}]$, and set the gradient to zero, one finds the normal equation for the solution for $\mathbf{H}$:
$$
\mathbf{X X^{T} H^{T}} =\mathbf{ X Z}
$$
Which implies the least-squares form is identical to the standard update for $\mathbf{H}$ in expectation.

\subsection*{M-step for the observation covariance}

$$\ln{\mathds{P}(\mathbf{z_{k}|x_{k}, R, H})} \propto_{\mathbf{R}} -\frac{1}{2} \ln{|\mathbf{R}|} - \frac{1}{2} \mathbf{(z_{k} - H x_{k})^{T} R^{-1} (z_{k} - H x_{k})}
$$
Introducing a sum across time-steps and making use of the maximum likelihood solution for the Gaussian, the observation covariance is given as: $\mathbf{R} = N^{-1} \sum_{k=1}^{N} \E[\mathbf{(z_{k} - H x_{k})(z_{k} - H x_{k})^{T}}]$.

\subsection*{M-step for the initial state and state covariance}

The log-likelihood is dependent on the initial state as:
$$
\ln{\mathds{P}(\mathbf{x_{1}|\mu_{0}, P_{0}})} \propto_{\mathbf{\mu_{0}}} \mathbf{(x_{1} - \mu_{0})^{T} P_{0}^{-1} (x_{1} - \mu_{0})}
$$
We assume in this case that all means have unique initial conditions with shared dynamics therefrom. As such, the MLE for the multivariate Gaussian is simply given in closed-form as: $\mathbf{\mu_{0}} = \E[\mathbf{x_{1}}]$.

Likewise, expressing the proportionality of the likelihood for the initial covariance:
$$
\ln{\mathds{P}(\mathbf{x_{1}|\mu_{0}, P_{0}})} \propto_{\mathbf{P_{0}}} -\frac{1}{2} \ln{|\mathbf{P_{0}}|} -\frac{1}{2}\mathbf{(x_{1} - \mu_{0})^{T} P_{0}^{-1} (x_{1} - \mu_{0})}
$$
Recapitulating the standard result for the MLE for the covariance of a multivariate Gaussian, the initial state covariance is given as \cite{KFEM_OG}:
$$
\mathbf{P_{0}} = \E[\mathbf{(x_{1} - \mu_{0})(x_{1} - \mu_{0})^{T}}] = \E[\mathbf{x_{1}x_{1}^{T}}] - \mathbf{\mu_{0} \mu_{0}^{T}}$$


\end{document}